\documentclass[10pt]{article} % For LaTeX2e
\usepackage[accepted]{tmlr}
\usepackage{hyperref}
\usepackage{url}

\usepackage{subcaption}
\usepackage{algorithm}
\usepackage[noend]{algpseudocode}
\usepackage[T1]{fontenc}
\usepackage[utf8]{inputenc}
\usepackage[english]{babel}
\usepackage{amsthm}
\usepackage{amsmath}
\usepackage{amssymb}
\usepackage{graphicx}
\usepackage{dsfont}
\usepackage{array}
\usepackage{multirow}
\usepackage{multicol}
\usepackage{caption}
\usepackage{color}
\usepackage{float}
\usepackage{enumerate}
\usepackage{lipsum}
\usepackage{booktabs}
\usepackage[capitalize,noabbrev]{cleveref}
\usepackage{thmtools}
\usepackage{thm-restate}
\usepackage{mathtools}
\usepackage{tablefootnote}

\theoremstyle{plain}

\theoremstyle{plain}

\theoremstyle{plain}
\newtheorem{lem}{\protect\lemmaname}
\theoremstyle{plain}

\theoremstyle{plain}
  
\theoremstyle{definition}

\theoremstyle{definition}

\theoremstyle{definition}
\newtheorem{rem}{\protect\remarkname}

\makeatother

\providecommand{\claimname}{Claim}
\providecommand{\lemmaname}{Lemma}
\providecommand{\propositionname}{Proposition}
\providecommand{\theoremname}{Theorem}
\providecommand{\corollaryname}{Corollary} 
\providecommand{\definitionname}{Definition}
\providecommand{\assumptionname}{Assumption}
\providecommand{\remarkname}{Remark}

\newcommand{\overbar}[1]{\mkern 1.25mu\overline{\mkern-1.25mu#1\mkern-0.25mu}\mkern 0.25mu}

\DeclareMathOperator*{\argmax}{arg\,max}
\DeclareMathOperator*{\argmin}{arg\,min}

\newcommand{\abf}{{\mathbf{a}}}
\newcommand{\bbf}{{\mathbf{b}}}
\newcommand{\ebf}{{\mathbf{e}}}
\newcommand{\kbf}{{\mathbf{k}}}
\newcommand{\ubf}{{\mathbf{u}}}
\newcommand{\vbf}{{\mathbf{v}}}
\newcommand{\wbf}{{\mathbf{w}}}
\newcommand{\xbf}{{\mathbf{x}}}
\newcommand{\ybf}{{\mathbf{y}}}
\newcommand{\zbf}{{\mathbf{z}}}
\newcommand{\Ibf}{{\mathbf{I}}}
\newcommand{\Kbf}{{\mathbf{K}}}

\newcommand{\RR}{{\mathbb{R}}}
\newcommand{\EE}{{\mathbb{E}}}
\newcommand{\NN}{{\mathbb{N}}}

\newcommand{\fbar}{{\overbar{f}}}
\newcommand{\gbar}{{\overbar{g}}}

\newcommand{\sigbar}{{\overbar{\sigma}}}

\newcommand{\xbfbar}{{\overbar{\mathbf{x}}}}
\newcommand{\zbfbar}{{\overbar{\mathbf{z}}}}

\newcommand{\gtil}{{\widetilde{g}}}
\newcommand{\htil}{{\widetilde{h}}}

\newcommand{\util}{{\widetilde{u}}}

\newcommand{\Dtil}{{\widetilde{D}}}

\newcommand{\Ltil}{{\widetilde{L}}}
\newcommand{\Otil}{{\widetilde{O}}}

\newcommand{\mutil}{{\widetilde{\mu}}}
\newcommand{\sigtil}{{\widetilde{\sigma}}}

\newcommand{\xbftil}{{\widetilde{\mathbf{x}}}}
\newcommand{\zbftil}{{\widetilde{\mathbf{z}}}}

\newcommand{\dhat}{{\widehat{d}}}

\newcommand{\Cc}{\mathcal{C}}
\newcommand{\Dc}{\mathcal{D}}

\newcommand{\Fc}{\mathcal{F}}
\newcommand{\Gc}{\mathcal{G}}
\newcommand{\Hc}{\mathcal{H}}

\newcommand{\Sc}{\mathcal{S}}

\newcommand{\Xc}{\mathcal{X}}

\newcommand{\dia}[1]{\mathrm{diam}\big(#1\big)}
\newcommand{\del}[1]{\Delta^{(#1)}}
\newcommand{\xx}[1]{\mathbf{x}^{(#1)}}
\newcommand{\zz}[1]{\mathbf{z}^{(#1)}}
\newcommand{\ff}[1]{f^{(#1)}}

\newcommand{\muu}[1]{\mu^{(#1)}}
\newcommand{\muutil}[1]{\widetilde{\mu}^{(#1)}}
\newcommand{\sig}[1]{\sigma^{(#1)}}
\newcommand{\UCB}{\mathrm{UCB}}
\newcommand{\LCB}{\mathrm{LCB}}
\newcommand{\UCBB}[1]{\mathrm{UCB}^{(#1)}}
\newcommand{\LCBB}[1]{\mathrm{LCB}^{(#1)}}
\newcommand{\UCBBB}[1]{\overline{\mathrm{UCB}}^{(#1)}}
\newcommand{\LCBBB}[1]{\overline{\mathrm{LCB}}^{(#1)}}
\renewcommand{\cite}{\citep}

\title{Regret Bounds for\\Noise-Free Cascaded Kernelized Bandits}

\author{\name Zihan Li \email lizihan@u.nus.edu \\
      \addr National University of Singapore
      \AND
      \name Jonathan Scarlett \email scarlett@comp.nus.edu.sg \\
      \addr National University of Singapore}

\def\month{05}  % Insert correct month for camera-ready version
\def\year{2024} % Insert correct year for camera-ready version
\def\openreview{\url{https://openreview.net/forum?id=oCfamUtecN}} % Insert correct link to OpenReview for camera-ready version

\begin{document}

\maketitle
%\graphicspath{{../}}
\begin{abstract}
We consider optimizing a function network in the noise-free grey-box setting with RKHS function classes, where the exact intermediate results are observable. We assume that the structure of the network is known (but not the underlying functions comprising it), and we study three types of structures: (1) chain: a cascade of scalar-valued functions, (2) multi-output chain: a cascade of vector-valued functions, and (3) feed-forward network: a fully connected feed-forward network of scalar-valued functions. We propose a sequential upper confidence bound based algorithm GPN-UCB along with a general theoretical upper bound on the cumulative regret. In addition, we propose a non-adaptive sampling based method along with its theoretical upper bound on the simple regret for the Mat\'ern kernel. We also provide algorithm-independent lower bounds on the simple regret and cumulative regret. Our regret bounds for GPN-UCB have the same dependence on the time horizon as the best known in the vanilla black-box setting, as well as near-optimal dependencies on other parameters (e.g., RKHS norm and network length).
\end{abstract}

\section{Introduction}

Black-box optimization of an expensive-to-evaluate function based on point queries is a ubiquitous problem in machine learning. Bayesian optimization (or Gaussian process optimization) refers to a class of methods using Gaussian processes (GPs), whose main idea is to place a prior over the unknown function and update the posterior according to point query results. Bayesian optimization has a wide range of applications including parameter tuning \cite{snoek2012practical}, experimental design \cite{griffiths2020constrained}, and robotics \cite{lizotte2007automatic}.  While function evaluations are noisy in most applications, there are also scenarios where noise-free modeling can be suitable, such as simulation \cite{nguyen2016cascade}, goal-driven dynamics learning \cite{bansal2017goal}, and density map alignment \cite{singer2023alignment}.

In the literature on Bayesian optimization, the problem of optimizing a real-valued black-box function is usually studied under two settings: (1) Bayesian setting: the target function is sampled from a known GP prior, and (2) non-Bayesian setting: the target function has a low norm in reproducing kernel Hilbert space (RKHS). 

In this work, we consider a setting falling ``in between'' the white-box setting (where the full definition of the target function is known) and the black-box setting, namely, a grey-box setting in which the algorithms can leverage partial internal information of the target function beyond merely the final outputs or even slightly modify the target function  \cite{astudillo2021thinking}.  Existing grey-box optimization methods exploit internal information such as observations of intermediate outputs for composite functions \cite{astudillo2019bayesian} and lower fidelity but faster approximation of the final output for modifiable functions \cite{huang2006sequential}. Numerical experiments show that the grey-box methods significantly outperform standard black-box methods \cite{astudillo2021thinking}.

Though any real-valued network can be treated as a single black-box function and solved with classical Bayesian optimization methods, we explore the benefits offered by utilizing the network structure information and the exact intermediate results under the noise-free grey-box setting. Several practical applications of the cascaded setting (e.g., alloy heat treatment and simulation) are highlighted in \cref{sec:motivation}.

\subsection{Related Work}

Numerous works have proposed Bayesian optimization algorithms for optimizing a single real-valued black-box function under the RKHS setting. For the noisy setting, \cite{srinivas2010gaussian,chowdhury17kernelized,gupta2022regret} provided a typical cumulative regret $\Otil(\sqrt{T}\gamma_T)$, where $T$ is the time horizon, $\gamma_T$ is the maximum information gain associated to the underlying kernel, and the $\Otil(\cdot)$ notation hides the poly-logarithmic factors. Recently, \cite{camilleri2021high,salgia2021domain,li2022gaussian} achieved cumulative regret $\Otil(\sqrt{T\gamma_T})$, which nearly matches algorithm-independent lower bounds for the squared exponential and Mat\'ern kernels \cite{scarlett2017lower,cai2021lower}. For the noise-free setting, \cite{bull2011convergence} achieved a nearly optimal simple regret $O(T^{-\nu/d})$ for the Mat\'ern kernel with smoothness $\nu$. This result implies a two-batch algorithm that uniformly selects $T^\frac{d}{\nu+d}$ points in the first batch and repeatedly picks the returned point in the second batch, has cumulative regret $O(T^\frac{d}{\nu+d})$. In addition, \cite{lyu2019efficient} provided deterministic cumulative regret $O(\sqrt{T\gamma_T})$, with the rough idea being to substitute zero noise into the analysis of GP-UCB \cite{srinivas2010gaussian}. Recently, \cite{salgia2023random} proposed a batch algorithm based on random sampling, attaining cumulative regret $\Otil(T^{1-\nu/d})$ when $\nu<d$ and $O(\mathrm{poly}(\log T))$ when $\nu\ge d$.

Meanwhile, several Bayesian optimization algorithms for optimizing a composition of multiple functions under the noise-free setting have been proposed. \cite{nguyen2016cascade} provides a method for cascade Bayesian optimization; \cite{astudillo2019bayesian} studies optimizing a composition of a black-box function and a known cheap-to-evaluate function; and \cite{astudillo2021bayesian} studies optimizing a network of functions sampled from a GP prior under the grey-box setting. Both \cite{astudillo2019bayesian} and \cite{astudillo2021bayesian} prove the asymptotic consistency of their expected improvement sampling based methods. 

The most related works to ours are \cite{kusakawa2021bayesian} and \cite{sussex2023modelbased}. \cite{kusakawa2021bayesian} introduces two confidence bound based algorithms along with their regret guarantees for both noise-free and noisy settings, as well as an expected improvement based algorithm without theory. \cite{sussex2023modelbased} considers directed grey-box networks representing a causal structure, and proposes an expected improvement based method with regret guarantee.  A detailed comparison regarding problem setup and theoretical performance is provided in \cref{sec:comparison}. In short, we significantly improve certain dependencies in their regret for a UCB-type approach, and we study two new directions -- non-adaptive sampling and algorithm-independent lower bounds -- that were not considered therein (summarized in Section \ref{sec:contributions}).

For function networks with $m$ layers, our cumulative regret bounds are expressed in terms of
\begin{align*}
    \Sigma_T=\max_{i\in[m]}\max_{\zbf_1,\dots,\zbf_T\in\Xc^{(i)}}\sum_{t=1}^T \sig{i}_{t-1}(\zbf_t),
\end{align*}
where $\Xc^{(i)}$ and $\sig{i}_{t-1}(\cdot)$ denote the domain and posterior standard deviation  of the $i$-th layer respectively. This is a term for general layer-composed networks associated to domains $\Xc^{(1)},\dots,\Xc^{(m)}$ and kernel $k$. When $m=1$, the term $\max_{\xbf_1,\dots,\xbf_T\in\Xc}\sum_{t=1}^T \sigma_{t-1}(\xbf_t)$ often appears in the cumulative regret analysis of classic black-box optimization \cite{srinivas2010gaussian,lyu2019efficient,vakili2022open}. Explicit upper bounds on $\Sigma_T$ will be discussed in \cref{sec:ucb_disc}.

\subsection{Contributions} \label{sec:contributions}

We study the problem of optimizing an $m$-layer function network in the noise-free grey-box setting, where the exact intermediate results are observable. We focus on three types of network structures: (1) chain: a cascade of scalar-valued functions, (2) multi-output chain: a cascade of vector-valued functions, and (3) feed-forward network: a fully connected feed-forward network of scalar-valued functions.  Then:

\begin{itemize} \itemsep0ex
	\item We propose a fully sequential upper confidence bound based algorithm GPN-UCB along with its upper bound on cumulative regret for each network structure.  Our regret bound significantly reduces certain dependencies compared to \cite{kusakawa2021bayesian}, in particular showing that their dependence on a ``posterior standard deviation Lipschitz constant'' can be completely removed.
%	\item We show in \cref{sec:noisy} that GPN-UCB extends readily to the noisy setting, and that the resulting regret bound is sub-linear in $T$ for considerably broader kernels than \cite{sussex2023modelbased}.
	\item We introduce a non-adaptive sampling based method, and provide its theoretical upper bound on the simple regret for the Mat\'ern kernel. 
	% \item \blue{Our algorithms (including noisy GPN-UCB) attain better theoretical performance than existing works \cite{kusakawa2021bayesian,sussex2023modelbased}, and numerical experiments (deferred to \cref{sec:exp}) show that they also have good empirical performance.}
	\item We provide algorithm-independent lower bounds on the simple and cumulative regret for an arbitrary algorithm optimizing any chain, multi-output chain, or feed-forward network associated to the Mat\'ern kernel. In broad regimes of interest, these provide evidence or even proof that our upper bounds are near-optimal.
	\item While the goals of this paper are essentially entirely theoretical, we show in \cref{sec:exp} that (slight variations of) our algorithms can be effective in at least simple experimental scenarios.
\end{itemize}

Let $d$ denote the dimension of the domain, $d_{\max}$ denote the maximum dimension among all the $m$ layers, and $D_{2,m}$ denote the product of dimensions from the second layer to the last layer. With $B>0$ restricting the magnitude and smoothness of each layer and $L>1$ restricting the slope of each layer, a partial summary of the proposed cumulative regret bounds for the Mat\'ern kernel with smoothness $\nu\ge 1$ when $d\ge\nu$ and $T=\Omega\big((B(cL)^{m-1})^{d/\nu}\big)$ for $c=\Theta(1)$ is displayed in \cref{tab:summary_cml}.  From this table, we note the following:
\begin{itemize}
	\item The upper and lower bounds share the same $BL^{m-1}T^{1-\nu/d}$ dependence when $d=d_{\max}$ and simultaneously a conjecture of \cite{vakili2022open} holds.
	\item Even without such a conjecture, the dependence on $\Sigma_T$ is precisely that given in state-of-the-art bounds for the vanilla black-box setting \cite{vakili2022open}, and rigorous upper bounds on it are known.
\end{itemize}
See \cref{sec:ucb_disc} for further details on the conjecture and rigorous bounds.

\begin{table*}[t!]
	\centering
	\begin{tabular}{lll}
		\toprule
		Lower bound & $\Omega(B(cL)^{m-1}T^{1-\nu/d})$ &\\
		\midrule
		Upper bound (chains) & $O(2^{m} BL^{m-1}\Sigma_T)$&$\stackrel{(c)}{=}O(2^{m} BL^{m-1}T^{1-\nu/d})$ \\
		Upper bound (multi-output chains) & $O(5^m BL^{m-1}\Sigma_T)$&$\stackrel{(c)}{=}O(5^m BL^{m-1}T^{1-\nu/d_{\max}})$ \\
		Upper bound (feed-forward networks) & $O(2^{m} \sqrt{D_{2,m}} BL^{m-1}\Sigma_T)$&$\stackrel{(c)}{=}O(2^{m} \sqrt{D_{2,m}} BL^{m-1}T^{1-\nu/d_{\max}})$ \\
		\bottomrule                 
	\end{tabular}
	\caption{Summary of cumulative regret bounds for the Mat\'ern kernel when $d\ge\nu\ge1$ and $T=\Omega\big((B(cL)^{m-1})^{d/\nu}\big)$ for some $c=\Theta(1)$. Here $\stackrel{(c)}{=}$ indicates the behavior when a conjecture of \cite{vakili2022open} on the black-box setting holds.}
	\label{tab:summary_cml}
\end{table*} 
\begin{figure*}[t!]
	\centering
	\includegraphics[width=0.8\textwidth]{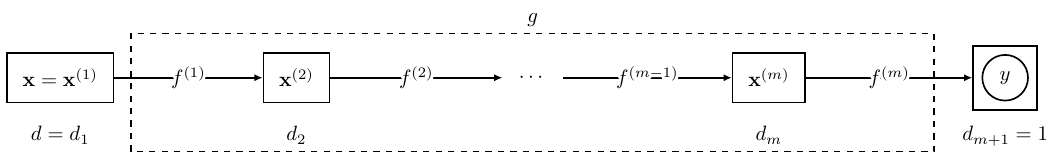}
	\caption{A function network $g$ of $m$ layers with input $\xbf$ and output $y$.}
	\label{fig:setup}
\end{figure*}
To our knowledge, we are the first to attain provably near-optimal scaling (in broad cases of interest), and doing so requires both improving the existing upper bounds and attaining novel lower bounds.  A full summary of our theoretical results is provided in \cref{sec:summary}. Perhaps our most restrictive assumption is noise-free observations, but we believe this is a crucial stepping stone towards the noisy setting (as was the case with regular black-box optimization, e.g., \cite{bull2011convergence}). % Moreover, we extend one of our main algorithms and results to the noisy setting in \cref{sec:noisy} and show that its regret is sub-linear in much broader scenarios than existing results.

\section{Problem Setup}\label{sec:setup}

We consider optimizing a real-valued grey-box function $g$ on $\Xc=[0,1]^d$ based on noise-free point queries. As shown in \cref{fig:setup}, the target function $g$ is known to be a network of $m$ unknown layers $\ff{i}$ with $i\in[m]$. In general, for any input $\xbf\in\Xc$, the network $g$ has $\xx{1}=\xbf$ and 
\begin{align*}
    \xx{i+1}&=\ff{i}(\xx{i})\hspace{12ex}\text{ for }i\in[m-1],\\
    y&=g(\xbf)=\ff{m}(\xx{m})\hspace{3ex}\in\RR,
\end{align*}
where $\xx{i}$ has dimension $d_i$ for each $i\in[m]$. The domain of $\ff{i}$ is $\Xc^{(i)}$, and the range of $\ff{i}$ is $\Xc^{(i+1)}$.\footnote{Note the equivalent notation $\Xc^{(1)} = \Xc$ and $d_1 = d$.} For any $\zbf\in\Xc^{(i)}$, there exists $\xbf\in\Xc$ such that $\xx{i}=\zbf$. 

We aim to find $\xbf^\ast=\argmax_{\xbf\in\Xc}g(\xbf)$ based on a sequence of point queries up to time horizon $T$. When we query $g$ with input $\xbf_t$ at time step $t$, the intermediate noise-free results $\xx{2}_t,\dots,\xx{m}_t$ and the final noise-free output $y_t$ are accessible. We measure the performance as follows: 
\begin{itemize}
    \item \textbf{Simple regret}: With $\xbf_T^\ast$ being the additional point returned after $T$ rounds, the simple regret is defined as $r_T^\ast=g(\xbf^\ast)-g(\xbf_T^\ast)$;
    \item \textbf{Cumulative regret}: The cumulative regret incurred over $T$ rounds is defined as $R_T=\sum_{i=1}^T r_t$ with $r_t=g(\xbf^\ast)-g(\xbf_t)$.
\end{itemize}

\subsection{Kernelized Bandits}

We assume $g$ is a composition of multiple constituent functions, for which we consider both scalar-valued functions and vector-valued functions based on a given kernel. 
For a scalar-valued kernel $k$ and a known constant $B>0$, we consider scalar-valued functions that lie in $\Hc_k(B)$, the reproducing kernel Hilbert space (RKHS) associated to $k$, with norm at most $B$. In this work, we focus on the Mat\'ern kernel $k_\text{Mat\'ern}$ with smoothness $\nu>0$. Similarly, for an operator-valued kernel $\Gamma$ and a known constant $B>0$, we consider vector-valued functions in $\Hc_\Gamma(B)$, the RKHS corresponding to $\Gamma$ with norm at most $B$. More details on RKHS and $k_\text{Mat\'ern}$ are given in \cref{sec:rkhs}.

\subsection{Surrogate GP Model}
As is common in kernelized bandit problems, our algorithms employ a surrogate Bayesian GP model for $f\in\Hc_k(B)$. For prior with zero mean and kernel $k$, given a sequence of points $(\xbf_1,\dots,\xbf_t)$ and their noise-free observations $\mathbf{y}_t=(y_1,\dots,y_t)$ up to time $t$, the posterior distribution of the function is a GP with mean and variance given by \cite{rasmussen2006gaussian}
\begin{align}
    \mu_t(\xbf) &= \mathbf{k}_t(\xbf)^T \mathbf{K}_t^{-1}\mathbf{y}_t \label{eq:mean}\\
    \sigma_t(\xbf)^2  &= k(\xbf,\xbf) - \mathbf{k}_t(\xbf)^T \mathbf{K}_t^{-1}\mathbf{k}_t(\xbf),\label{eq:var}
\end{align}
where $\mathbf{k}_t(\xbf) = [k(\xbf, \xbf_i)]_{i=1}^{t}\in\RR^{t\times 1}$ and $\mathbf{K}_t = [k(\xbf_i, \xbf_j)]_{i,j=1}^t\in\RR^{t\times t}$. The following lemma shows that the posterior confidence region defined with parameter $B$ is always deterministically valid.
\begin{lem}\textup{\citep[Corollary 3.11]{kanagawa2018gaussian}}
\label{lem:conf}
For $f\in\Hc_k(B)$, let $\mu_t(\xbf)$ and $\sigma_t(\xbf)^2$ denote the posterior mean and variance based on $t$ points $(\xbf_1,\dots,\xbf_t)$ and their noise-free observations $(y_1,\dots,y_t)$ using \eqref{eq:mean} and \eqref{eq:var}. Then, it holds for all $\xbf\in\Xc$ that % $|f(\xbf)-\mu_t(\xbf)|\le B\sigma_t(\xbf).$
\begin{align*}
    |f(\xbf)-\mu_t(\xbf)|\le B\sigma_t(\xbf).
\end{align*}
\end{lem}
We also impose a surrogate GP model for functions in $\Hc_\Gamma(B)$. The posterior mean and covariance matrix based on $(\xbf_1,\dots,\xbf_t)$ and the noise-free observations $Y_t=(\ybf_1,\dots, \ybf_t)$ are \cite{chowdhury2021no}
\begin{align}
    \mu_t(\xbf)&=G_t(\xbf)^T G_t^{-1}Y_t,\label{eq:multi_mean}\\
    \Gamma_t(\xbf,\xbf)&=\Gamma(\xbf,\xbf)-G_t(\xbf)^T G_t^{-1}G_t(\xbf)\label{eq:multi_var},
\end{align}
where $G_t(\xbf)=[\Gamma(\xbf,\xbf_i)]_{i=1}^t \in \RR^{nt\times n}$, $G_t=[\Gamma(\xbf_i,\xbf_j)]_{i,j=1}^t \in \RR^{nt\times nt}$, and $Y_t=[\ybf_i]_{i=1}^t\in\RR^{nt\times 1}$. With $\|\cdot\|_2$ denoting the spectral norm, the following lemma provides a deterministic confidence region.

\begin{restatable}{lem}{mulconf}
\label{lem:mul}
For $f\in\Hc_\Gamma(B)$, let $\mu_t(\xbf)$ and $\Gamma_t(\xbf,\xbf)$ denote the posterior mean and variance based on $t$ points $(\xbf_1,\dots,\xbf_t)$ and their noise-free observations $(\ybf_1,\dots,\ybf_t)$ using \eqref{eq:multi_mean} and \eqref{eq:multi_var}. Then, it holds for all $\xbf\in\Xc$ that % $\|f(\xbf)-\mu_t(\xbf)\|_2\le B \|\Gamma_t(\xbf,\xbf)\|_2^{1/2}.$
\begin{align*}
    \|f(\xbf)-\mu_t(\xbf)\|_2\le B \|\Gamma_t(\xbf,\xbf)\|_2^{1/2}.
\end{align*}
\end{restatable}
The proof is given in \cref{sec:cr_mul}.

\subsection{Lipschitz Continuity}
We also assume that each constituent function in the network $g$ is Lipschitz continuous. For a constant $L> 1$, we denote by $\Fc(L)$ the set of functions such that
\begin{align*}
    \Fc(L)=\{f:\|f(\xbf)-f(\xbf')\|_2\le L\|\xbf-\xbf'\|_2,\forall \xbf,\xbf'\},
\end{align*}
where $L$ is called the Lipschitz constant. This is a mild assumption, as \cite{lee2022multi} has shown that Lipschitz continuity is a guarantee for functions in $\Hc_k(B)$ for the commonly-used squared exponential kernel and Mat\'ern kernel with smoothness $\nu>1$.

\subsection{Network Structures}
In this work, we consider three types of network structure; example figures are included in \cref{app:figs}:
\begin{itemize}
    \item \textbf{Chain}: For a scalar-valued kernel $k$, a chain is a cascade of scalar-valued functions. Specifically, $d_1\ge 1$, $d_2=d_3=\cdots=d_m=1$, and $\ff{i}\in\Hc_k(B)\cap\Fc(L)$ for each $i\in[m]$.
    \item \textbf{Multi-output chain}: For an operator-valued kernel $\Gamma$, a multi-output chain is a cascade of vector-valued functions. Specifically, $d_i\ge 1$ and $\ff{i}\in\Hc_\Gamma(B)\cap\Fc(L)$ for each $i\in[m]$.
    \item \textbf{Feed-forward network}: For a scalar-valued kernel $k$, a feed-forward network is a fully-connected feed-forward network of scalar-valued functions: $d_i\ge 1$ and $\ff{i}(\zbf)=[\ff{i,j}(\zbf)]_{j=1}^{d_{i+1}}$ with $\ff{i,j}\in\Hc_k(B)\cap\Fc(L)$ for each $i\in[m],j\in[d_{i+1}]$.
\end{itemize}
In each case, the network $g$ is scalar-valued with the dimension of the final output $y$ being $d_{m+1}=1$.

\section{GPN-UCB Algorithm and Regret Bounds}

\begin{algorithm}[!t]
\caption{GPN-UCB (Gaussian Process Network - Upper Confidence Bound)}
\label{algo:ucb}
\begin{algorithmic}[1]
    % \Require Domain $\mathcal{X}$, GP prior $(\mu_0,\sigma_0)$, time horizon $T$.
    \For{$t \gets 1, 2, \dots, T$} 
        \State Select $\xbf_t\gets\argmax_{\xbf\in\Xc}\UCB_{t-1}(\xbf)$
        \State Obtain observations $\xx{2}_t,\dots,\xx{m}_t,$ and $y_t$.
        \State Compute $\UCB_t$ using \eqref{eq:ucb_chain}, \eqref{eq:ucb_mul}, or \eqref{eq:ucb_net} based on $\{\xx{1}_s,\dots,\xx{m}_s,y_s\}_{s=1}^t$.
    \EndFor   
\end{algorithmic}
\end{algorithm}

In this section, we propose a fully sequential algorithm GPN-UCB (see \cref{algo:ucb}) for chains, multi-output chains, and feed-forward networks. The algorithm works with structure-specific upper confidence bounds. Similar to GP-UCB for scalar-valued functions \cite{srinivas2010gaussian}, the proposed algorithm repeatedly queries the point with the highest posterior upper confidence bound, while the posterior upper confidence bound $\UCB_{t-1}$ used here is computed based on not only the historical final outputs $\{y_s\}_{s=1}^{t-1}$ but also the intermediate results $\{\xx{2}_s,\dots,\xx{m}_s\}_{s=1}^{t-1}$.

\subsection{GPN-UCB for Chains}

A chain is a cascade of scalar-valued functions. For each $i\in[m]$, we denote by $\muu{i}_t$ and $\sig{i}_t$ the posterior mean and standard deviation of $\ff{i}$ computed using \eqref{eq:mean} and \eqref{eq:var} based on $\{\xx{i}_s,\xx{i+1}_s\}_{s=1}^{t}$.\footnote{Note the equivalent notation $\xx{1} = \xbf$ and $\xx{m+1} = y$.} Then, based on \cref{lem:conf}, the upper confidence bound and lower confidence bound of $\ff{i}(\zbf)$ based on $t$ exact observations are defined as follows:
\begin{align}
    \UCBB{i}_t(\zbf) &= \muu{i}_{t}(\zbf)+B\sig{i}_{t}(\zbf), \label{eq:UCB_orig}\\
    \LCBB{i}_t(\zbf) &= \muu{i}_{t}(\zbf)-B\sig{i}_{t}(\zbf). \label{eq:LCB_orig}
\end{align}
Since $\ff{i}\in\Fc(L)$, we have for any $\zbf,\zbf'$ that
\begin{align}
    \UCBB{i}_t(\zbf')+L\|\zbf-\zbf'\|_2
    &\ge \ff{i}(\zbf')+L\|\zbf-\zbf'\|_2
    \ge \ff{i}(\zbf),\label{eq:lip_ucb}\\
    \LCBB{i}_t(\zbf')-L\|\zbf-\zbf'\|_2
    &\le \ff{i}(\zbf')-L\|\zbf-\zbf'\|_2
    \le \ff{i}(\zbf).
\end{align}
It follows that
\begin{align}
    \UCBBB{i}_t(\zbf)&:=\min_{\zbf'}\big(\UCBB{i}_t(\zbf')+L\|\zbf-\zbf'\|_2\big),\label{eq:ucb_i}\\
    \LCBBB{i}_t(\zbf)&:=\max_{\zbf'}\big(\LCBB{i}_t(\zbf')-L\|\zbf-\zbf'\|_2\big)\label{eq:lcb_i}
\end{align}
are also valid confidence bounds for $\ff{i}(\zbf)$. $\UCBBB{i}_t(\zbf)$ is the lower envelope of a collection of upper bounds for $\ff{i}(\zbf)$, which can be obtained by considering multiple values of $\zbf'$ in \eqref{eq:lip_ucb}. Then, since $g$ is a cascade of $\ff{i}$'s, for any input $\xbf$, we can recursively construct a confidence region of $\xx{i+1}$ based on the confidence region of $\xx{i}$, and the following UCB for $g(\xbf)$ is valid:
\begin{align}
    \UCB_t(\xbf)&=\max_{\zbf\in\del{m}_t(\xbf)}\UCBBB{m}_t(\zbf),\label{eq:ucb_chain}
    % \LCB_t(\xbf)&=\min_{\zbf\in\del{m}_t(\xbf)}\LCBBB{m}_t(\zbf),
\end{align}
where $\del{i}_t(\xbf)$ denotes the confidence region of $\xx{i}$:
\begin{align*}
    \del{1}_t(\xbf)&=\{\xbf\} \\
    \del{i+1}_t(\xbf)&=\Bigg[\min_{\zbf\in\del{i}_t(\xbf)}\LCBBB{i}_t(\zbf),\max_{\zbf\in\del{i}_t(\xbf)}\UCBBB{i}_t(\zbf)\Bigg]
\end{align*}
for $i\in[m-1]$.  The theoretical performance of \cref{algo:ucb} for chains using the upper confidence bound in \eqref{eq:ucb_chain} is provided in the following theorem.

\begin{restatable}[GPN-UCB for chains]{thm}{ucbchain}
\label{thm:ucb_chain}
Under the setup of \cref{sec:setup}, given $B>0$ and $L> 1$, a scalar-valued kernel $k$, and a chain $g=\ff{m}\circ\ff{m-1}\circ\cdots\circ\ff{1}$ with $\ff{i}\in\Hc_k(B)\cap\Fc(L)$ for each $i\in[m]$, \cref{algo:ucb} achieves
% $R_T\le 2^{m+1} BL^{m-1}\Sigma_T,$
\begin{align*}
    R_T\le 2^{m+1} BL^{m-1}\Sigma_T,
\end{align*}
where $\Sigma_T=\max\limits_{i\in[m]}\max\limits_{\zbf_1,\dots,\zbf_T\in\Xc^{(i)}} \sum_{t=1}^T \sig{i}_{t-1}(\zbf_t)$. \footnote{In this definition and analogous definitions below, $\sig{i}_{t-1}$ is defined according to the hypothetical sampled points $\xbf_{\tau}^{(i)} = \zbf_{\tau}$ for $\tau=1,\dotsc,t-1$.}
\end{restatable}
The proof is given in \cref{sec:ucb_chain}, and upper bounds on $\Sigma_T$ will be discussed in \cref{sec:ucb_disc}.  Regardless of such upper bounds, we note that $B\Sigma_T$ serves as a noise-free regret bound for standard GP optimization \cite{vakili2022open}, and thus, the key distinction here is the multiplication by $L^{m-1}$.  See Section \ref{sec:lower} for a study of the extent to which this dependence is unavoidable.

%\begin{rem} \label{rem:comp} \blue{to delete}
%    As stated above, our algorithm may be difficult to implement exactly in practice; in particular: (i) Since $\Xc^{(2)},\dots,\Xc^{(m)}$ are not known, \eqref{eq:ucb_i} and \eqref{eq:lcb_i} are computed based on all $\zbf'\in\RR^{d_i}$, which may be difficult.  On the other hand, we view this as a ``theoretical trick'', and we expect that in practical scenarios the effect of skipping this step and using \eqref{eq:UCB_orig}--\eqref{eq:LCB_orig} in place of \eqref{eq:ucb_i}--\eqref{eq:lcb_i} would be minimal. (ii) Recursively computing \eqref{eq:ucb_chain} may also be challenging, but some practical methods for approximating this kind of acquisition function are discussed in \citep[Sec.~4.3]{astudillo2021bayesian}.   The same goes for the variations below.
%\end{rem}

We note that GPN-UCB may be difficult to implement \emph{exactly} in practice; in particular: (i) Since $\Xc^{(2)},\dots,\Xc^{(m)}$ are not known, \eqref{eq:ucb_i} and \eqref{eq:lcb_i} are computed based on all $\zbf'\in\RR^{d_i}$; (ii) Recursively computing \eqref{eq:ucb_chain} is also resource consuming. However, these problems can be alleviated by (i) only considering $\zbf'$ sufficiently close to $\zbf$ (since distant ones should have no impact) and (ii) replacing each confidence region by its intersection with a fixed \emph{discrete} domain (e.g., a finite grid).  In \cref{sec:exp}, we show that such a practical variant can be effective, at least in simple experimental scenarios.

\subsection{GPN-UCB for Multi-Output Chains}

A multi-output chain is a cascade of vector-valued functions. For any input $\zbf$ of the multi-output function $\ff{i}$, we define the confidence region of $\ff{i}(\zbf)$ as
\begin{align}
    \overbar{\Cc}^{(i)}_t(\zbf)&=\bigcap_{\zbf'}\Cc^{(i)}_t(\zbf,\zbf'), \label{eq:cbarz}
\end{align}
where
\begin{align}
    \Cc^{(i)}_t(\zbf')&=\{\muu{i}_t(\zbf')+ \ubf: \ubf\in\RR^{d_{i+1}},\|\ubf\|_2\le B\|\Gamma_t^{(i)}(\zbf,\zbf)\|_2^{1/2} \},\label{eq:cz}\\
    \Cc^{(i)}_t(\zbf,\zbf')&=\{\vbf+\wbf:
    \vbf\in\Cc^{(i)}_t(\zbf'),\|\wbf\|_2\le L\|\zbf-\zbf'\|_2\}.\label{eq:czz}
\end{align}
\cref{lem:mul} shows that $\Cc^{(i)}_t(\zbf')$ is a valid deterministic confidence region for $\ff{i}(\zbf')$. Assuming $\ff{i}\in\Fc(L)$, $\{\ff{i}(\zbf')+\wbf: \wbf\in\RR^{d_{i+1}}, \|\wbf\|_2\le L\|\zbf-\zbf'\|_2\}$ containing all the points satisfying the Lipschitz property is a valid confidence region for $\ff{i}(\zbf)$, and therefore its superset $\Cc^{(i)}_t(\zbf,\zbf')$ is also a valid confidence region for $\ff{i}(\zbf)$. Since $\ff{i}(\zbf)$ must belong to the intersection of all its confidence regions, $\overbar{\Cc}^{(i)}_t(\zbf)$ is again a valid deterministic confidence region for $\ff{i}(\zbf)$.
Hence, noting that $\overbar{\Cc}^{(m)}_{t}$ is a subset of $\RR$ (unlike the vector-valued layers), the upper confidence bound for $g(\xbf)$ for any input $\xbf$ based on $t$ observations is
\begin{align}
    \UCB_{t}(\xbf)&=\max_{\zbf\in\del{m}_{t}(\xbf)} \overbar{\Cc}^{(m)}_{t}(\zbf),\label{eq:ucb_mul}
    % \LCB_{t}(\xbf)&=\min_{\zbf\in\del{m}_{t}(\xbf)} \overbar{\Cc}^{(m)}_{t}(\zbf),
\end{align}
where $\del{i}_t(\xbf)$ denotes the confidence region of $\xx{i}$ based on $t$ observations such that
\begin{align}
    \del{1}_t(\xbf)&=\{\xbf\} \nonumber\\
    \del{i+1}_t(\xbf)&=\bigcup_{\zbf\in\del{i}_t(\xbf)} \overbar{\Cc}^{(i)}_t(\zbf)\hspace{8ex}\text{ for }i\in[m-1].\label{eq:del}
\end{align}

The cumulative regret achieved by \cref{algo:ucb} for multi-output chains using the upper confidence bound in \eqref{eq:ucb_mul} is provided in the following theorem.

\begin{restatable}[GPN-UCB for multi-output chains]{thm}{ucbmul}
\label{thm:ucb_mul}
Under the setup of \cref{sec:setup}, given $B>0$ and $L> 1$, an operator-valued kernel $\Gamma$, and a multi-output chain $g=\ff{m}\circ\ff{m-1}\circ\cdots\circ\ff{1}$ with $\ff{i}\in\Hc_\Gamma(B)\cap\Fc(L)$ for each $i\in[m]$, \cref{algo:ucb} achieves
% $R_T\le 5^m BL^{m-1}\Sigma^\Gamma_T,$
\begin{align*}
    R_T\le 5^m BL^{m-1}\Sigma^\Gamma_T,
\end{align*}
where $\Sigma^\Gamma_T=\max\limits_{i\in[m]}\max\limits_{\zbf_1,\dots,\zbf_T\in\Xc^{(i)}} \sum_{t=1}^T\|\Gamma^{(i)}_{t-1}(\zbf_t,\zbf_t)\|_2^{1/2}$.
\end{restatable}
The proof is given in \cref{sec:ucb_mul}.

\begin{rem}\label{rem:mul}
Fix $\Xc^{(1)},\Xc^{(2)},\dots,\Xc^{(m)}$ with dimension $d_1,d_2,\dots,d_m\ge 1$ respectively. For $\Gamma(\cdot,\cdot)=k(\cdot,\cdot)\Ibf$ with $k$ being a scalar-valued kernel and $\Ibf$ being the identity matrix of size $d_{i+1}$, 
%recalling that
%\begin{align*}
%	\Sigma_T&=\max\limits_{i\in[m]}\max\limits_{\zbf_1,\dots,\zbf_T\in\Xc^{(i)}} \sum_{t=1}^T \sig{i}_{t-1}(\zbf_t),\\
%	\Sigma^\Gamma_T&=\max\limits_{i\in[m]}\max\limits_{\zbf_1,\dots,\zbf_T\in\Xc^{(i)}} \sum_{t=1}^T\|\Gamma^{(i)}_{t-1}(\zbf_t,\zbf_t)\|_2^{1/2},
%\end{align*}
it follows from \cref{lem:multi_var} (see \cref{sec:multi_var}) that $\Sigma^\Gamma_T=\Sigma_T$.
\end{rem}

The upper bound on $\Sigma^\Gamma_T$ for general operator-valued kernels will be discussed in \cref{sec:ucb_disc}.

\subsection{GPN-UCB for Feed-Forward Networks}

In the feed-forward network structure, $\ff{i}(\zbf)=[\ff{i,j}(\zbf)]_{j=1}^{d_{i+1}}$ and each $\ff{i,j}\in\Hc_k(B)\cap\Fc(L)$ is a scalar-valued function. Similar to \eqref{eq:ucb_i} and \eqref{eq:lcb_i}, with $\muu{i,j}_t(\zbf)$ and $\sig{i,j}_t(\zbf)^2$ denoting the posterior mean and variance of $\ff{i,j}(\zbf)$ using \eqref{eq:mean} and \eqref{eq:var}, the following confidence bounds on $\ff{i,j}(\zbf)$ based on $\{(\xx{i}_s,\xx{i+1,j}_s)\}_{s=1}^t$ are valid:
\begin{align}
    \UCBBB{i,j}_t(\zbf)&=\min_{\zbf'}(\UCBB{i,j}_t(\zbf')+L\|\zbf-\zbf'\|_2), \nonumber \\
    \LCBBB{i,j}_t(\zbf)&=\max_{\zbf'}(\LCBB{i,j}_t(\zbf')-L\|\zbf-\zbf'\|_2), \label{eq:ucb_ij_bar}
\end{align}
where
\begin{align}
    \UCBB{i,j}_t(\zbf) &= \muu{i,j}_{t}(\zbf)+B\sig{i,j}_{t}(\zbf), \nonumber \\
    \LCBB{i,j}_t(\zbf) &= \muu{i,j}_{t}(\zbf)-B\sig{i,j}_{t}(\zbf). \label{eq:ucb_ij_final}
\end{align}
Then, the upper confidence bound of $g(\xbf)$ based on $t$ observations is
\begin{align}
    \UCB_t(\xbf)&=\max_{\zbf\in\del{m}_t(\xbf)}\UCBBB{m,1}_t(\zbf),\label{eq:ucb_net}
    % \LCB_t(\xbf)&=\min_{\zbf\in\del{m}_t(\xbf)}\LCBBB{m,1}_t(\zbf),
\end{align}
where
\begin{align}
    \del{1}_t(\xbf)=&\{\xbf\},\nonumber\\
    \del{i+1,j}_t(\xbf)=&\Bigg[\min_{\zbf\in\del{i}_t(\xbf)}\LCBBB{i,j}_t(\zbf) , \max_{\zbf\in\del{i}_t(\xbf)}\UCBBB{i,j}_t(\zbf) \Bigg]&\text{ for }i\in[m-1],j\in[d_{i+1}],\nonumber\\
    \del{i}_t(\xbf)=& \del{i,1}_t(\xbf)\times\cdots\times\del{i,d_i}_t(\xbf)&\text{ for }i\in[m].\label{eq:del_d}
\end{align}

The following theorem provides the theoretical performance of \cref{algo:ucb} for feed-forward networks using the upper confidence bound in \eqref{eq:ucb_net}.

\begin{restatable}[GPN-UCB for feed-forward networks]{thm}{ucbnet}
\label{thm:ucb_net}
Under the setup of \cref{sec:setup}, given $B>0$ and $L> 1$, a scalar-valued kernel $k$, and a feed-forward network $g=\ff{m}\circ\ff{m-1}\circ\cdots\circ\ff{1}$ with $\ff{i}(\zbf)=[\ff{i,j}(\zbf)]_{j=1}^{d_{i+1}}$ and $\ff{i,j}\in\Hc_k(B)\cap\Fc(L)$ for each $i\in[m],j\in[d_{i+1}]$, \cref{algo:ucb} achieves
% $R_T\le 2^{m+1} \sqrt{D_{2,m}} B L^{m-1} \Sigma_T,$
\begin{align*}
    R_T\le 2^{m+1} \sqrt{D_{2,m}} B L^{m-1} \Sigma_T,
\end{align*}
where $D_{2,m}=\prod_{i=2}^{m} d_i$ and $\Sigma_T=\max\limits_{i\in[m]}\max\limits_{\zbf_1,\dots,\zbf_T\in\Xc^{(i)}}\sum_{t=1}^T \sig{i,1}_{t-1}(\zbf_t)$.
\end{restatable}
The proof is given in \cref{sec:ucb_net}.

Since $\sig{i,1}$ in the feed-forward network setting is computed using \eqref{eq:var}, which is exactly the same as how $\sig{i}$ is computed in the chain setting, despite the slightly different superscripts, the $\Sigma_T$ term in \cref{thm:ucb_net} represents the same quantity as in \cref{thm:ucb_chain} (depending on $\Xc^{(1)},\Xc^{(2)},\dots,\Xc^{(m)}$). In particular, when $d_2=\cdots=d_m=1$, \cref{thm:ucb_net} recovers \cref{thm:ucb_chain}.

\subsection{Upper Bounds on $\Sigma_T$ and $\Sigma^\Gamma_T$}
\label{sec:ucb_disc}

A simple way to establish an upper bound on the $\Sigma_T$ term in \cref{thm:ucb_chain}, \cref{rem:mul}, and \cref{thm:ucb_net} is to essentially set the noise term to be zero in a known result for the noisy setting. For the scalar-valued function (GP bandit) optimization problem under the noisy setting, most existing upper bounds on cumulative regret are expressed in terms of the maximum information gain corresponding to the kernel defined as
$\gamma_t=\max_{\xbf_1,\dots,\xbf_t}\frac{1}{2}\log\det(\mathbf{I}_t+\lambda^{-1}\mathbf{K}_t)$
%\begin{align*}
%    \gamma_t=\max_{\xbf_1,\dots,\xbf_t}\frac{1}{2}\log\det(\mathbf{I}_t+\lambda^{-1}\mathbf{K}_t)
%\end{align*}
for a free parameter $\lambda>0$ \cite{srinivas2010gaussian}, and \cite{srinivas2010gaussian} has shown that the sum of posterior variances in the noisy setting satisfies $\sum_{t=1}^T\sigma'_{t-1}(\xbf_t)^2=O(\gamma_T)$, where
$\sigma'_t(\xbf)^2 = k(\xbf,\xbf) - \mathbf{k}_t(\xbf)^T (\mathbf{K}_t+\lambda\mathbf{I}_t)^{-1}\mathbf{k}_t(\xbf).$
%\begin{align*}
%    \sigma'_t(\xbf)^2  &= k(\xbf,\xbf) - \mathbf{k}_t(\xbf)^T (\mathbf{K}_t+\lambda\mathbf{I}_t)^{-1}\mathbf{k}_t(\xbf).
%\end{align*}
Using $\sigma_t(\xbf)\le\sigma'_t(\xbf)$ and the Cauchy-Schwartz inequality, we have $\Sigma_T=O\big(\sqrt{T\gamma_T}\big)$. An existing upper bound on $\gamma_T$ for the Mat\'ern kernel with dimension $d$ and smoothness $\nu$ on a fixed compact domain is
$\gamma_T^\text{Mat\'ern}=\Otil\big(T^\frac{d}{2\nu+d}\big)$
 \cite{vakili2021information}.
%\begin{align*}
%    \gamma_T^\text{Mat\'ern}=\Otil\big(T^\frac{d}{2\nu+d}\big).
%\end{align*}
In our setting, a simple sufficient condition for this bound to apply is that  $d_{\max}=\max_{i\in[m]}d_i$, $L$, and $m$ are constant, since then the Lipschitz assumption implies that each domain $\Xc^{(1)},\Xc^{(2)},\dots,\Xc^{(m)}$ is also compact/bounded.  More generally, we believe that uniformly bounded domains is a mild assumption, and when it holds, the bound  $\Sigma_T=O\big(\sqrt{T\gamma_T}\big)$ simplifies to $\Sigma_T=O\big(T^\frac{\nu+d_{\max}}{2\nu+d_{\max}}\big)$.

In addition, \cite{vakili2022open} provides the following conjecture on the upper bound on $\Sigma_T$ for the Mat\'ern kernel\footnote{In more detail, \cite{vakili2022open} shows that an analysis of GP-UCB gives rise to the quantity $\Theta_T^\ast=\max_{x_1,\dots,x_T}\sum_{t=1}^T\sigma_{t-1}(x_t)$, where the maximum is over an arbitrary sequence of points (not necessarily those of GP-UCB).  For $(x^\ast_1,\dots,x^\ast_T)$ where the maximum is achieved, \cite{vakili2022open} conjectures that $(x^\ast_1,\dots,x^\ast_T)$ are roughly uniformly distributed across the domain. The desired upper bound on $\Theta_T^\ast$ (and, in turn, our $\Sigma_T$) is derived by assuming that this conjecture holds.}
\begin{align*}
    \Sigma^\text{Mat\'ern}_T=\begin{cases}
    O(T^{1-\nu/d_{\max}}) &\text{ when }d_{\max}>\nu,\\
    O(\log T)&\text{ when }d_{\max}=\nu,\\
    O(1)&\text{ when }d_{\max}<\nu.
    \end{cases}
\end{align*}
We will discuss in \cref{sec:discussion} how if this conjecture is true, we can deduce the near-optimality of GPN-UCB for the Mat\'ern kernel. We note that \cite{vakili2022open} primarily conjectured on vanilla noise-free cumulative regret by conjecturing an upper bound on $\Sigma_T$. Recently, \cite{salgia2023random} used a random sampling algorithm with elimination to attain the conjectured cumulative regret, while leaving open the conjecture on $\Sigma_T$ and whether GP-UCB attains the same regret (though arguably further increasing its plausibility).

For an arbitrary operator-valued kernel $\Gamma:\Xc\times\Xc\to\RR^{n\times n}$ and a free parameter $\lambda$, the maximum information gain is defined as
% \cite{chowdhury2021no}
$\gamma^\Gamma_t=\max_{\xbf_1,\dots,\xbf_t}\frac{1}{2}\log\det(\mathbf{I}_{nt}+\lambda^{-1}G_t)$  \cite{chowdhury2021no}.
%\begin{align*}
%    \gamma^\Gamma_t=\max_{\xbf_1,\dots,\xbf_t}\frac{1}{2}\log\det(\mathbf{I}_{nt}+\lambda^{-1}G_t).
%\end{align*}
\cite{chowdhury2021no} has shown that $\sum_{t=1}^T\|\Gamma_{t-1}(\xbf_t,\xbf_t)\|_2=O(\gamma^\Gamma_T)$, and therefore $\Sigma^\Gamma_T=O\big(\sqrt{T\gamma^\Gamma_T}\big)$ by similar reasoning to above.

%While our focus in this paper is on the noise-free setting, the analysis of GPN-UCB can readily be extended to the noisy setting; the details are provided in \cref{sec:noisy}, and the improvements over \cite{kusakawa2021bayesian,sussex2023modelbased} are detailed in \cref{sec:comparison}.

\section{Non-Adaptive Sampling Based Method}

In this section, we propose a simple non-adaptive sampling based method (see \cref{algo:fd}) for each structure, and provide the corresponding theoretical simple regret for the Mat\'ern kernel. For a set of $T$ sampled points $\{\xbf_s\}_{s=1}^T$, its fill distance is defined as the largest distance from a point in the domain to the closest sampled point
$\delta_T=\max_{\xbf\in\Xc}\min_{s\in[T]}\|\xbf-\xbf_s\|_2$
 \cite{wendland2004scattered}.
%\begin{align*}
%    \delta_T=\max_{\xbf\in\Xc}\min_{s\in[T]}\|\xbf-\xbf_s\|_2.
%\end{align*}
\cref{algo:fd} samples $T$ points with $\delta_T=O(T^{-1/d})$. For $\Xc=[0,1]^d$, a simple way to construct such a sample is to use a uniform $d$-dimensional grid with step size $T^{-1/d}$. The algorithm observes the selected points in parallel, computes a structure-specific ``composite mean'' $\mu_T^g$ (to be defined shortly) for the overall network $g$, and returns the point that maximizes $\mu_T^g$.

\begin{algorithm}[!t]
\caption{Non-Adaptive Sampling Based Method}
\label{algo:fd}
\begin{algorithmic}[1]
    \State Choosing $\{\xbf_s\}_{s=1}^T$ such that $\delta_T=O(T^{-\frac{1}{d}})$.
    \State Obtain observations $\{\xx{2}_s,\dots,\xx{m}_s,y_s\}_{s=1}^T$.
    \State Compute $\mu_T^g$ 
%    using \eqref{eq:mean_chain}, \eqref{eq:mean_mul}, or \eqref{eq:mean_net} 
    based on $\{\xx{2}_s,\dots,\xx{m}_s,y_s\}_{s=1}^T$.
    \Ensure $\xbf^\ast_T=\argmax_{\xbf\in\Xc}\mu_T^g(\xbf)$.
\end{algorithmic}
\end{algorithm}

%\subsection{Non-Adaptive Sampling Method for Chains}

The composite posterior mean of $g(\xbf)$ with chain structure is defined as
\begin{align}
    \mu_T^g(\xbf)=(\muu{m}_T\circ\muu{m-1}_T\circ\cdots\circ\muu{1}_T)(\xbf),\label{eq:mean_chain}
\end{align}
where $\muu{i}_T$ denotes the posterior mean of $\ff{i}$ computed using \eqref{eq:mean} based on $\{(\xx{i}_s,\xx{i+1}_s)\}_{s=1}^T$ for each $i\in[m]$. 
Then, the following theorem provides the theoretical upper bound on the simple regret of \cref{algo:fd} using \eqref{eq:mean_chain}.  Note that the notation $\Otil(\cdot)$ hides poly-logarithmic factors \emph{with respect to the argument}, e.g., $\Otil(\sqrt{T}) = O( \sqrt{T} \cdot (\log T)^{O(1)})$ and $\Otil(2^{n}) = O(2^n \cdot n^{O(1)})$.

\begin{restatable}[Non-adaptive sampling method for chains]{thm}{fdchain}
\label{thm:fd_chain}
Under the setup of \cref{sec:setup}, given $B=\Theta(L)$,  $k=k_\text{Mat\'ern}$ with smoothness $\nu$, and a chain $g=\ff{m}\circ\ff{m-1}\circ\cdots\circ\ff{1}$ with $\ff{i}\in\Hc_k(B)\cap\Fc(L)$ for each $i\in[m]$, we have
\begin{itemize}  \itemsep0ex
	\item When $\nu\le 1$, \cref{algo:fd} achieves
	\begin{align*}
		r^\ast_T=\Otil(BL^{m-1}T^{-\nu^m/d});
	\end{align*}
	\item When $\nu>1$, \cref{algo:fd} achieves
	\begin{align*}
		r^\ast_T=\Otil\big(\max\big\{BL^{(m-1)\nu}T^{-\nu/d},
		B^{1+\nu+\nu^2+\cdots+\nu^{m-2}}L^{\nu^{m-1}}T^{-\nu^2/d} \big\}\big).
	\end{align*}
\end{itemize}
\end{restatable}
The proof is given in \cref{sec:fd_chain}, and the optimality will be discussed in \cref{sec:discussion}.
When $\nu>1$, the simple regret upper bound takes the maximum of two terms. The first term has a smaller constant factor, while the second term has a smaller $T$-dependent factor. By taking the highest-order constant factor and the highest-order $T$-dependent factor, we can deduce the weaker but simpler bound  $r^\ast_T=O(B^{1+\nu+\nu^2+\cdots+\nu^{m-2}}L^{\nu^{m-1}}T^{-\nu/d})$.

We also consider two more restrictive cases, where we remove the assumption of $B=\Theta(L)$, but have additional assumptions on $g$ as follows:
\begin{itemize} \itemsep0ex
	\item \textbf{Case 1}: We additionally assume that $\muu{i}_T\circ\cdots\circ\muu{1}_T(\xbf^\ast)\in\Xc^{(i+1)}$ and $\muu{i}_T\circ\cdots\circ\muu{1}_T(\xbf^\ast_T)\in \Xc^{(i+1)}$ for all $i\in[m-1]$.
	\item \textbf{Case 2}: We additionally assume that all the domains $\Xc^{(i)}$ are known. Defining
%	$\muutil{i}_T(\zbf) = \argmin_{\zbf'\in\Xc^{(i+1)}}|\muu{i}_T(\zbf)-\zbf'|,$
	\begin{align*}
		\muutil{i}_T(\zbf) = \argmin_{\zbf'\in\Xc^{(i+1)}}|\muu{i}_T(\zbf)-\zbf'|,
	\end{align*}
	we slightly modify the algorithm to return
%	$\xbf^\ast_T=\argmax_{\xbf\in\Xc}(\muutil{m}_T\circ\cdots\circ\muutil{1}_T)(\xbf).$
	\begin{align*}
		\xbf^\ast_T=\argmax_{\xbf\in\Xc}(\muutil{m}_T\circ\cdots\circ\muutil{1}_T)(\xbf).
	\end{align*}
\end{itemize}
\begin{rem}
	Under the assumptions of either Case 1 or Case 2, \cref{algo:fd} achieves for chains that
	\begin{align}
		r^\ast_T=
		\begin{cases}
			O(BL^{m-1}T^{-\nu/d})&\text{ when $\nu\le 1$,}\\
			O(BL^{(m-1)\nu}T^{-\nu/d})&\text{ when $\nu> 1$.}
		\end{cases} 
        \label{eq:more_restrictive}
	\end{align} 
\end{rem}
The proof is given in \cref{sec:rest_chain}.

The composite posterior means and simple regret upper bounds for multi-output chains and feed-forward networks are provided in \cref{sec:fd_mul} and \cref{sec:fd_net} respectively, where the simple regret upper bounds are stated only for the case that the domain of each layer is a hyperrectangle. Removing this restrictive assumption is left for future work.

%\subsection{Non-Adaptive Sampling Method for Multi-Output Chains}
%For multi-output chains, the composite posterior mean of $g(\xbf)$ is
%\begin{align}
%	\mu_T^g(\xbf)=(\muu{m}_T\circ\muu{m-1}_T\circ\cdots\circ\muu{1}_T)(\xbf),\label{eq:mean_mul}
%\end{align}
%where $\muu{i}_T$ denotes the posterior mean of $\ff{i}$ computed using \eqref{eq:multi_mean} based on $\{(\xx{i}_s,\xx{i+1}_s)\}_{s=1}^T$ for each $i\in[m]$. 
%
%Assuming each $\Xc^{(i)}$ is a hyperrectangle, the simple regret upper bound for multi-output chains is provided in \cref{sec:fd_mul}.
%
%\subsection{Non-Adaptive Sampling Method for Feed-Forward Networks}
%For feed-forward networks of scalar-valued functions, let $\muu{i,j}_T$ denote the posterior mean of $\ff{i,j}$ computed using \eqref{eq:mean} based on $\{(\xx{i}_s,\xx{i+1,j}_s)\}_{s=1}^T$. The composite posterior mean of $g(\xbf)$ with feed-forward network structure is
%\begin{align}
%	\mu_T^g(\xbf)=\muu{m,1}_T(\zz{m})\label{eq:mean_net}
%\end{align}
%with
%\begin{align*}
%	\zz{1}&=\xbf,\\
%	\zz{i+1,j}&=\muu{i,j}_T(\zz{i})&&\hspace{-10ex}\text{ for }i\in[m-1],j\in[d_{i+1}],\\
%	\zz{i+1}&=\muu{i}_T(\zz{i})\\
%	&=[\zz{i+1,1},\dots,\zz{i+1,d_i}]&&\text{ for }i\in[m-1].
%\end{align*}
%
%Assuming each $\Xc^{(i)}$ is a hyperrectangle, the simple regret upper bound for feed-forward networks is provided in \cref{sec:fd_net}.

\section{Algorithm-Independent Lower Bounds}
\label{sec:lower}

In this section, we provide algorithm-independent lower bounds on the simple regret and cumulative regret for any algorithm optimizing chains, multi-output chains, or feed-forward networks for the scalar-valued kernel $k_\text{Mat\'ern}$ or the operator-valued $\Gamma_\text{Mat\'ern}(\cdot,\cdot)=k_\text{Mat\'ern}(\cdot,\cdot)\mathbf{I}$ with smoothness $\nu$. 

\begin{restatable}[Lower bound on simple regret]{thm}{lowersim}
\label{thm:lower_sim}
Fix $\epsilon\in(0,\frac{1}{2}]$, sufficiently large $B>0$, $k=k_\text{Mat\'ern}$, and $\Gamma=\Gamma_\text{Mat\'ern}$ with smoothness $\nu\ge 1$. Suppose that there exists an algorithm (possibly randomized) that achieves average simple regret $\EE[r_T^\ast]\le\epsilon$ after $T$ rounds for any $m$-layer chain, multi-output chain, or feed-forward network on $[0,1]^d$ with some $L=\Theta(B)$. Then, provided that $\frac{\epsilon}{B}$ is sufficiently small, it is necessary that
% $T=\Omega\Big(\big(B(cL)^{m-1}/\epsilon\big)^{d/\nu}\Big)$
\begin{align*}
    T=\Omega\Bigg(\Big(\frac{B(cL)^{m-1}}{\epsilon}\Big)^{d/\nu}\Bigg)
\end{align*}
for some $c=\Theta(1)$.
\end{restatable}

The proof is given in \cref{sec:lower_sim}, and the high-level steps are similar to \cite{bull2011convergence}, but the main differences are significant. For each structure, we consider a collection of $M$ hard functions $\Gc=\{g_1,\dots,g_M\}$, where each $g_j$ is obtained by shifting a base function $\gbar$ of the specified structure and cropping the shifted function into $[0,1]^d$. Then, we show that there exists a worst-case function in $\Gc$ with the provided lower bound. Different from \cite{bull2011convergence}, the hard functions we construct here are function networks. We define the first layer as a ``needle'' function with much smaller height and width than \cite{bull2011convergence}. For subsequent layers, we construct a function with corresponding RKHS norm such that the output is always larger than the input. As a consequence, the ``needle'' function gets higher when being fed into subsequent layers, and the composite function is a ``needle'' function with some specified height but a much smaller width.

\begin{rem}
It will be evident from the proof that the constant $c$ is always strictly less than $1$.  Ideally, we would like it to be close to $1$ so that $(cL)^m$ is similar to $L^m$, with the latter quantity appearing in our upper bounds.  It turns out that $c$ can indeed be arbitrarily close to $1$ in most cases.  Specifically, we will show in \cref{sec:dis_c} that when $\nu > 1$ and $\epsilon$ is small enough, $c$ simply becomes  the ratio of the minimum slope to the maximum slope of the kernel function (as a function of the Euclidean distance $\|\xbf - \xbf'\|$) on $[u-\util,u+\util]$, where $u$ and $\util$ can be arbitrarily small.  Since the squared exponential (SE) and Mat\'ern kernels have no sharp changes as a function of $\|\xbf - \xbf'\|$, this ensures that $c$ can be arbitrarily close to one when $\nu > 1$ and $\epsilon$ is small.  In \cref{sec:dis_c}, we will also demonstrate cases where $c$ is not too small (e.g., $c > 0.93$) even when the above-mentioned quantities $(u,\util)$ are moderate (e.g., $(u,\util) = (0.5,0.3)$).
\end{rem}

The lower bound on simple regret readily implies the following lower bound on cumulative regret.
\begin{restatable}[Lower bound on cumulative regret]{thm}{lowercml}
\label{thm:lower_cml}
Fix sufficiently large $B>0$, $k=k_\text{Mat\'ern}$, and $\Gamma=\Gamma_\text{Mat\'ern}$ with smoothness $\nu\ge 1$. Suppose that there exists an algorithm (possibly randomized) that achieves average cumulative regret $\EE[R_T]$ after $T$ rounds for any $m$-layer chain, multi-output chain, or feed-forward network on $[0,1]^d$ with some $L=\Theta(B)$. Then, it is necessary that
\begin{align*}
    \EE[R_T]=\begin{cases}
    	\Omega\big(\min\{T, B(cL)^{m-1}T^{1-\nu/d}\}\big)&\text{ when }d>\nu,\\
    	\Omega\big(\min\{T, \big(B(cL)^{m-1}\big)^{d/\nu}\}\big)&\text{ when }d\le\nu,
    \end{cases}
\end{align*}
for some $c=\Theta(1)$.
\end{restatable}
The proof is given in \cref{sec:lower_cml}.

\section{Comparison of Bounds}
\label{sec:discussion}

In this section, we compare the algorithmic upper bounds of GPN-UCB (\cref{algo:ucb}) and non-adaptive sampling (\cref{algo:fd}) to the algorithmic-independent lower bounds in \cref{sec:lower}.  We present our discussion conditioned on the conjecture of \cite{vakili2022open} being true, but we re-iterate that even without this, any $\Sigma_T$ dependence still matches the vanilla setting, and has known rigorous upper bounds as detailed in \cref{sec:ucb_disc},
A table summarizing the regret bounds for the Mat\'ern kernel is provided in \cref{sec:summary}. 

For GPN-UCB, the cumulative regret upper bound for chains (\cref{thm:ucb_chain}) matches the lower bound (\cref{thm:lower_cml}) up to a $2^m$ factor when $d\ge\nu\ge1$ and $T=\Omega\big((B(cL)^{m-1})^{d/\nu}\big)$. The upper bound for multi-output chains (\cref{thm:ucb_mul}) is similarly optimal (up to a $5^m$ term) when $d_{\max}=d\ge\nu\ge1$, while there is always an $O(\sqrt{D_{2,m}})$ gap for the $T$-independent factor of feed-forward networks (\cref{thm:ucb_net}). When $d<\nu$ and $T=\Omega\big((B(cL)^{m-1})^{d/\nu}\big)$, the cumulative regret lower bound for all the three structures is $\Omega\big((B(cL)^{m-1})^{d/\nu}\big)$, while the upper bound always contains an $O(BL^{m-1})$ factor; hence, the terms behave similarly but some gaps still remain.

We expect that the discrepancies for multi-output chains and feed-forward networks are due to the looseness of the proposed lower bound. Since the hard functions $\Gc$ used in analysis always produce a single-entry vector output for intermediate layers, for a fixed value of $B$, there might exist a worse hard function network with more nonzero entries for intermediate outputs and a probably higher final regret.

For non-adaptive sampling, when $\nu=1$, the upper bound for chains (\cref{thm:fd_chain}) matches the lower bound (\cref{thm:lower_sim}) up to a $c^{m-1}$ factor. When $\nu>1$, \cref{thm:fd_chain} shows that the simple regret upper bound takes the maximum of two terms, where the first term has a matched $T$-dependent factor. However, both terms have a larger $T$-independent factor than the lower bound when $\nu>1$; this arises due to magnifying the uncertainty from each layer to the next.

% This is caused by the increasing posterior variance as the sampled points passing through each layer.

\section{Conclusion} \label{sec:conclusion}

We have proposed an upper confidence bound based method GPN-UCB and a non-adaptive sampling based method for optimizing chains, multi-output chains, and feed-forward networks in the noise-free grey-box setting.  Our regret bounds significantly improve certain dependencies compared to previous works, and we provide lower bounds that are near-matching in broad cases of interest. An immediate direction for future work is to explore noisy extensions of our algorithms (as well as lower bounds), ideally attaining analogous improvements over existing works \cite{kusakawa2021bayesian,sussex2023modelbased} as those that we attained in the noiseless setting (as discussed in Appendix \ref{sec:comparison}).
% An interesting direction for future work would be to further explore the noisy setting beyond our GPN-UCB extension in Section \ref{sec:noisy}, e.g., by considering non-adaptive methods or lower bounds.

% \blue{We provided cumulative regret upper bounds along with a noisy version for the former and simple regret upper bounds for the latter. The algorithm-independent lower bounds on simple regret and cumulative regret show that our algorithms attain near-optimal scaling in broad cases of interest.}

\subsubsection*{Acknowledgment}
This work was supported by the Singapore Ministry of Education Academic Research Fund Tier 1 under grant number A-8000872-00-00.

\newpage
\bibliography{CascadeOpt_References}

\newpage
\appendix

{\LARGE\bf\sffamily Appendix \par}
% {\Large\bf\sffamily Regret Bounds for Noise-Free Cascaded Kernelized Bandits \par}
%\medskip
%All citations in this appendix are to the reference list in the main body.
\section{Applications of the Cascaded Setting}
\label{sec:motivation}
Optimization of composite functions (cascaded functions, or function networks) has wide applications in optimizing multi-stage processes, where the output of the current stage is the input of the next stage. For example, a real-world application of grey-box composite function optimization in material science is alloy heat treatment, which consists of multiple heat treatment steps and the resulting hardness after each step is available. The objective is to find the starting concentration and heat treatment (temperature and time) that maximize the hardness of the final product \cite{nguyen2016cascade}. Note that to find the best heat treatment for all the steps, the algorithms are expected to support additional input for intermediate layers, and we will show in \cref{sec:comparison} that our algorithms and theories can be easily adapted to support this. 

Similarly, as an application in simulation, solar cell simulators can also utilize the technique of composite function optimization to maximize the power generation efficiency \cite{kusakawa2021bayesian}. As discussed by \cite{astudillo2021bayesian}, composite functions also arise in numerous areas e.g., engineering design, material design, system design, reinforcement learning, and Markov decision processes. Many analogous studies for white-box settings also exist \cite{drusvyatskiy2019efficiency,wang2017stochastic} with various roles highlighted in learning tasks; black-box variants then become indispensable when gradients are unavailable.

\section{Reproducing Kernel Hilbert Space (RKHS)}
\label{sec:rkhs}
\subsection{Scalar-Valued Functions}
For a given scalar-valued kernel $k:\Xc\times\Xc\to\RR$, consider a function space $\Sc_k:=\{f(\cdot)=\sum_{i=1}^{n_0} a_i k(\cdot,\xbf_i): n_0\in\NN, a_i\in\RR, \xbf_i\in \Xc\}$. Then, the reproducing kernel Hilbert space (RKHS) corresponding to kernel $k$, denoted by $\Hc_k$, can be obtained by forming the completion of $\Sc_k$, and the elements in $\Hc_k$ are called scalar-valued kernelized bandits or GP bandits. 
$\Hc_k$ is equipped with the inner product
\begin{align}
	\langle f, f' \rangle_k = \sum_{i=1}^{n_1} \sum_{j=1}^{n_2} a_i b_j k(\xbf_i,\xbf'_j)\label{eq:norm}
\end{align}
for $f=\sum_{i=1}^{n_1} a_i k(\cdot,\xbf_i)$ and $f'=\sum_{j=1}^{n_2} b_j k(\cdot,\xbf'_j)$. This inner product satisfies the reproducing property, such that $\langle f(\cdot), k(\cdot,\xbf)\rangle_k=f(\xbf),\forall \xbf\in\Xc,\forall f\in\Hc_k$. The RKHS norm of $f$ is $\|f\|_k=\sqrt{\langle f, f\rangle_k}$, and we use $\Hc_k(B):=\{f\in\Hc_k:\|f\|_k\le B\}$ to denote the set of functions whose RKHS norm is upper bounded by some known constant $B>0$. In this work, we mainly focus on the Mat\'ern kernel:                           
\begin{align*}
	% 	k_\text{SE}(\xbf,\xbf')&=\exp\Bigg(-\frac{(d_{\xbf,\xbf'})^2}{2l^2}\Bigg), \\
	k_\text{Mat\'ern}(\xbf,\xbf')&=\frac{2^{1-\nu}}{\Gamma(\nu)}\Bigg(\frac{\sqrt{2\nu}d_{\xbf,\xbf'}}{l}\Bigg)^\nu B_\nu \Bigg(\frac{\sqrt{2\nu}d_{\xbf,\xbf'}}{l}\Bigg),
\end{align*}
where $d_{\xbf,\xbf'}=\| \xbf-\xbf'\|_2$, $l>0$ denotes the length-scale, $\nu>0$ is a smoothness parameter, $\Gamma$ is the Gamma function, and $B_\nu$ is the modified Bessel function.

\subsection{Vector-Valued Functions}
An operator-valued kernel $\Gamma:\Xc\times\Xc\to\RR^{n\times n}$ is called a multi-task kernel on $\Xc$ if $\Gamma(\cdot,\cdot)$ is symmetric positive definite. Moreover, a single-task kernel with $n=1$ recovers a scalar-valued kernel. For a given multi-task kernel $\Gamma$, similarly to the scalar-valued kernels, there exists an RKHS of vector-valued functions $\Hc_\Gamma$, which is the completion of $\Sc_\Gamma:=\{f(\cdot)=\sum_{i=1}^{n_0} \Gamma(\cdot,\xbf_i)\abf_i: n_0\in\NN, \abf_i\in\RR^n, \xbf_i\in \Xc\}$. The elements in $\Hc_\Gamma$ are called vector-valued kernelized bandits, and $\Hc_\Gamma$ is equipped with the inner product \cite{carmeli2006vector}
\begin{align}
	\langle f,f'\rangle_\Gamma=\sum_{i=1}^{n_1}\sum_{j=1}^{n_2}\langle \Gamma(\xbf_i,\xbf_j') \abf_i,\bbf_j\rangle\label{eq:mul_norm}
\end{align}
for $f=\sum_{i=1}^{n_1}\Gamma(\cdot,\xbf_i)\abf_i$ and $f'=\sum_{j=1}^{n_2}\Gamma(\cdot,\xbf'_j)\bbf_i$. This inner product satisfies the reproducing property, such that $\langle f(\cdot),\Gamma(\cdot,\xbf)\vbf\rangle_\Gamma=\langle f(\xbf),\vbf\rangle_2,\forall\xbf\in\Xc,\forall \vbf\in\RR^n,\forall f\in\Hc_\Gamma.$ The RKHS norm of $f\in\Hc_\Gamma$ is $\|f\|_\Gamma=\sqrt{\langle f,f\rangle_\Gamma}$, and we focus on $\Hc_\Gamma(B):=\{f\in\Hc_\Gamma:\|f\|_\Gamma\le B\}$ containing functions with norm at most $B>0$.  We will often pay particular attention to the Mat\'ern kernel $\Gamma_\text{Mat\'ern}(\cdot,\cdot)=k_\text{Mat\'ern}(\cdot,\cdot)\Ibf$, where $k_\text{Mat\'ern}$ is the scalar-valued Mat\'ern kernel and $\Ibf$ is the $n\times n$ identity matrix.

\section{Figures Illustrating the Network Structures} \label{app:figs}

We depict the chain structure in \cref{fig:chain}, the multi-output chain structure in \cref{fig:mul}, and the feed-forward network structure in \cref{fig:net}.

\begin{figure}[t!]
	\centering
	\includegraphics[width=0.5\linewidth]{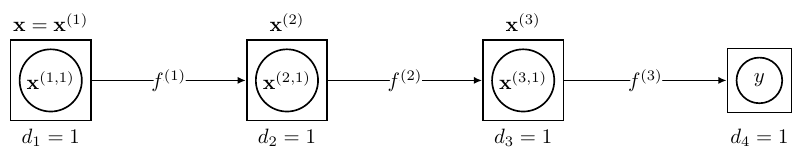}
	\caption{A chain with $m=3$, where $\{\ff{i}\}_{i=1}^3$ are scalar-valued functions.}
	\label{fig:chain}
	\includegraphics[width=0.5\linewidth]{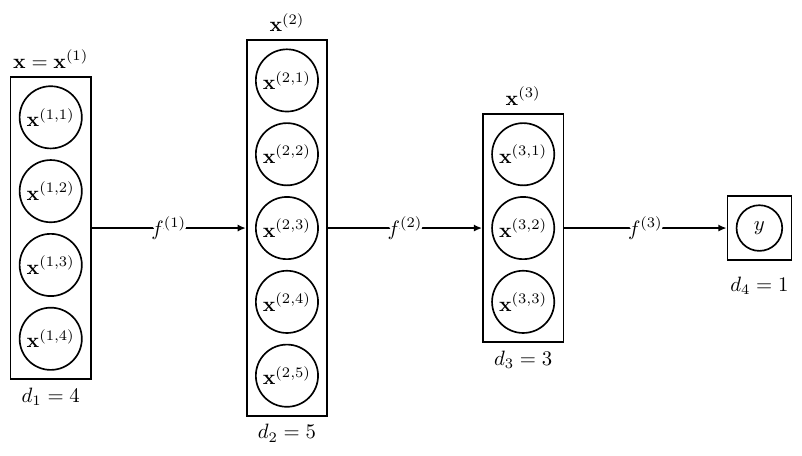}
	\caption{A multi-output chain with $m=3$, where $\{\ff{i}\}_{i=1}^3$ are vector-valued functions.}
	\label{fig:mul}
	\includegraphics[width=0.5\linewidth]{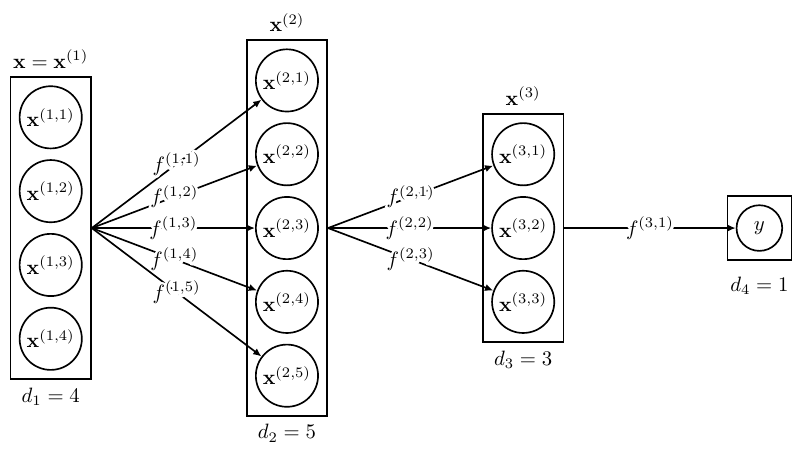}
	\caption{A feed-forward network with $m=3$, where $\ff{i,j}$ is a scalar-valued function for each $i\in[3],j\in[d_{i+1}]$.}
	\label{fig:net}
\end{figure}

\section{Confidence Region for Vector-Valued Functions (Proof of \cref{lem:mul})}
\label{sec:cr_mul}
In this section, we prove \cref{lem:mul}, which is restated as follows.
\mulconf*
\begin{proof}
We first review the GP posterior and confidence region for operator-valued kernel $\Gamma:\Xc\times\Xc\to\RR^{n\times n}$ in the noisy setting.

\begin{restatable}{lem}{confmul}
\textup{\citep[Theorem 1]{chowdhury2021no}}
\label{lem:mul_conf}
For $f\in\Hc_\Gamma(B)$, given a sequence of points $(\xbf_1,\dots,\xbf_t)$ and their noisy observations $Y'_t=(\ybf'_1,\dots,\ybf'_t)$, where $\ybf'_i=f(\xbf_i)+\epsilon_i$ with $\epsilon_i$ being i.i.d. $\sigma$-sub-Gaussian for each $i\in[t]$ for some $\sigma>0$, let $\mu'_t(\xbf)$ and $\Gamma'_t(\xbf,\xbf)$ denote the posterior mean and variance computed using
\begin{align}
    \mu'_t(\xbf)&=G_t(\xbf)^T (G_t+\lambda \Ibf_{nt})^{-1}Y'_t,\label{eq:mu_noisy}\\
    \Gamma'_t(\xbf,\xbf')&=\Gamma(\xbf,\xbf')-G_t(\xbf)^T (G_t+\lambda \Ibf_{nt})^{-1}G_t(\xbf'),\label{eq:gamma_noisy}
\end{align}
where $G_t(\xbf)=[\Gamma(\xbf,\xbf_i)]_{i=1}^t \in \RR^{nt\times n}$, $G_t=[\Gamma(\xbf_i,\xbf_j)]_{i,j=1}^t \in \RR^{nt\times nt}$, $Y'_t=[\ybf'_i]_{i=1}^t\in\RR^{nt\times 1}$, and $\lambda>0$ is a regularization parameter.
Then, for any $\lambda>0$ and $\delta\in(0,1]$, with probability at least $1-\delta$, it holds for all $\xbf\in\Xc$ that
\begin{align}
    \|f(\xbf)-\mu'_t(\xbf)\|_2\le \alpha_t\|\Gamma'_t(\xbf,\xbf)\|_2^{1/2}
\end{align}
with $\alpha_t=B+\frac{\sigma}{\sqrt{n}}\sqrt{2\log(1/\delta)+\log\det(I_{nt}+\lambda^{-1}G_t)}$.
\end{restatable}
Since zero noise is $\sigma$-sub-Gaussian for any $\sigma>0$, by setting $\sigma=\lambda=\frac{1}{a}$ and then taking $a\to\infty$, \eqref{eq:mu_noisy} and \eqref{eq:gamma_noisy} converge to
\begin{align}
    \mu_t(\xbf)&=G_t(\xbf)^T G_t^{-1}Y_t,\\
    \Gamma_t(\xbf,\xbf')&=\Gamma(\xbf,\xbf')-G_t(\xbf)^T G_t^{-1}G_t(\xbf'),
\end{align}
with $Y_t=(\ybf_1,\dots, \ybf_t)$, thus yielding the posterior mean and variance based on $(\xbf_1,\dots,\xbf_t)$ and their noise-free observations $Y_t$.

Then, by setting $\sigma=\delta=\lambda=\frac{1}{a}$ for $a\to\infty$, with $\{\lambda_i\}_{i=1}^{nt}$ being the eigenvalues of $G_t$, we obtain
\begin{align}
	\lim\limits_{a\to\infty}\alpha_t&=B+\frac{1}{\sqrt{n}} \cdot \lim\limits_{a\to\infty} \sqrt{\frac{2\log a}{a^2}+\frac{\sum_{i=1}^{nt} \log(a\lambda_i+1)}{a^2}} = B.
%	&=B+\frac{1}{\sqrt{n}} \cdot \lim\limits_{a\to\infty}  \sqrt{ \frac{1}{a^2} + \sum_{i=1}^{nt} \frac{\lambda_i}{2a(a\lambda_i+1)}}\tag{by L'H\^{o}pital's rule}\\
\end{align}
Hence, we obtain \cref{lem:mul} for the deterministic confidence region based on noise-free observations.
\end{proof}

\section{Posterior Variance for Vector-Valued Functions}
\label{sec:multi_var}

In this section, we state and prove \cref{lem:multi_var}.

\begin{restatable}{lem}{multivar}
\label{lem:multi_var}
For a scalar-valued kernel $k$, define $\Gamma(\xbf,\xbf')=k(\xbf,\xbf')\Ibf_n$ with $\Ibf_n$ being the $n\times n$ identity matrix. For $f_1\in\Hc_k(B)$ and $f_2\in\Hc_\Gamma(B)$ with domain $\Xc$ and constant $B>0$, let $\sigma_t(\xbf)^2$ and $\Gamma_t(\xbf,\xbf)$ denote the posterior variance (matrix) for $f_1$ and $f_2$ based on $t$ points $(\xbf_1,\dots,\xbf_t)$ in the noise-free setting computed using \eqref{eq:var} and \eqref{eq:multi_var}. Then, it holds for all $\xbf\in\Xc$ that
\begin{align}
    \Gamma_t(\xbf,\xbf)=\sigma_t(\xbf)^2\Ibf_n.
\end{align}
\end{restatable}
\begin{proof}
Let $\sigma'_t(\xbf)^2$ and $\Gamma'_t(\xbf,\xbf)$ denote the posterior variance (matrix) for $f_1$ and $f_2$ based on $t$ points $(\xbf_1,\dots,\xbf_t)$ in the noisy setting. For any $\lambda>0$, we have
\begin{align}
    \sigma'_t(\xbf)^2  &= k(\xbf,\xbf) - \mathbf{k}_t(\xbf)^T (\mathbf{K}_t+\lambda\mathbf{I}_t)^{-1}\mathbf{k}_t(\xbf),\\
    \Gamma'_t(\xbf,\xbf)&=\Gamma(\xbf,\xbf)-G_t(\xbf)^T (G_t+\lambda \Ibf_{nt})^{-1}G_t(\xbf'),\label{eq:multi_var_noisy}
\end{align}
where $\mathbf{k}_t(\xbf) = [k(\xbf, \xbf_i)]_{i=1}^{t}\in\RR^{t\times 1}$, $\mathbf{K}_t = [k(\xbf_i, \xbf_j)]_{i,j=1}^t\in\RR^{t\times t}$, $G_t(\xbf)=[\Gamma(\xbf,\xbf_i)]_{i=1}^t \in \RR^{nt\times n}$, and $G_t=[\Gamma(\xbf_i,\xbf_j)]_{i,j=1}^t \in \RR^{nt\times nt}$.
With $\otimes$ denoting the Kronecker product, we have
$G_t(\xbf)= \kbf_t(\xbf)\otimes\Ibf_n$ and $G_t= \Kbf_t\otimes\Ibf_n$. Then, it follows from \eqref{eq:multi_var_noisy} that
\begin{align}
    \Gamma'_t(\xbf,\xbf)&=\Gamma(\xbf,\xbf)-G_t(\xbf)^T(G_t+\lambda \Ibf_{nt})^{-1}G_t(\xbf)\\
    &=k(\xbf,\xbf)\Ibf_n-\big(\kbf_t(\xbf)^T\otimes\Ibf_n\big)(\mathbf{K}_t\otimes \Ibf_{n}+\lambda\Ibf_{nt})^{-1}\big(\kbf_t(\xbf)\otimes\Ibf_n\big)\\
    &=k(\xbf,\xbf)\Ibf_n-\big(\kbf_t(\xbf)^T\otimes\Ibf_n\big)\big((\mathbf{K}_t+\lambda \Ibf_t)\otimes \Ibf_{n}\big)^{-1}\big(\kbf_t(\xbf)\otimes\Ibf_n\big)\\
    &=k(\xbf,\xbf)\Ibf_n-\big(\kbf_t(\xbf)^T\otimes\Ibf_n\big)\big((\Kbf_t+\lambda \Ibf_t)^{-1}\otimes \Ibf_{n}\big)\big(\kbf_t(\xbf)\otimes\Ibf_n\big)\\
    &=k(\xbf,\xbf)\Ibf_n-\big(\kbf_t(\xbf)^T(\Kbf_t+\lambda\Ibf_t)^{-1}\kbf_t(\xbf)\big)\otimes\Ibf_n\\
    &=k(\xbf,\xbf)\Ibf_n-\big(\kbf_t(\xbf)^T(\Kbf_t+\lambda\Ibf_t)^{-1}\kbf_t(\xbf)\big)\Ibf_n\\
    &=\big(k(\xbf,\xbf)-(\kbf_t(\xbf)^T(\Kbf_t+\lambda\Ibf_t)^{-1}\kbf_t(\xbf)\big)\Ibf_n\\
    &=\sigma'_t(\xbf)^2\Ibf_n
\end{align}
Taking $\lambda\to0$, we obtain in the noise-free setting that
\begin{align}
    \Gamma_t(\xbf,\xbf)=\sigma_t(\xbf)^2\Ibf_n.
\end{align}
\end{proof}

\section{Analysis of GPN-UCB (\cref{algo:ucb})}
\subsection{Proof of \cref{thm:ucb_chain} (Chains)}
\label{sec:ucb_chain}
In this section, we prove \cref{thm:ucb_chain}, which is restated as follows.
\ucbchain*

\begin{proof}
With $\xbfbar_t^{(i)}=\argmax_{\zbf\in\del{i}_{t-1}(\xbf_t)}\UCBBB{i}(\zbf)$ and $\xbftil_t^{(i)}=\argmin_{\zbf\in\del{i}_{t-1}(\xbf_t)}\LCBBB{i}(\zbf)$, the simple regret is upper bounded as follows:
\begin{align}
    r_t&=g(\xbf^\ast)-g(\xbf_t) \label{eq:sim_0} \\
    &\le \UCB_{t-1}(\xbf_t)-\LCBBB{m}_{t-1}(\xbftil_t^{(m)})\label{eq:sim_1}\\
    &= \UCBBB{m}_{t-1}(\xbfbar_t^{(m)})-\LCBBB{m}_{t-1}(\xbftil_t^{(m)}) \label{eq:sim_2}\\
    &\le \UCBB{m}_{t-1}(\xx{m}_t)+L\|\xbfbar_t^{(m)}-\xx{m}_t\|_2 -\LCBB{m}_{t-1}(\xx{m}_t)+L\|\xx{m}_t-\xbftil_t^{(m)}\|_2\label{eq:sim_3}\\
    &\le 2B\sig{m}_{t-1}(\xx{m}_t)+2L\cdot\dia{\del{m}_{t-1}(\xbf_t)}\label{eq:sim_4}\\
    &\le 2B\Big(
    \sig{m}_{t-1}(\xx{m}_t)
    +(2L)\sig{m-1}_{t-1}(\xx{m-1}_t) +\cdots+(2L)^{m-1}\sig{1}_{t-1}(\xx{1}_t)
    \Big)\label{eq:sim_5}\\
    &= 2B\sum_{i=1}^m(2L)^{m-i}\sig{i}_{t-1}(\xx{i}_t),\label{eq:sim_6}
\end{align}
where:
\begin{itemize}
    \item \eqref{eq:sim_1} follows since $g(\xbf^\ast)\le \UCB_{t-1}(\xbf^\ast)\le \UCB_{t-1}(\xbf_t)$ (due to the algorithm maximizing the UCB score) and $g(\xbf_t)\ge\LCBBB{m}_{t-1}(\xbf_t^{(m)})\ge \LCBBB{m}_{t-1}(\xbftil_t^{(m)})$;
    \item \eqref{eq:sim_2} follows from \eqref{eq:ucb_chain} and the definition of $\xbfbar_t^{(m)}$;
    \item \eqref{eq:sim_3} follows from \eqref{eq:ucb_i} and \eqref{eq:lcb_i};
    \item \eqref{eq:sim_4} follows by defining $\dia{\del{i}_{t-1}(\xbf_t)}$ as the diameter of confidence region $\del{i}_{t-1}(\xbf_t)$;
    \item \eqref{eq:sim_5} follows from the following recursion: 
    \begin{align}
    \dia{\del{i+1}_{t-1}(\xbf_t)}&= \UCBBB{i}_{t-1}(\xbfbar^{(i)}_t)-\LCBBB{i}_{t-1}(\xbftil^{(i)}_t)\\
    &\le 2B\sig{i}_{t-1}(\xx{i}_t)+2L\cdot\dia{\del{i}_{t-1}(\xbf_t)},
    \end{align}
    where the inequality is obtained by following \eqref{eq:sim_2}-\eqref{eq:sim_4} with $m$ replaced by $i$.
\end{itemize}
Then, the cumulative regret is
\begin{align}
    R_T&=\sum_{t=1}^T r_t\\
    &\le 2B \sum_{t=1}^T \sum_{i=1}^m(2L)^{m-i}\sig{i}_{t-1}(\xx{i}_t)\\
    &= 2B \frac{(2L)^m-1}{2L-1} \Sigma_T\\
    &\le 2^{m+1} BL^{m-1}\Sigma_T,
\end{align}
where $\Sigma_T=\max\limits_{i\in[m]}\max\limits_{\zbf_1,\dots,\zbf_T\in\Xc^{(i)}}\sum_{t=1}^T\sig{i}_{t-1}(\zbf_t)$.
\end{proof}

\subsection{Proof of \cref{thm:ucb_mul} (Multi-Output Chains)}
\label{sec:ucb_mul}
In this section, we prove \cref{thm:ucb_mul}, which is restated as follows.
\ucbmul*
\begin{proof}
For each $i\in [m-1]$, there must exist $\xbfbar^{(i)}_t,\xbftil^{(i)}_t\in\del{i}_{t-1}(\xbf_t)$ such that the following upper bound on $\dia{\del{i+1}_{t-1}(\xbf_t)}$ in terms of $\dia{\del{i}_{t-1}(\xbf_t)}$ holds:
\begin{align}
    \dia{\del{i+1}_{t-1}(\xbf_t)} &\le \|\ff{i}(\xbfbar^{(i)}_t)-\ff{i}(\xbftil^{(i)}_t)\|_2+\dia{\overbar{\Cc}^{(i)}_{t-1}(\xbfbar^{(i)}_t)}+\dia{\overbar{\Cc}^{(i)}_{t-1}(\xbftil^{(i)}_t)}\label{eq:multi_del_1}\\
    &\le L\cdot\dia{\del{i}_{t-1}(\xbf_t)}+\dia{\Cc^{(i)}_{t-1}(\xbfbar^{(i)}_t,\xx{i}_t)}+\dia{\Cc^{(i)}_{t-1}(\xbftil^{(i)}_t,\xx{i}_t)}\label{eq:multi_del_2}\\
    &\le L\cdot\dia{\del{i}_{t-1}(\xbf_t)}+2\cdot\dia{\Cc^{(i)}_{t-1}(\xx{i}_t)}+2L\|\xbfbar^{(i)}_t-\xx{i}_t\|_2+2L\|\xbftil^{(i)}_t-\xx{i}_t\|_2\label{eq:multi_del_3}\\
    &\le 4B\|\Gamma_{t-1}^{(i)}(\xx{i}_t,\xx{i}_t)\|_2^{1/2}+5L\cdot\dia{\del{i}_{t-1}(\xbf_t)}, \label{eq:multi_del_4}
\end{align}
where:
\begin{itemize}
    \item \eqref{eq:multi_del_1} holds since there must exist $\xbfbar^{(i)}_t,\xbftil^{(i)}_t\in\del{i}_{t-1}(\xbf_t)$ such that $\dia{\del{i+1}_{t-1}(\xbf_t)}=\|\zbfbar-\zbftil\|_2$ for some $\zbfbar\in\overbar{\Cc}^{(i)}_{t-1}(\xbfbar^{(i)}_t)$ and some $\zbftil\in\overbar{\Cc}^{(i)}_{t-1}(\xbftil^{(i)}_t)$ (see \eqref{eq:del}), and the right hand side follows from the triangle inequality along with $\ff{i}(\xbfbar^{(i)}_t)\in \overbar{\Cc}^{(i)}_{t-1}(\xbfbar^{(i)}_t)$ and $\ff{i}(\xbftil^{(i)}_t)\in \overbar{\Cc}^{(i)}_{t-1}(\xbftil^{(i)}_t)$ (by \cref{lem:mul}). Here $\dia{\overbar{\Cc}^{(i)}_{t-1}(\cdot)}$ denotes the diameter of $\overbar{\Cc}^{(i)}_{t-1}(\cdot)$ (i.e., the Euclidean distance between the most distant pair of points in $\overbar{\Cc}^{(i)}_{t-1}(\cdot)$);
    \item \eqref{eq:multi_del_2} holds since $\ff{i}$ has Lipschitz constant $L$, and since $\xbfbar^{(i)}_t,\xbftil^{(i)}_t\in\del{i}_{t-1}(\xbf_t)$ and $\overbar{\Cc}^{(i)}_{t-1}(\zbf)\subseteq\Cc^{(i)}_{t-1}(\zbf,\xx{i}_t) $ for all $\zbf\in\del{i}_{t-1}(\xbf_t)$ by the definition in \eqref{eq:cbarz};
    \item \eqref{eq:multi_del_3} holds since $\dia{\Cc^{(i)}_{t-1}(\zbf,\xx{i}_t)}\le \dia{\Cc^{(i)}_{t-1}(\xx{i}_t)}+2L\|\zbf-\xx{i}_t\|_2$ by the definition in \eqref{eq:czz};
    \item \eqref{eq:multi_del_4} holds since $\dia{\Cc^{(i)}_{t-1}(\xx{i}_t)}=2B\|\Gamma_{t-1}^{(i)}(\xx{i}_t,\xx{i}_t)\|_2^{1/2}$ (see \eqref{eq:cz}) and  $\xx{i}_t,\xbfbar^{(i)}_t,\xbftil^{(i)}_t\in\del{i}_{t-1}(\xbf_t)$.
\end{itemize}

Analogous to the UCB, we define $\LCB_t(\xbf)=\min_{\zbf\in\del{m}_t(\xbf)}\overbar{\Cc}_t^{(m)}(\zbf)$.  Moreover, we define
$\xbfbar^{(m)}_t=\argmax_{\zbf\in\del{m}_{t-1}(\xbf_t)}\big(\max \overbar{\Cc}^{(m)}_{t-1}(\zbf)\big)$, and $\xbftil^{(m)}_t=\argmin_{\zbf\in\del{m}_{t-1}(\xbf_t)} \big(\min\overbar{\Cc}^{(m)}_{t-1}(\zbf)\big)$.  Then, we have
\begin{align}
    r_t={}&g(\xbf^\ast)-g(\xbf_t)\\
    \le{}& \UCB_{t-1}(\xbf^\ast) - \LCB_{t-1}(\xbf_t)\\
    \le{}& \UCB_{t-1}(\xbf_t) - \LCB_{t-1}(\xbf_t)\\
    ={}& \max\overbar{\Cc}^{(m)}_{t-1}(\xbfbar^{(m)}_t) - \min \overbar{\Cc}^{(m)}_{t-1}(\xbftil^{(m)}_t) \\
    \le{}& \Big(\muu{m}_{t-1}(\xx{m}_t)+B\|\Gamma_{t-1}^{(m)}(\xx{m}_t,\xx{m}_t)\|_2^{1/2}+L\|\xbfbar^{(m)}_t-\xx{m}_t\|_2\Big)\nonumber\\
    &-\Big(\muu{m}_{t-1}(\xx{m}_t)-B\|\Gamma_{t-1}^{(m)}(\xx{m}_t,\xx{m}_t)\|_2^{1/2}-L\|\xbftil^{(m)}_t-\xx{m}_t\|_2\Big)\label{eq:rt_ucb_1}\\
    \le{}& 2B \|\Gamma_{t-1}^{(m)}(\xx{m}_t,\xx{m}_t)\|_2^{1/2}+2L\cdot\dia{\del{m}_{t-1}(\xbf_t)}\label{eq:rt_ucb_1a}\\
    \le{}& 2B \|\Gamma_{t-1}^{(m)}(\xx{m}_t,\xx{m}_t)\|_2^{1/2}+2L\Big( 4B \|\Gamma_{t-1}^{(m-1)}(\xx{m-1}_t,\xx{m-1}_t)\|_2^{1/2}+ 5L\cdot\dia{\del{m-1}_{t-1}(\xbf_t)}\Big)\label{eq:rt_ucb_2}\\
    \le{}& 2B \|\Gamma_{t-1}^{(m)}(\xx{m}_t,\xx{m}_t)\|_2^{1/2} 
    + \sum_{i=1}^{m-1} 2L \cdot (5L)^{m-1-i} \cdot 4B \|\Gamma_{t-1}^{(i)}(\xx{i}_t,\xx{i}_t)\|_2^{1/2}, \label{eq:rt_ucb_3}
\end{align}
where:
\begin{itemize}
    \item \eqref{eq:rt_ucb_1} holds since
    \begin{align}
        \max\overbar{\Cc}^{(m)}_{t-1}(\xbfbar^{(m)}_t)
        &\le \max\Cc^{(m)}_{t-1}(\xbfbar^{(m)}_t,\xx{m}_t)&&\text{(by \eqref{eq:cbarz})}\\
        &\le \max\Cc^{(m)}_{t-1}( \xx{m}_t) + L\|\xbfbar^{(m)}_t-\xx{m}_t\|_2 &&\text{(by \eqref{eq:czz})}\\
        &= \muu{m}_{t-1}(\xx{m}_t)+B\|\Gamma_{t-1}^{(m)}(\xx{m}_t,\xx{m}_t)\|_2^{1/2}+L\|\xbfbar^{(m)}_t-\xx{m}_t\|_2 &&\text{(by \eqref{eq:cz})}.
    \end{align}
    Similarly, we have
    \begin{align}
        \min\overbar{\Cc}^{(m)}_{t-1}(\xbftil^{(m)}_t)
        &\ge \min\Cc^{(m)}_{t-1}(\xbftil^{(m)}_t,\xx{m}_t)\\
        &\ge \min\Cc^{(m)}_{t-1}( \xx{m}_t) - L\|\xbftil^{(m)}_t-\xx{m}_t\|_2\\
        &= \muu{m}_{t-1}(\xx{m}_t)-B\|\Gamma_{t-1}^{(m)}(\xx{m}_t,\xx{m}_t)\|_2^{1/2}-L\|\xbftil^{(m)}_t-\xx{m}_t\|_2.
    \end{align}
    \item \eqref{eq:rt_ucb_1a} holds since all of the $\xbf$ vectors in \eqref{eq:rt_ucb_1a} lie in $\del{m}_{t-1}(\xbf_t)$;
    \item \eqref{eq:rt_ucb_2} and \eqref{eq:rt_ucb_3} follow from the recursive relation in \eqref{eq:multi_del_4}.
\end{itemize}
Then, the cumulative regret is
\begin{align}
    R_T={}&\sum_{t=1}^T r_t\\
    \le{}& 2B \sum_{t=1}^T\|\Gamma_{t-1}^{(m)}(\xx{m}_t,\xx{m}_t)\|_2^{1/2} 
    + 2L \cdot 4B \sum_{t=1}^T\|\Gamma_{t-1}^{(m-1)}(\xx{m-1}_t,\xx{m-1}_t)\|_2^{1/2}\nonumber\\
    {}&+ \sum_{i=1}^{m-2} 2L \cdot (5L)^{m-1-i} \cdot 4B \sum_{t=1}^T\|\Gamma_{t-1}^{(i)}(\xx{i}_t,\xx{i}_t)\|_2^{1/2} \\
    \le{}& 4B\frac{(5L)^m-1}{5L-1} \Sigma^\Gamma_T \\
    \le{}& 5^m BL^{m-1}\Sigma^\Gamma_T ,
\end{align}
where $\Sigma^\Gamma_T=\max\limits_{i\in[m]}\max\limits_{\zbf_1,\dots,\zbf_T\in\Xc^{(i)}} \sum_{t=1}^T\|\Gamma^{(i)}_{t-1}(\zbf_t,\zbf_t)\|_2^{1/2}$.
\end{proof}

\subsection{Proof of \cref{thm:ucb_net} (Feed-Forward Networks)}
\label{sec:ucb_net}
In this section, we prove \cref{thm:ucb_net}, which is restated as follows.
\ucbnet*
\begin{proof}
Recall the UCB and LCB definitions in \eqref{eq:ucb_ij_bar}--\eqref{eq:ucb_ij_final}.  
For a fixed input $\xbf\in\Xc$, defining $\xbfbar^{(i,j)}=\argmax_{\zbf\in\del{i}_t(\xbf)}\UCBBB{i,j}_t(\zbf)$ and $\xbftil^{(i,j)}=\argmin_{\zbf\in\del{i}_t(\xbf)}\LCBBB{i,j}_t(\zbf)$, the diameter of the confidence region of $\xx{i+1,j}$ is
\begin{align}
    \dia{\del{i+1,j}_t(\xbf)}&= \max_{\zbf\in\del{i}_t(\xbf)}\UCBBB{i,j}_t(\zbf) - \min_{\zbf\in\del{i}_t(\xbf)}\LCBBB{i,j}_t(\zbf) \\
    &= \UCBBB{i,j}_{t}(\xbfbar^{(i,j)})-\LCBBB{i,j}_{t}(\xbftil^{(i,j)}) \\
    &\le \UCBB{i,j}_{t}(\xx{i})+L\|\xbfbar^{(i,j)}-\xx{i}\|_2 -\LCBB{i,j}_{t}(\xx{i})+L\|\xx{i}-\xbftil^{(i,j)}\|_2\\
    &\le 2B\sig{i,j}_{t}(\xx{i})+2L\cdot\dia{\del{i}_{t}(\xbf)},
\end{align}
and therefore, the squared diameter of the confidence region of $\xx{i+1}=[\xx{i+1,1},\dots,\xx{i+1, d_{i+1}}]^T$ is
\begin{align}
    \dia{\del{i+1}_t(\xbf)}^2&=\sum_{j=1}^{d_{i+1}} \dia{\del{i+1,j}_t(\xbf)}^2\\
    &\le \sum_{j=1}^{d_{i+1}} \Big( 2B\sig{i,j}_{t}(\xx{i})+2L\cdot\dia{\del{i}_{t}(\xbf)} \Big)^2\\
    &= d_{i+1} \Big( 2B\sig{i,1}_{t}(\xx{i})+2L\cdot\dia{\del{i}_{t}(\xbf)} \Big)^2,
\end{align}
where the last step follows since $\sig{i,j}_t(\cdot)=\sig{i,1}_t(\cdot)$ for each $j\in[d_{i+1}]$ by the symmetry of our setup (each function is associated with the same kernel).

Then, by recursion, the diameter of the confidence region of $\xx{m}$ is
\begin{align}
    \dia{\del{m}_t(\xbf)}
    \le{}& 2\sqrt{d_m}B\sig{m-1,1}_t(\xx{m-1})+2\sqrt{d_m}L\cdot\dia{\del{m-1}_t(\xbf)}\\
    \le{}& 2\sqrt{d_m}B\sig{m-1,1}_t(\xx{m-1})+(2\sqrt{d_m}L)(2\sqrt{d_{m-1}}B)\sig{m-2,1}_t(\xx{m-2}) \nonumber \\
    & +(2\sqrt{d_m}L)(2\sqrt{d_{m-1}}L)\dia{\del{m-2}_t(\xbf)}\\
    \le{}& \sum_{i=1}^{m-1} \Big(\prod_{s=i+2}^m (2\sqrt{d_s}L) \Big) (2\sqrt{d_{i+1}}B) \sig{i,1}_t(\xx{i}),
\end{align}
where we set $\prod_{s=m+1}^m (2\sqrt{d_s}L)=1$.

Similarly to \eqref{eq:sim_0}--\eqref{eq:sim_6} for the case of chains, with $\xbfbar_t^{(m)}=\argmax_{\zbf\in\del{m}_{t-1}(\xbf_t)}\UCBBB{m,1}_{t-1}(\zbf)$ and $\xbftil_t^{(m)}=\argmin_{\zbf\in\del{m}_{t-1}(\xbf_t)}\LCBBB{m,1}_{t-1}(\zbf)$, the simple regret is upper bounded as follows:
\begin{align}
    r_t&=g(\xbf^\ast)-g(\xbf_t)\\
    &\le \UCB_{t-1}(\xbf_t)-\LCBBB{m,1}_{t-1}(\xbftil_t^{(m)})\\
    &= \UCBBB{m,1}_{t-1}(\xbfbar_t^{(m)})-\LCBBB{m,1}_{t-1}(\xbftil_t^{(m)})\\
    &\le \UCBB{m,1}_{t-1}(\xx{m}_t)+L\|\xbfbar_t^{(m)}-\xx{m}_t\|_2 -\LCBB{m,1}_{t-1}(\xx{m}_t)+L\|\xx{m}_t-\xbftil_t^{(m)}\|_2\\
    &\le 2B\sig{m,1}_{t-1}(\xx{m}_t)+2L\cdot\dia{\del{m}_{t-1}(\xbf_t)}\\
    &\le 2B\sig{m,1}_{t-1}(\xx{m}_t)+ 2L\sum_{i=1}^{m-1} \Big(\prod_{s=i+2}^m (2\sqrt{d_s}L) \Big) (2\sqrt{d_{i+1}}B) \sig{i,1}_{t-1}(\xx{i}_t)\\
    &\le \sum_{i=1}^{m} 2^{m-i+1} B L^{m-i} \Big(\prod_{s=i+1}^{m} \sqrt{d_s} \Big)  \sig{i,1}_{t-1}(\xx{i}_t),
\end{align}
where we use the convention $\prod_{s=m+1}^m \sqrt{d_s}=1$.

Hence, the cumulative regret is
\begin{align}
    R_T&= \sum_{t=1}^T r_t\\
    &\le \sum_{i=1}^{m} 2^{m-i+1} B L^{m-i} \Big(\prod_{s=i+1}^{m} (\sqrt{d_s}) \Big) \sum_{t=1}^T  \sig{i,1}_{t-1}(\xx{i}_t) \\
    &\le 2B\frac{(2L)^m-1}{2L-1} \sqrt{D_{2,m}} \Sigma_T \\
    &\le 2^{m+1} B L^{m-1} \sqrt{D_{2,m}} \Sigma_T ,
\end{align}
where $D_{2,m}=\prod_{i=2}^{m} d_i$ and $\Sigma_T=\max\limits_{i\in[m]}\max\limits_{\zbf_1,\dots,\zbf_T\in\Xc^{(i)}}\sum_{t=1}^T \sig{i,1}_{t-1}(\zbf_t)$.
\end{proof}

\section{Analysis of Non-Adaptive Sampling (\cref{algo:fd})}

As mentioned in the main body, the analysis in this appendix is restricted to the case that all domains in the network are \emph{hyperrectangular}.  This is a somewhat restrictive assumption, but we note that the primary reason for assuming this is to be able to apply Lemma \ref{lem:fill} below.  Hence, if Lemma \ref{lem:fill} can be generalized, then our results also generalize in the same way.

%A key foundation of the theoretical analysis of \cref{algo:fd} is the following lemma, which 
The lemma, stated as follows, provides an upper bound on the posterior standard deviation on a hyperrectangular domain for the Mat\'ern kernel in terms of the fill distance of the sampled points and the kernel parameter.
\begin{restatable}{lem}{fillmat}\textup{\citep[Theorem 5.4]{kanagawa2018gaussian}}
\label{lem:fill}
For the Mat\'ern kernel with smoothness $\nu$, a hyperrectangular domain $\Xc$, and any set of $T$ points $\Xc_T\subset \Xc$ with fill distance $\delta_T$, define
\begin{align}
	\sigbar_T:=\max_{\xbf\in\Xc}\sigma_T(\xbf),
\end{align}
where $\sigma_T(\cdot)$ is the posterior standard deviation computed based on $\Xc_T$ using $\eqref{eq:var}$. There exists a constant $\delta_0$ depending only on the kernel parameters such that if $\delta_T\le \delta_0$ then
\begin{align}
    \sigbar_T = O\big( (\delta_T)^\nu \big).
\end{align}
\end{restatable}
Recalling that $\{\xx{i}_s\}_{s=1}^T$ denote the intermediate outputs of $\{\xbf_s\}_{s=1}^T$ right after $\ff{i-1}$, we define the fill distance of the $\Xc^{(i)}$ with respect to $\{\xx{i}_s\}_{s=1}^T$ as
\begin{align}
	\delta_T^{(i)}&=\max_{\zbf\in\Xc^{(i)}}\min_{s\in[T]}\|\zbf-\xx{i}_s\|_2.\label{eq:delta_i_T}
\end{align}
We also provide corollaries on the posterior standard deviation of a single layer for each network structure.
\begin{restatable}[Chains]{cor}{fillchain}
	\label{cor:fill_chain}
	With $k$ being the Mat\'ern kernel with smoothness $\nu$, consider a chain $g=\ff{m}\circ\ff{m-1}\circ\cdots\circ\ff{1}$ with $\ff{i}\in\Hc_k(B)\cap\Fc(L)$ for each $i\in[m]$.
	For any set of $T$ points $\Xc_T=\{\xbf_s\}_{s=1}^T\subset\Xc$ with fill distance $\delta_T$, let $\Xc^{(i)}_T=\{\xx{i}_s\}_{s=1}^T$ be the noise-free observations of $\ff{i-1}\circ\cdots\circ\ff{1}$.
	Define
	\begin{align}
		\sigbar^{(i)}_T:=\max_{\zbf\in\Xc^{(i)}}\sig{i}_T(\zbf),
	\end{align}
	where $\sig{i}_T(\cdot)$ is the posterior standard deviation computed based on $\Xc^{(i)}_T$ using $\eqref{eq:var}$. Then, for any $i\in[m]$, there exists a constant $\delta_0$ depending only on the kernel parameters such that if $\delta_T\le \delta_0$ then
	\begin{align}
		\sigbar^{(i)}_T =  O\big((L^{i-1} \delta_T)^\nu\big).
	\end{align}
\end{restatable}
\begin{proof}
	First, it is straightforward that $\ff{i-1}\circ\cdots\circ\ff{1}$ has Lipschitz constant $L^{i-1}$. For any input $\xbf, \xbf_s\in\Xc$, $\ff{i-1}\circ\cdots\circ\ff{1}$ outputs $\xx{i},\xx{i}_s$ such that 
	\begin{align}
		|\xx{i}-\xx{i}_s|=|\ff{i-1}\circ\cdots\circ\ff{1}(\xbf)-\ff{i-1}\circ\cdots\circ\ff{1}(\xbf_s)|\le L^{i-1} \|\xbf-\xbf_s\|_2.
	\end{align}
	
	If $\{\xbf_s\}_{s=1}^T$ has fill distance $\delta_T$, then the fill distance of $\{\xx{i}_s\}_{s=1}^T$ is
	\begin{align}
		\delta_T^{(i)}&=\max_{\zbf\in\Xc^{(i)}}\min_{s\in[T]}|\zbf-\xx{i}_s|=\max_{\xbf\in\Xc}\min_{s\in[T]}|\xx{i}-\xx{i}_s|
		\le \max_{\xbf\in\Xc}\min_{s\in[T]} L^{i-1} \|\xbf-\xbf_s\|_2
		= L^{i-1} \delta_T.\label{eq:fill_i}
	\end{align}
	Then, by \cref{lem:fill} we have
	\begin{align}
		\sigbar^{(i)}_T = \max_{\zbf\in\Xc^{(i)}} \sig{i}_T(\zbf) = O\big((\delta_T^{(i)})^\nu\big) = O \big((L^{i-1} \delta_T)^\nu\big).
	\end{align}
\end{proof}

\begin{restatable}[Multi-output chains]{cor}{fillmul}
	\label{cor:fill_mul}
	For $\Gamma(\cdot,\cdot)=k(\cdot,\cdot)\Ibf_n$, with $k$ being the Mat\'ern kernel with smoothness $\nu$ and $\Ibf_n$ being the identity matrix of size $n$, consider a chain $g=\ff{m}\circ\ff{m-1}\circ\cdots\circ\ff{1}$ with $\ff{i}\in\Hc_\Gamma(B)\cap\Fc(L)$ and $\Xc^{(i)}$ being a hyperrectangle for each $i\in[m]$ and $n=d_{i+1}$.
	For any set of $T$ points $\Xc_T=\{\xbf_s\}_{s=1}^T\subset\Xc$ with fill distance $\delta_T$, let $\Xc^{(i)}_T=\{\xx{i}_s\}_{s=1}^T$ be the noise-free observations of $\ff{i-1}\circ\cdots\circ\ff{1}$.
	Define
	\begin{align}
		\overline{\Gamma}^{(i)}_T:=\max_{\zbf\in\Xc^{(i)}}\|\Gamma^{(i)}_T(\zbf,\zbf)\|_2^{1/2},
	\end{align}
	where $\Gamma^{(i)}_T(\zbf,\zbf)$ is the posterior variance matrix computed based on $\Xc^{(i)}_T$ using $\eqref{eq:multi_var}$, and $\|\cdot\|_2$ denotes the matrix spectral norm. Then, for any $i\in[m]$, there exists a constant $\delta_0$ depending only on the kernel parameters such that if $\delta_T\le \delta_0$ then
	\begin{align}
		\overline{\Gamma}^{(i)}_T =  O\big((L^{i-1} \delta_T)^\nu\big).
	\end{align}
\end{restatable}
\begin{proof} Recalling that $\|\cdot\|_2$ with a matrix argument denotes the spectral norm, we have
	\begin{align}
		\overbar{\Gamma}^{(i)}_T&=\max_{\zbf\in\Xc^{(i)}}\|\Gamma^{(i)}_T(\zbf,\zbf)\|_2^{1/2}\\
		&=\max_{\zbf\in\Xc^{(i)}} \|\big(\sig{i}_T(\zbf)\big)^2 \Ibf_n \|_2^{1/2} && \text{(by \cref{lem:multi_var})} \\
		&=\max_{\zbf\in\Xc^{(i)}} \sig{i}_T(\zbf) && \text{(by $\|\Ibf_n\|_2=1$)} \\
		&=O\big((\delta_T^{(i)})^\nu\big) && \text{(by \cref{lem:fill})}\label{eq:fill_mul}\\
		&= O \big((L^{i-1} \delta_T)^\nu\big). &&\text{(by \eqref{eq:fill_i})}
	\end{align}
\end{proof}

\begin{restatable}[Feed-forward networks]{cor}{fillnet}
	\label{cor:fill_net}
	With $k$ being the Mat\'ern kernel with smoothness $\nu$, consider a feed-forward network $g=\ff{m}\circ\ff{m-1}\circ\cdots\circ\ff{1}$
	with $\ff{s}(\cdot)=[\ff{s,j}(\cdot)]_{j=1}^{d_{s+1}}$ and $\ff{s,j}\in\Hc_k(B)\cap\Fc(L)$ and $\Xc^{(s)}$ being a hyperrectangle for each $s\in[m],j\in[d_{s+1}]$.
	For any set of $T$ points $\Xc_T=\{\xbf_s\}_{s=1}^T\subset\Xc$ with fill distance $\delta_T$, let $\Xc^{(i)}_T=\{\xx{i}_s\}_{s=1}^t$ be the noise-free observations of $\ff{i-1}\circ\cdots\circ\ff{1}$.
	Define
	\begin{align}
		\sigbar^{(i,j)}_T:=\max_{\zbf\in\Xc^{(i)}}\sig{i,j}_T(\zbf),
	\end{align}
	where $\sig{i,j}_T(\cdot)$ is the posterior standard deviation computed based on $\Xc^{(i)}_T$ using $\eqref{eq:var}$. Then, for each $i\in[m]$ and $j\in[d_{i+1}]$, there exists a constant $\delta_0$ depending only on the kernel parameters such that if $\delta_T\le \delta_0$ then
	\begin{align}
		\sigbar^{(i,j)}_T =  O\big((\sqrt{D_{2,i}}L^{i-1} \delta_T)^\nu\big), \label{eq:sigbar}
	\end{align}
	where $D_{2,i}=\prod_{s=2}^i d_s$.
\end{restatable}
\begin{proof}
	Since $\ff{s,j}\in\Fc(L)$ for each $s\in[m]$, we have for any $\zbf,\zbf'\in\Xc^{(s)}$ that
	\begin{align}
		\|\ff{s}(\zbf)-\ff{s}(\zbf')\|_2
		&=\sqrt{\sum_{j=1}^{d_{s+1}}\|\ff{s,j}(\zbf)-\ff{s,j}(\zbf')\|_2^2}\\
		&\le \sqrt{\sum_{j=1}^{d_{s+1}}L^2\|\zbf-\zbf'\|_2^2}\\
		&= \sqrt{d_{s+1}}L\|\zbf-\zbf'\|_2,
	\end{align}
	and therefore $\ff{i-1}\circ\cdots\circ\ff{1}$ is Lipschitz continuous with constant $\sqrt{D_{2,i}}L^{i-1}$, where $D_{2,i}=\prod_{s=1}^i d_s$.
	If $\{\xbf_s\}_{s=1}^T$ has fill distance $\delta_T$, then the fill distance of $\{\xx{i}_s\}_{s=1}^T$ is
	\begin{align}
		\delta_T^{(i)}&=\max_{\zbf\in\Xc^{(i)}}\min_{s\in[T]}\|\zbf-\xx{i}_s\|_2\\
		&=\max_{\xbf\in\Xc}\min_{s\in[T]}\|\xx{i}-\xx{i}_s\|_2\\
		&\le \max_{\xbf\in\Xc}\min_{s\in[T]}\sqrt{D_{2,i}} L^{i-1} \|\xbf-\xbf_s\|_2\\
		&\le \sqrt{D_{2,i}} L^{i-1} \delta_T.\label{eq:fill_i_net}
	\end{align}
	Hence, by \cref{lem:fill} we have
	\begin{align}
		\sigbar^{(i,j)}_T = \max_{\zbf\in\Xc^{(i)}} \sig{i,j}_T(\zbf) = O\big((\delta_T^{(i)})^\nu\big) = O \big((\sqrt{D_{2,i}}L^{i-1} \delta_T)^\nu\big).
	\end{align}
\end{proof}

\begin{rem}
	\cref{cor:fill_mul} and \cref{cor:fill_net} require $\Xc^{(i)}$ being a hyperrectangle. This is because they are the consequences of \cref{lem:fill}, which only holds for a hyperrectangular domain. For chains, the requirement of hyperrectangular domain is trivial, since $\Xc=[0,1]^d$ and $\Xc^{(2)},\dots,\Xc^{(m)}$ have single dimension.
\end{rem}

\subsection{Proof of \cref{thm:fd_chain} (Chains)}
\label{sec:fd_chain}
In this section, we prove \cref{thm:fd_chain}, which is restated as follows.
\fdchain*

\begin{proof}
Since $\ff{i}\in\Fc(L)$ implies that $\ff{m}\circ\cdots\circ\ff{i}$ is Lipschitz continuous with constant $L^{m-i+1}$, defining
\begin{align}
	\sigtil^{(i)}_T(\xbf):=(\sig{i}_T\circ\muu{i-1}_T\circ\cdots\circ\muu{1}_T)(\xbf),
\end{align}
we have for all $\xbf\in\Xc$ that
\begin{align}
	g(\xbf)={}&(\ff{m}\circ\ff{m-1}\circ\cdots\circ\ff{2})\big(\ff{1}(\xbf)\big)\label{eq:fd_start}\\
	\le{}& (\ff{m}\circ\ff{m-1}\circ\cdots\circ\ff{2})\big(\muu{1}_T(\xbf)\big) + L^{m-1} B\sig{1}_T(\xbf)\\
	={}& (\ff{m}\circ\ff{m-1}\circ\cdots\circ\ff{3})\big((\ff{2}\circ\muu{1}_T)(\xbf)\big) + L^{m-1} B\sig{1}_T(\xbf)\\
	\le{}& (\ff{m}\circ\ff{m-1}\circ\cdots\circ\ff{3})\big((\muu{2}_T\circ\muu{1}_T)(\xbf)\big) + L^{m-2}B\sig{2}_T\big(\muu{1}_T(\xbf)\big) + L^{m-1} B\sig{1}_T(\xbf)\\
	% ={}&(\ff{m}\circ\ff{m-1}\circ\cdots\circ\ff{4})\big(\ff{3}\circ\muu{2}_T\circ\muu{1}_T(\xbf)\big) + L^{m-2}B\sig{2}\big(\muu{1}_T(\xbf)\big) + L^{m-1} B\sig{1}_T(\xbf)\\
	% \le{}& (\ff{m}\circ\ff{m-1}\circ\cdots\circ\ff{4})\big(\muu{3}_T\circ\muu{2}\circ\muu{1}_T(\xbf)\big) + L^{m-3}B\sig{3}\big(\muu{2}_T\circ\muu{1}_T(\xbf)\big) \nonumber\\
	% &+ L^{m-2}B\sig{2}_T\big(\muu{1}_T(\xbf)\big) + L^{m-1} B\sig{1}_T(\xbf)\\
	\le{}& (\muu{m}_T\circ\muu{m-1}_T\circ\cdots\circ\muu{1}_T)(\xbf)+B\sig{m}_T\big((\muu{m-1}_T\circ\cdots\circ\muu{1}_T)(\xbf)\big)\nonumber\\
	&+LB\sig{m-1}_T\big((\muu{m-2}_T\circ\cdots\circ\muu{1}_T)(\xbf)\big)+\cdots+ L^{m-1} B\sig{1}_T(\xbf)\\
	={}& (\muu{m}_T\circ\muu{m-1}_T\circ\cdots\circ\muu{1}_T)(\xbf) + B\sum_{i=1}^m L^{m-i} (\sig{i}_T\circ\muu{i-1}_T\circ\cdots\circ\muu{1})(\xbf)\\
	={}& (\muu{m}_T\circ\muu{m-1}_T\circ\cdots\circ\muu{1}_T)(\xbf) + B\sum_{i=1}^m L^{m-i} \sigtil_T^{(i)}(\xbf),\label{eq:fd_1}
\end{align}
where the first two inequalities use the confidence bounds and Lipschitz assumption, and the third inequality follows by continuing recursively.  Similarly, we also have
\begin{align}
	g(\xbf)\ge (\muu{m}_T\circ\muu{m-1}_T\circ\cdots\circ\muu{1}_T)(\xbf) - B\sum_{i=1}^m L^{m-i} \sigtil_T^{(i)}(\xbf).
\end{align}
Hence, defining $\sigtil_T^{(i)}:=\max_{\xbf\in\Xc}\sigtil_T^{(i)}(\xbf)$, it follows that
\begin{align}
	r^\ast_T={}&g(\xbf^\ast)-g(\xbf^\ast_T)\\
	\le{}& \Big( (\muu{m}_T\circ\muu{m-1}_T\circ\cdots\circ\muu{1}_T)(\xbf^\ast) + B\sum_{i=1}^m L^{m-i} \sigtil_T^{(i)} (\xbf^\ast) \Big)\nonumber\\
	&-\Big( (\muu{m}_T\circ\muu{m-1}_T\circ\cdots\circ\muu{1}_T)(\xbf^\ast_T) - B\sum_{i=1}^m L^{m-i} \sigtil_T^{(i)} (\xbf^\ast_T) \Big) \\
	%\le{}& \Big( (\muu{m}_T\circ\muu{m-1}_T\circ\cdots\circ\muu{1}_T)(\xbf^\ast_T) + B\sum_{i=1}^m L^{m-i} \sigtil_T^{(i)}  (\xbf^\ast)\Big)\nonumber\\
	%&-\Big( (\muu{m}_T\circ\muu{m-1}_T\circ\cdots\circ\muu{1}_T)(\xbf^\ast_T) - B\sum_{i=1}^m L^{m-i} \sigtil_T^{(i)} (\xbf^\ast_T) \Big)\\
	\le{}& 2B \sum_{i=1}^m L^{m-i} \max\{\sigtil_T^{(i)} (\xbf^\ast), \sigtil_T^{(i)} (\xbf^\ast_T)\}\label{eq:fd_end} \\
	\le{}& 2B \sum_{i=1}^m L^{m-i} \sigtil_T^{(i)},
\end{align}
where \eqref{eq:fd_end} uses the fact that $\xbf^\ast_T$ maximizes the composite posterior mean function.  It remains to derive an upper bound on $\sigtil_T^{(i)}$ for $i\in[m]$. Firstly, by \cref{lem:fill}, we have
\begin{align}
	\sigtil^{(1)}_T(\xbf)=\sig{1}_T(\xbf)=O\big((\delta_T)^\nu\big)=O(T^{-\nu/d}).\label{eq:sigtil_1}
\end{align}
We now represent the upper bound on $\sigtil^{(i+1)}_T$ in terms of $\{\sigtil^{(j)}_T:j\in[i]\}$ for $i\in[m-1]$. For all $\xbf\in\Xc$, the distance between $(\muu{i}_T\circ\cdots\circ\muu{1}_T)(\xbf)$ and its closest point in $\Xc^{(i+1)}$ is
\begin{align}
	\delta_{i+1}(\xbf)&=\min_{\zbf\in\Xc^{(i+1)}}|(\muu{i}_T\circ\cdots\circ\muu{1}_T)(\xbf)-\zbf|\\
	&\le|\muu{i}_T\circ\cdots\circ\muu{1}_T(\xbf)-\ff{i}\circ\cdots\circ\ff{1}(\xbf)|\\
	&= B\sum_{j=1}^{i}L^{i-j}\sigtil^{(j)}_T(\xbf), \label{eq:delta_i_1}
\end{align}
where the last step follows from \eqref{eq:fd_1}.
Then, recalling the definition of $\delta^{(i)}_T$ in \eqref{eq:delta_i_T}, we have that the distance between $(\muu{i}_T\circ\cdots\circ\muu{1}_T)(\xbf)$ and its closest point in $\Xc^{(i+1)}_T$ is
\begin{align}
	\delta'_{i+1}(\xbf)&=\min_{s\in[T]}|(\muu{i}_T\circ\cdots\circ\muu{1}_T)(\xbf)-\xx{i+1}_s|\\
	&\le\delta^{(i+1)}_T+\delta_{i+1}(\xbf) \label{eq:delta_prime_2}\\
	&\le O(L^i T^{-1/d})+B\sum_{j=1}^{i}L^{i-j}\sigtil^{(j)}_T(\xbf) \label{eq:delta_prime_3}\\
	&= O\Big(\max\Big\{L^iT^{-1/d}, i B \cdot \max_{j\in[i]}L^{i-j} \sigtil^{(j)}_T(\xbf) \Big\}\Big),
\end{align}
where \eqref{eq:delta_prime_2} applies the triangle inequality, and \eqref{eq:delta_prime_3} uses \eqref{eq:fill_i} and \eqref{eq:delta_i_1}.

Now, we extend $\Xc^{(i+1)}$ to $\widetilde{\Xc}^{(i+1)}$, the shortest interval that covers both $\{\muu{i}_T\circ\cdots\circ\muu{1}_T(\xbf):\xbf\in\Xc\}$ and $\Xc^{(i+1)}$. Then, the fill distance of the extended domain with respect to $\Xc^{(i+1)}_T$ is
\begin{align}
	\delta'_{i+1}:=\max_{\xbf\in\Xc}\delta'_{i+1}(\xbf)=O\Big(\max\Big\{L^iT^{-1/d},  i B \cdot\max_{j\in[i]}L^{i-j} \sigtil^{(j)}_T \Big\}\Big),
\end{align}
and by \cref{lem:fill}, we have
\begin{align}
	\sigtil^{(i+1)}_T=\max_{\xbf\in\Xc}\sigtil^{(i+1)}_T(\xbf)=O\big((\delta'_{i+1})^\nu\big)=O\Bigg(\Big(\max\Big\{L^iT^{-1/d},  i B\cdot\max_{j\in[i]}L^{i-j} \sigtil^{(j)}_T \Big\}\Big)^\nu\Bigg).\label{eq:sigtil}
\end{align}
We now split the analysis into two cases.

\underline{The case that $\nu \le 1$}: 
Recall that we assume $B=\Theta(L)$, and using the fact that $\Otil(L^{c\nu})$ is equivalent to $O\big( (L^{c\nu} {\rm poly}(c\log L)) \big)$ for fixed $\nu$, it follows from \eqref{eq:sigtil_1} and \eqref{eq:sigtil} that when $\nu \le 1$, we have
\begin{align}
	\sigtil_T^{(2)}&=O\Big(\max\Big\{L^\nu T^{-\nu/d},B^\nu T^{-\nu^2/d}\Big\}\Big)=O(L^\nu T^{-\nu^2/d}), \\
	\sigtil^{(3)}_T&=O\Big(\max\Big\{L^{2\nu}T^{-\nu/d}, 2^\nu B^\nu L^\nu T^{-\nu^2/d}, 2^\nu B^\nu L^{\nu^2} T^{-\nu^3/d}\Big\}\Big)=O(2^\nu B^\nu L^\nu T^{-\nu^3/d})=\Otil( B^\nu L^\nu T^{-\nu^3/d}),\\
	&\vdotswithin{=}\nonumber\\
	\sigtil^{(i)}_T&=\Otil\big((i-1)^\nu B^\nu L^{(i-2)\nu}T^{-\nu^i/d}\big)=\Otil(B^\nu L^{(i-2)\nu}T^{-\nu^i/d}),
\end{align}
and the simple regret is
\begin{align}
	r^\ast_T&\le 2B\sum_{i=1}^m L^{m-i}\sigtil^{(i)}_T=\Otil(BL^{m-1}T^{-\nu^m/d}).
\end{align}

\underline{The case that $\nu \le 1$}:  
With $B=\Theta(L)$, using the property of $\Otil(\cdot)$ similarly to the first case, it follows from \eqref{eq:sigtil_1} and \eqref{eq:sigtil} that when $\nu > 1$, we have
\begin{align}
	\sigtil_T^{(2)}&=O\Big(\max\Big\{L^\nu T^{-\nu/d},B^\nu T^{-\nu^2/d}\Big\}\Big)=O(L^\nu T^{-\nu/d}), \\
	\sigtil^{(3)}_T&=O\Big(\max\Big\{L^{2\nu}T^{-\nu/d}, 2^\nu B^\nu L^{\nu^2} T^{-\nu^2/d}\Big\}\Big)=O(2^\nu B^\nu L^{\nu^2} T^{-\nu/d})=\Otil(B^\nu L^{\nu^2} T^{-\nu/d}),\\
	&\vdotswithin{=}\nonumber\\
	\sigtil^{(i)}_T&=\Otil\Big(\max\Big\{L^{(i-1)\nu}T^{-\nu/d}, (i-1)^\nu B^{\nu+\cdots+\nu^{i-2}}L^{\nu^{i-1}}T^{-\nu^2/d}\Big\}\Big)\nonumber\\
	&= \Otil\Big(\max\Big\{L^{(i-1)\nu}T^{-\nu/d}, B^{\nu+\cdots+\nu^{i-2}}L^{\nu^{i-1}}T^{-\nu^2/d}\Big\}\Big),
\end{align}
and the simple regret is
\begin{align}
	r^\ast_T&\le 2B\sum_{i=1}^m L^{m-i}\sigtil^{(i)}_T
	=\Otil\Big(\max\Big\{BL^{(m-1)\nu}T^{-\nu/d}, B^{1+\nu+\nu^2+\cdots+\nu^{m-2}}L^{\nu^{m-1}}T^{-\nu^2/d} \Big\}\Big).
\end{align}
\end{proof}

\subsection{Two More Restrictive Cases for Chains}
\label{sec:rest_chain}
Recall the following two more restrictive cases introduced in the main body:
\begin{itemize}
	\item \textbf{Case 1}: We additionally assume that $(\muu{i}_T\circ\cdots\circ\muu{1}_T)(\xbf^\ast)\in\Xc^{(i+1)}$ and $(\muu{i}_T\circ\cdots\circ\muu{1}_T)(\xbf^\ast_T)\in \Xc^{(i+1)}$ for $i\in[m-1]$.
	\item \textbf{Case 2}: We additionally assume that all the $\Xc^{(i)}$'s are known. Defining
	\begin{align}
		\muutil{i}_T(\zbf) &= \argmin_{\zbf'\in\Xc^{(i+1)}}|\muu{i}_T(\zbf)-\zbf'|, \label{eq:mutil}
	\end{align}
	we let the algorithm return
	\begin{align}
		\xbf^\ast_T=\argmax_{\xbf\in\Xc}(\muutil{m}_T\circ\cdots\circ\muutil{1}_T)(\xbf).
	\end{align}
\end{itemize}
%To discuss the simple regret upper bound for these two cases, we first introduce the following corollary, providing an upper bound on postieror standard deviation for points inside the domain.
In Case 1, it follows from \cref{cor:fill_chain} (along with \eqref{eq:sigbar} and $\delta_T= \Theta(T^{-\frac{1}{d}})$) that for each $\xbf'\in\{\xbf^\ast,\xbf^\ast_T\}$
\begin{align}
	\sigtil^{(i+1)}_T(\xbf')&=(\sig{i+1}_T\circ\muu{i}_T\circ\cdots\circ\muu{1}_T)(\xbf')\le\sigbar^{(i+1)}_T=O\big((L^i \delta_T)^\nu\big)=O\big((L^i T^{-1/d})^\nu\big).
\end{align}
By substituting these upper bounds into \eqref{eq:fd_end}, it holds that
\begin{align}
	r^\ast_T&\le 2B \sum_{i=1}^m L^{m-i} \max\{\sigtil_T^{(i)} (\xbf^\ast), \sigtil_T^{(i)} (\xbf^\ast_T)\} \\
	&\le 2B \sum_{i=1}^m L^{m-i}  O\big((L^i T^{-1/d})^\nu\big)\\
	&=\begin{cases}
		O(BL^{m-1}T^{-\nu/d})&\text{ when $\nu\le 1$,}\\
		O(BL^{(m-1)\nu}T^{-\nu/d})&\text{ when $\nu> 1$.}
	\end{cases}
\end{align}
In Case 2, for all $\zbf\in\Xc^{(i)}$, we have $\muutil{i}_T(\zbf)\in\Xc^{(i+1)}$, and
\begin{align}
	|\ff{i}(\zbf)-\muutil{i}_T(\zbf)|&\le  |\ff{i}(\zbf)-\muu{i}_T(\zbf)| + |\muu{i}_T(\zbf)-\muutil{i}_T(\zbf)|\\
	&\le |\ff{i}(\zbf)-\muu{i}_T(\zbf)| + |\muu{i}_T(\zbf)-\ff{i}(\zbf)| \\
	&\le 2B\sig{i}_T(\zbf),
\end{align}
where we used \eqref{eq:mutil} and the confidence bounds.  
By recursion, we also have $(\muutil{i}_T\circ\cdots\circ\muutil{1}_T)(\xbf)\in\Xc^{(i+1)}$ for all $\xbf\in\Xc$, and therefore
\begin{align}
	(\sig{i+1}_T\circ\muutil{i}_T\circ\cdots\circ\muutil{1}_T)(\xbf)\le\sigbar^{(i+1)}_T.
\end{align}

Hence, replacing $\mu$ with $\mutil$ in \eqref{eq:fd_start}--\eqref{eq:fd_end}, it follows from \cref{cor:fill_chain} that
\begin{align}
	r^\ast_T\le 4B\sum_{i=1}^m L^{m-i}\sigbar_T^{(i)}=
	\begin{cases}
		O(BL^{m-1}T^{-\nu/d})&\text{ when $\nu\le 1$,}\\
		O(BL^{(m-1)\nu}T^{-\nu/d})&\text{ when $\nu> 1$.}
	\end{cases}
\end{align}

\subsection{Non-Adaptive Sampling Method for Multi-Output Chains}
\label{sec:fd_mul}
For multi-output chains, the composite posterior mean of $g(\xbf)$ is
\begin{align}
	\mu_T^g(\xbf)=(\muu{m}_T\circ\muu{m-1}_T\circ\cdots\circ\muu{1}_T)(\xbf),\label{eq:mean_mul}
\end{align}
where $\muu{i}_T$ denotes the posterior mean of $\ff{i}$ computed using \eqref{eq:multi_mean} based on $\{(\xx{i}_s,\xx{i+1}_s)\}_{s=1}^T$ for each $i\in[m]$. 

We again assume that each $\Xc^{(i)}$ is a hyperrectangle of dimension $d_i$.  Then, the upper bound on simple regret of \cref{algo:fd} using \eqref{eq:mean_mul} is provided in the following theorem.

\begin{restatable}[Non-adaptive sampling method for multi-output chains]{thm}{fdmul}
	\label{thm:fd_mul}
	Under the setup of \cref{sec:setup}, given $B=\Theta(L)$, $k=k_\text{Mat\'ern}$, $\Gamma(\cdot,\cdot)=k(\cdot,\cdot)\Ibf$, and a multi-output chain  $g=\ff{m}\circ\ff{m-1}\circ\cdots\circ\ff{1}$ with $\ff{i}\in\Hc_\Gamma(B)\cap\Fc(L)$ and $\Xc^{(i)}$ being a hyperrectangle for each $i\in[m]$,
	\begin{itemize}
		\item when $\nu\le 1$, \cref{algo:fd} achieves
		\begin{align*}
			r^\ast_T=\Otil(BL^{m-1}T^{-\nu^m/d});
		\end{align*}
		\item when $\nu>1$, \cref{algo:fd} achieves
		\begin{align*}
			r^\ast_T=\Otil\big(\max\big\{BL^{(m-1)\nu}T^{-\nu/d}, B^{1+\nu+\nu^2+\cdots+\nu^{m-2}}L^{\nu^{m-1}}T^{-\nu^2/d} \big\}\big).
		\end{align*}
	\end{itemize}
\end{restatable}

\begin{proof}
The analysis is similar to the case of chains, so we omit some details and focus on the main differences.  
Defining $\widetilde{\Gamma}_T^{(i)}(\xbf)=\|\Gamma_T^{(i)}\circ\muu{i-1}_T\circ\cdots\circ\muu{1}_T(\xbf)\|_2^{1/2}$, similarly to the case of chains, we have
\begin{align}
	|g(\xbf) - (\muu{m}_T\circ\muu{m-1}_T\circ\cdots\circ\muu{1}_T)(\xbf)|\le B\sum_{i=1}^m L^{m-i} \widetilde{\Gamma}_T^{(i)}(\xbf),
\end{align}
and
\begin{align}
	r^\ast_T=g(\xbf^\ast)-g(\xbf^\ast_T)&\le 2 B\sum_{i=1}^m L^{m-i} \max\{ \widetilde{\Gamma}_T^{(i)}(\xbf^\ast),\widetilde{\Gamma}_T^{(i)}(\xbf^\ast_T) \}.
\end{align}
For each $\xbf'\in\{\xbf^\ast,\xbf^\ast_T\}$, we extend $\Xc^{(i+1)}$ to the smallest hyperrectangle $\widetilde{\Xc}^{(i+1)}$ that covers $(\muu{i}_T\circ\cdots\circ\muu{1})(\xbf')$, and the original domain $\Xc^{(i+1)}$ (see \cref{fig:mul_fill}). Recalling the definition of $\delta^{(i+1)}_T$ in \eqref{eq:delta_i_T}, with $\delta_{i+1}(\xbf)=\min_{\zbf\in\Xc^{(i+1)}}\|(\muu{i}_T\circ\cdots\circ\muu{1})(\xbf)-\zbf\|_2$, the fill distance of $\widetilde{\Xc}^{(i+1)}$ with respect to $\{\xx{i+1}_s\}_{s=1}^T$ is
\begin{align}
	\delta'_{i+1}&=\max_{\zbf'\in\widetilde{\Xc}^{(i+1)}}\min_{s\in[T]}\|\zbf'-\xx{i+1}_s\|_2\\
	&\le \delta_T^{(i+1)}+\max_{\zbf'\in\widetilde{\Xc}^{(i+1)}}\min_{\zbf\in\Xc^{(i+1)}}\|\zbf'-\zbf\|_2\\
	&\le \delta_T^{(i+1)}+\delta_{i+1}(\xbf')\\
	&\le \delta_T^{(i+1)}+B\sum_{j=1}^i L^{i-j} \widetilde{\Gamma}_T^{(i)}(\xbf').
\end{align}
\begin{figure}[h!]
	\centering
	\includegraphics[width=0.6\linewidth]{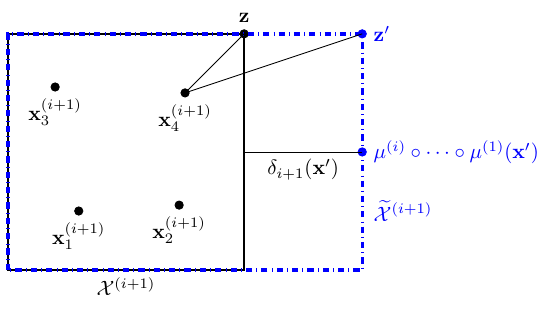}
	\caption{Extended domain for multi-output chains.}
	\label{fig:mul_fill}
\end{figure}
Since $\widetilde{\Xc}^{(i+1)}$ is a hyperrectangle and $\muu{i}_T\circ\cdots\circ\muu{1}(\xbf')\in\widetilde{\Xc}^{(i+1)}$, it follows from \eqref{eq:fill_mul} that
\begin{align}
	\widetilde{\Gamma}_T^{(i+1)}(\xbf')&=O\big((\delta'_{i+1})^\nu\big)\\
	&=O\Bigg(\Big(\max\Big\{L^iT^{-1/d},  i B\cdot\max_{j\in[i]}L^{i-j} \widetilde{\Gamma}_T^{(j)} (\xbf') \Big\}\Big)^\nu\Bigg).
\end{align}
By \cref{lem:multi_var}, we also have
\begin{align}
	\widetilde{\Gamma}_T^{(1)}(\xbf')=\|\Gamma_T^{(1)}(\xbf')\|_2^{1/2}=\sigma_T^{(1)}(\xbf')=O\big((\delta_T)^\nu)=O(T^{-\nu/d}).
\end{align}

This recursive relation is exactly the same as that of chains in \eqref{eq:sigtil}, showing \cref{thm:fd_chain} extends to multi-output chains.

For the two more restrictive cases, \cref{cor:fill_mul}, with the same posterior standard deviation upper bound as \cref{cor:fill_chain}, implies that the results in \cref{sec:rest_chain} also extend to multi-output chains.

%When $\muu{i}_T\circ\cdots\circ\muu{1}(\xbf')$ is located at the corner of $\widetilde{\Xc}^{(i+1)}$, it is straightforward that the fill distance of $\widetilde{\Xc}^{(i+1)}$ is no larger than $\max_{\xbf\in\Xc}\delta'_{i+1}(\xbf)$. When neither $\muu{i}_T\circ\cdots\circ\muu{1}(\xbf')$ is at the corner of $\widetilde{\Xc}^{(i+1)}$, the distance between any point in $\widetilde{\Xc}^{(i+1)}$ to its nearest point in the original domain $\Xc^{(i+1)}$ is no larger than $\max_{\xbf\in\Xc}\delta_{i+1}(\xbf)$. Hence, by the triangle inequality, the fill distance of $\widetilde{\Xc}^{(i+1)}$ is upper bounded by $\max_{\xbf\in\Xc} \delta'_{i+1}(\xbf)$ (in \cref{fig:mul_fill}, $\|\zbf'-\xbf_4^{(i+1)}\|_2\le\|\zbf'-\zbf\|_2+\|\xbf_4^{(i+1)}-\zbf\|_2\le \delta_{i+1}(\xbf') +\delta_T^{(i+1)}$), and the results of chains also hold for multi-output chains. By \cref{cor:fill_mul}, the results in \cref{sec:rest_chain} extend to multi-ouput chains.

\end{proof}

\subsection{Non-Adaptive Sampling Method for Feed-Forward Networks}
\label{sec:fd_net}
For feed-forward networks of scalar-valued functions, let $\muu{i,j}_T$ denote the posterior mean of $\ff{i,j}$ computed using \eqref{eq:mean} based on $\{(\xx{i}_s,\xx{i+1,j}_s)\}_{s=1}^T$. The composite posterior mean of $g(\xbf)$ with feed-forward network structure is
\begin{align}
	\mu_T^g(\xbf)=\muu{m,1}_T(\zz{m})\label{eq:mean_net}
\end{align}
with
\begin{align}
	\zz{1}&=\xbf,\\
	\zz{i+1,j}&=\muu{i,j}_T(\zz{i})&&\text{ for }i\in[m-1],j\in[d_{i+1}],\\
	\zz{i+1}&=\muu{i}_T(\zz{i})=[\zz{i+1,1},\dots,\zz{i+1,d_i}]&&\text{ for }i\in[m-1].
\end{align}

Assuming each $\Xc^{(i)}$ is a hyperrectangle of dimension $d_i$, the following theorem provides the upper bound on simple regret of \cref{algo:fd} using \eqref{eq:mean_net} for feed-forward networks.

\begin{restatable}[Non-adaptive sampling method for feed-forward networks]{thm}{fdnet}
	\label{thm:fd_net}
	Under the setup of \cref{sec:setup}, given $B=\Theta(L)$,  $k=k_\text{Mat\'ern}$, and  a feed-forward network $g=\ff{m}\circ\ff{m-1}\circ\cdots\circ\ff{1}$ with $\ff{i}(\zbf)=[\ff{i,j}(\zbf)]_{j=1}^{d_{i+1}}$, $\ff{i,j}\in\Hc_k(B)\cap\Fc(L)$, and $\Xc^{(i)}$ being a hyperrectangle for each $i\in[m],j\in[d_{i+1}]$,
	\begin{itemize}
		\item when $\nu\le 1$, \cref{algo:fd} achieves
		\begin{align}
			r^\ast_T=\Otil(\sqrt{D_{2,m}}BL^{m-1}T^{-\nu^m/d});
		\end{align}
		\item when $\nu>1$, \cref{algo:fd} achieves
		\begin{align}
			r^\ast_T=\Otil\Big(\max\Big\{(D_{2,m})^{\nu/2}BL^{(m-1)\nu}T^{-\nu/d}, \Dtil_{2,m}^\nu B^{1+\nu+\nu^2+\cdots+\nu^{m-2}}L^{\nu^{m-1}}T^{-\nu^2/d} \Big\}\Big),
		\end{align}
	\end{itemize}
where $D_{2,m}=\prod_{i=2}^m d_i$ and $\Dtil_{2,m}^\nu=\prod_{i=2}^m (d_i)^{\nu^{m+1-i}/2}$.

\end{restatable}

\begin{proof}
With $\muu{i}_T(\cdot)=[\muu{i,j}_T(\cdot)]_{j=1}^{d_{i+1}}$, we have
\begin{align}
    \|\ff{i}(\zbf)-\muu{i}_T(\zbf)\|_2
    &=\sqrt{\sum_{j=1}^{d_{i+1}}\big(\ff{i,j}(\zbf)-\muu{i,j}_T(\zbf)\big)^2}\\
    &\le\sqrt{\sum_{j=1}^{d_{i+1}}\big(B\sig{i,j}_T(\zbf)\big)^2}\\
    &=\sqrt{d_{i+1}}B\sig{i,1}_T(\zbf),
\end{align}
where we use the confidence bounds and the fact that $\sig{i,j}_T(\zbf) = \sig{i,1}_T(\zbf)$ by the symmetry of our setup.

Since each $\ff{i,j}$ is Lipschitz continuous with parameter $L$, we have that $\ff{i}$ is Lipschitz continuous with parameter $\sqrt{d_{i+1}}L$, and $\ff{m}\circ\cdots\circ\ff{i}$ is Lipschitz continuous with parameter $\sqrt{D_{i+1,m+1}}L^{m-i+1}$, where $D_{i,j}=\prod_{s=i}^j d_s$ and $d_{m+1}=1$. 
Then, defining $\sigtil_T^{(i,1)}(\xbf)=(\sig{i,1}_T\circ\muu{i-1}_T\circ\cdots\circ\muu{1}_T)(\xbf)$, we can follow \eqref{eq:fd_start}--\eqref{eq:fd_1} to obtain
\begin{align}
    g(\xbf)={}&(\ff{m}\circ\ff{m-1}\circ\cdots\circ\ff{2})\big(\ff{1}(\xbf)\big)\\
    \le{}& (\ff{m}\circ\ff{m-1}\circ\cdots\circ\ff{2})\big(\muu{1}_T(\xbf)\big) + \sqrt{D_{3,m+1}} L^{m-1} \|\ff{1}(\xbf)-\muu{1}_T(\xbf)\|_2\\
    ={}& (\ff{m}\circ\ff{m-1}\circ\cdots\circ\ff{2})\big(\muu{1}_T(\xbf)\big) + \sqrt{D_{2,m+1}} L^{m-1} B\sig{1,1}_T(\xbf)\\
    \le{}& (\ff{m}\circ\ff{m-1}\circ\cdots\circ\ff{3})\big((\muu{2}_T\circ\muu{1}_T)(\xbf)\big) + \sqrt{D_{3,m+1}}L^{m-2}B\sig{2,1}_T\big(\muu{1}_T(\xbf)\big)\nonumber\\
    &+ \sqrt{D_{2,m+1}}L^{m-1} B\sig{1,1}_T(\xbf)\\
    %\le{}& (\ff{m}\circ\ff{m-1}\circ\cdots\circ\ff{4})\big(\muu{3}_T\circ\muu{2}_T\circ\muu{1}_T(\xbf)\big) + \sqrt{D_{4,m+1}}L^{m-3}B\sig{3,1}_T\big(\muu{2}_T\circ\muu{1}_T(\xbf)\big) \nonumber\\
    %&+ \sqrt{D_{3,m+1}}L^{m-2}B\sig{2,1}_T\big(\muu{1}_T(\xbf)\big) + \sqrt{D_{2,m+1}}L^{m-1} B\sig{1,1}_T(\xbf)\\
    \le{}& (\muu{m}_T\circ\muu{m-1}_T\circ\cdots\circ\muu{1}_T)(\xbf)+B\sig{m,1}\big((\muu{m-1}_T\circ\cdots\circ\muu{1}_T)(\xbf)\big)\nonumber\\
    &+LB\sig{m-1}\big((\muu{m-2}_T\circ\cdots\circ\muu{1}_T)(\xbf)\big)+\cdots+ \sqrt{D_{2,m+1}}L^{m-1} B\sig{1,1}_T(\xbf)\\
    \le{}& (\muu{m}_T\circ\muu{m-1}_T\circ\cdots\circ\muu{1}_T)(\xbf) + B\sum_{i=1}^m \sqrt{D_{i+1,m+1}} L^{m-i} \sigtil_T^{(i,1)}(\xbf).\label{eq:fd_2}
\end{align}
Hence, we have
\begin{align}
    r^\ast_T&=g(\xbf^\ast)-g(\xbf^\ast_T)\le 2 B\sum_{i=1}^m \sqrt{D_{i+1,m}} L^{m-i} \max\{\sigtil^{(i,1)}_T(\xbf^\ast),\sigtil^{(i,1)}_T(\xbf^\ast_T)\}.
\end{align}
It remains to upper bound $\max\{\sigtil^{(i,1)}_T(\xbf^\ast),\sigtil^{(i,1)}_T(\xbf^\ast_T)\}$ for $i\in[m]$. Firstly, by \cref{lem:fill}, we have
\begin{align}
	\sigtil^{(1,1)}_T(\xbf)=\sig{1,1}_T(\xbf)=O\big((\delta_T)^\nu\big)=O(T^{-\nu/d}).\label{eq:sigtil_net_1}
\end{align}
We now represent the upper bound on $\sigtil^{(i+1,1)}_T(\cdot)$ in terms of $\{\sigtil^{(j,1)}_T(\cdot):j\in[i]\}$ for $i\in[m-1]$. For all $\xbf\in\Xc$, the distance between $\muu{i}_T\circ\cdots\circ\muu{1}_T(\xbf)$ and its closest point in $\Xc^{(i+1)}$ is
\begin{align}
	\delta_{i+1}(\xbf)&=\min_{\zbf\in\Xc^{(i+1)}}|(\muu{i}_T\circ\cdots\circ\muu{1}_T)(\xbf)-\zbf|\\
	&\le|(\muu{i}_T\circ\cdots\circ\muu{1}_T)(\xbf)-\ff{i}\circ\cdots\circ\ff{1}(\xbf)|\\
	&\le B\sum_{j=1}^{i}\sqrt{D_{j+1,i+1}}L^{i-j}\sigtil^{(j,1)}_T(\xbf),\label{eq:delta_i_net}
\end{align}
where the last step follows from \eqref{eq:fd_2}.
Then, the distance between $(\muu{i}_T\circ\cdots\circ\muu{1}_T)(\xbf)$ and its closest point in $\Xc^{(i+1)}_T$ is
\begin{align}
	\delta'_{i+1}(\xbf)&=\min_{s\in[T]}|(\muu{i}_T\circ\cdots\circ\muu{1}_T)(\xbf)-\xx{i+1}_s|\\
	&\le\delta^{(i+1)}_T+\delta_{i+1}(\xbf)\label{eq:delta_prime_net_1}\\
	&\le O(\sqrt{D_{2,i+1}}L^i T^{-1/d})+B\sum_{j=1}^{i}\sqrt{D_{j+1,i+1}}L^{i-j}\sigtil^{(j,1)}_T(\xbf)\label{eq:delta_prime_net_2}\\
	&= O\Big(\max\Big\{\sqrt{D_{2,i+1}}L^iT^{-1/d}, iB\cdot \max_{j\in[i]}\sqrt{D_{j+1,i+1}}L^{i-j} \sigtil^{(j,1)}_T(\xbf) \Big\}\Big),
\end{align}
where \eqref{eq:delta_prime_net_1} applies the triangle inequality, and \eqref{eq:delta_prime_net_2} uses \eqref{eq:fill_i_net} and \eqref{eq:delta_i_net}.

Now, for each $\xbf'\in\{\xbf^\ast,\xbf^\ast_T\}$, we extend $\Xc^{(i+1)}$ to $\widetilde{\Xc}^{(i+1)}$, the smallest hyperrectangle that covers $(\muu{i}_T\circ\cdots\circ\muu{1}_T)(\xbf')$ and $\Xc^{(i+1)}$. Then, the fill distance of the extended domain with regard to $\Xc^{(i+1)}_T$ is
\begin{align}
	\delta'_{i+1}=O\Big(\max\Big\{\sqrt{D_{2,i+1}}L^iT^{-1/d}, iB\cdot\max_{j\in[i]}\sqrt{D_{j+1,i+1}}L^{i-j} \sigtil^{(j,1)}_T(\xbf') \Big\}\Big),
\end{align}
and by \cref{lem:fill}, we have
\begin{align}
	\sigtil^{(i+1,1)}_T(\xbf')=O\big((\delta'_{i+1})^\nu\big)=O\Bigg(\Big(\max\Big\{\sqrt{D_{2,i+1}}L^iT^{-1/d}, iB\cdot\max_{j\in[i]}\sqrt{D_{j+1,i+1}}L^{i-j} \sigtil^{(j,1)}_T (\xbf') \Big\}\Big)^\nu\Bigg).\label{eq:sigtil_net}
\end{align}
We now consider two cases;

\underline{The case that $\nu \le 1$}: 
Recalling the assumption $B=\Theta(L)$, and using the fact that $\Otil(L^{c\nu})$ is equivalent to $O\big( (L^{c\nu} {\rm poly}(c\log L)) \big)$ for fixed $\nu$, it follows from \eqref{eq:sigtil_net_1} and \eqref{eq:sigtil_net} that when $\nu\le 1$, we have
\begin{align}
	\sigtil_T^{(2,1)}(\xbf')&=O\Big(\max\Big\{(d_2)^{\nu/2}L^\nu T^{-\nu/d}, (d_2)^{\nu/2} B^\nu T^{-\nu^2/d}\Big\}\Big)=O\big((d_2)^{\nu/2}L^\nu T^{-\nu^2/d}\big), \\
	\sigtil^{(3,1)}_T(\xbf')&=O\Big(\max\Big\{(D_{2,3})^{\nu/2}L^{2\nu}T^{-\nu/d}, 2^\nu (D_{2,3})^{\nu/2} B^\nu  L^\nu T^{-\nu^2/d}, 2^\nu (d_3)^{\nu/2}  B^\nu (d_2)^{\nu^2/2} L^{\nu^2}  T^{-\nu^3/d}\Big\}\Big)\nonumber\\
	&=O\big(2^\nu(D_{2,3})^{\nu/2} B^\nu L^\nu T^{-\nu^3/d}\big)=\Otil\big( (D_{2,3})^{\nu/2} B^\nu L^\nu T^{-\nu^3/d}\big),\\
	&\vdotswithin{=}\nonumber\\
	\sigtil^{(i,1)}_T(\xbf')&=\Otil\big((i-1)^\nu(D_{2,i})^{\nu/2} B^\nu L^{(i-2)\nu}T^{-\nu^i/d}\big)=\Otil\big((D_{2,i})^{\nu/2} B^\nu L^{(i-2)\nu}T^{-\nu^i/d}\big),
\end{align}
and the simple regret is
\begin{align}
	r^\ast_T&\le 2B\sum_{i=1}^m \sqrt{D_{i+1,m}} L^{m-i}\max\{\sigtil^{(i,1)}_T(\xbf^\ast),\sigtil^{(i,1)}_T(\xbf^\ast_T)\}\\
	&=\Otil(\sqrt{D_{2,m}}BL^{m-1}T^{-\nu^m/d}).
\end{align}

\underline{The case that $\nu > 1$}: 
With $B=\Theta(L)$, and using the property of $\Otil(\cdot)$ similar to the first case, it follows from \eqref{eq:sigtil_net_1} and \eqref{eq:sigtil_net} that when $\nu > 1$, we have
\begin{align}
	\sigtil_T^{(2,1)}(\xbf')&=O\Big(\max\Big\{(d_2)^{\nu/2}L^\nu T^{-\nu/d}, (d_2)^{\nu/2} B^\nu T^{-\nu^2/d}\Big\}\Big)=O((d_2)^{\nu/2}L^\nu T^{-\nu/d}), \\
	\sigtil^{(3,1)}_T(\xbf')&=O\Big(\max\Big\{(D_{2,3})^{\nu/2}L^{2\nu}T^{-\nu/d}, (d_3)^{\nu/2}  B^\nu (d_2)^{\nu^2/2} L^{\nu^2} T^{-\nu^2/d}\Big\}\Big)\nonumber\\
	&=O\big(2^\nu(d_3)^{\nu/2} (d_2)^{\nu^2/2} B^\nu L^{\nu^2} T^{-\nu/d}\big)=\Otil\big((d_3)^{\nu/2} (d_2)^{\nu^2/2} B^\nu L^{\nu^2} T^{-\nu/d}\big),\\
	&\vdotswithin{=}\nonumber\\
	\sigtil^{(i,1)}_T(\xbf')&=\Otil\Big(\max\Big\{(D_{2,i})^{\nu/2}L^{(i-1)\nu}T^{-\nu/d}, (i-1)^\nu (d_i)^{\nu/2}(d_{i-1})^{\nu^2/2}\cdots(d_2)^{\nu^{i-1}/2}B^{\nu+\cdots+\nu^{i-2}}L^{\nu^{i-1}}T^{-\nu^2/d}\Big\}\Big)\nonumber\\
	&=\Otil\Big(\max\Big\{(D_{2,i})^{\nu/2}L^{(i-1)\nu}T^{-\nu/d}, (d_i)^{\nu/2}(d_{i-1})^{\nu^2/2}\cdots(d_2)^{\nu^{i-1}/2}B^{\nu+\cdots+\nu^{i-2}}L^{\nu^{i-1}}T^{-\nu^2/d}\Big\}\Big),
\end{align}
and the simple regret is
\begin{align}
	r^\ast_T&\le 2B\sum_{i=1}^m \sqrt{D_{i+1,m}} L^{m-i}\max\{\sigtil^{(i,1)}_T(\xbf^\ast),\sigtil^{(i,1)}_T(\xbf^\ast_T)\}\\
	&=\Otil\Big(\max\Big\{(D_{2,m})^{\nu/2}BL^{(m-1)\nu}T^{-\nu/d}, \Dtil_{2,m}^\nu B^{1+\nu+\nu^2+\cdots+\nu^{m-2}}L^{\nu^{m-1}}T^{-\nu^2/d} \Big\}\Big),
\end{align}
where $\Dtil_{2,m}^\nu=\prod_{i=2}^m (d_i)^{\nu^{m+1-i}/2}$.
\end{proof}

Next, we recall the following two restrictive cases introduced in the main body:
\begin{itemize}
	\item \textbf{Case 1}: We additionally assume that $(\muu{i}_T\circ\cdots\circ\muu{1}_T)(\xbf^\ast)\in\Xc^{(i+1)}$ and $(\muu{i}_T\circ\cdots\circ\muu{1}_T)(\xbf^\ast_T)\in \Xc^{(i+1)}$ for $i\in[m-1]$.
	\item \textbf{Case 2}: We additionally assume that all the $\Xc^{(i)}$'s are known. Defining
	\begin{align}
		\muutil{i}_T(\zbf) &= \argmin_{\zbf'\in\Xc^{(i+1)}}|\muu{i}_T(\zbf)-\zbf'|,
	\end{align}
	we let the algorithm return
	\begin{align}
		\xbf^\ast_T=\argmax_{\xbf\in\Xc}(\muutil{m}_T\circ\cdots\circ\muutil{1}_T)(\xbf).
	\end{align}
\end{itemize}
Similarly to chains and multi-output chains, in either case, it follows that
\begin{align}
	r^\ast_T\le 2 B\sum_{i=1}^m \sqrt{D_{i+1,m}} L^{m-i} \max\{\sigtil^{(i,1)}_T(\xbf^\ast),\sigtil^{(i,1)}_T(\xbf^\ast_T)\} =O\Big( B\sum_{i=1}^m \sqrt{D_{i+1,m}} L^{m-i} \sigbar_T^{(i,1)} \Big).
\end{align}
Moreover, by \cref{cor:fill_net}, we can further upper bound this by
\begin{align}
	r^\ast_T\le O\Big( B\sum_{i=1}^m \sqrt{D_{i+1,m}} L^{m-i} \sigbar_T^{(i,1)} \Big)=
	\begin{cases}
		O(\sqrt{D_{2,m}}BL^{m-1}T^{-\nu/d})&\text{ when $\nu\le 1$,}\\
		O\big((D_{2,m})^{\nu/2}BL^{(m-1)\nu}T^{-\nu/d}\big)&\text{ when $\nu> 1$.}
	\end{cases}
\end{align}

\section{Lower Bound on Simple Regret (Proof of \cref{thm:lower_sim})}
\label{sec:lower_sim}

Recall the high-level intuition behind our lower bound outlined in Section \ref{sec:lower}: We use the idea from \cite{bull2011convergence} of having a small bump that is difficult to locate, but unlike \cite{bull2011convergence}, we exploit Lipschitz functions at the intermediate layers to ``amplify'' that bump (which means the original bump can have a much narrower width for a given RKHS norm).

We first show that, for fixed $\epsilon>0$, we can construct a base function $\gbar$ with height $2\epsilon$ and support radius $w=\Theta\big((\frac{\epsilon}{ BL^{m-1}})^{1/\nu}\big)$ for each network structure.
 
\subsection{Hard Function for Chains} \label{sec:hard_chains}

To construct a base function with chain structure, we define the following two scalar-valued functions.

Considering the ``bump'' function $h(\xbf)=\exp\big(\frac{-1}{1-\|\xbf\|_2^2}\big)\mathds{1}\{\|\xbf\|_2< 1\}$, for some $\epsilon_1>0$ and width $w>0$, we define \cite{bull2011convergence}
\begin{align}
    \htil(\xbf,\epsilon_1,w)&=\frac{2\epsilon_1}{h(\mathbf{0})}h\Big(\frac{\xbf}{w}\Big),\label{eq:htil}
\end{align}
which is a scaled bump function with height $2\epsilon_1$ and compact support $\{\xbf\in\Xc:\|\xbf\|_2<w\}$.

For a fixed $u>0$ and $\Ltil=\sqrt{2}B/\sqrt{k(0)-k(2u)}$, where $k$ is the Mat\'ern kernel on $\RR$ with smoothness $\nu$ and $k(|x-x'|)=k(x,x')$, we define
\begin{align}
    \gtil_k(\cdot)& = \frac{\Ltil}{2} \big(k(\cdot,u)-k(\cdot,-u)\big).\label{eq:gtilk}
\end{align}

\begin{restatable}[Hard function for chains]{thm}{hardchain}
\label{thm:hard_chain}
For the Mat\'ern kernel $k$ with smoothness $\nu\ge1$, sufficiently small $\epsilon>0$, and sufficiently large $B$, there exists $\epsilon_1>0,L=\Theta(B), c=\Theta(1), w=\Theta\big((\frac{\epsilon}{ B(cL)^{m-1}})^{1/\nu}\big),$ and $u>0$ such that $\gbar:=\fbar^{(m)}\circ\fbar^{(m-1)}\circ\cdots\circ\fbar^{(1)}$ with $\fbar^{(1)}(\cdot)=\htil(\cdot,\epsilon_1,w)$ and $\fbar^{(s)}(\cdot)=\gtil_k(\cdot)$ for $s\in [2,m]$ has the following properties:
\begin{itemize}
    \item $\fbar^{(i)}\in\Hc_k(B)\cap\Fc(L)$ for each $i\in[m]$,
    \item $\max_{\xbf} \gbar(\xbf)= 2\epsilon$,
    \item $\gbar(\xbf)>0$ when $\|\xbf\|_2<w$ , and $\gbar(\xbf)=0$ otherwise.
\end{itemize}
\end{restatable}
\begin{proof}
For $k$ being the Mat\'ern kernel with smoothness $\nu$, recalling that
\begin{align}
    \htil(\xbf,\epsilon_1,w)&=\frac{2\epsilon_1}{h(\mathbf{0})}h\Big(\frac{\xbf}{w}\Big)
\end{align}
with $h$ being the bump function, \citep[Section A.2]{bull2011convergence} has shown that for some constant $C_1$,
\begin{align}
    \|\htil\|_k\le C_1 \frac{2\epsilon_1}{h(\mathbf{0})} \Big(\frac{1}{w}\Big)^\nu \|h\|_k.
\end{align}
Hence, we have $\|\htil\|_k\le B$ when $w=\big(\frac{2C_1\|h\|_k\epsilon_1}{h(\mathbf{0})B}\big)^{1/\nu}$. When $\frac{\epsilon_1}{B}$ is sufficiently small, the diameter of the support satisfies $2w\ll 1$. Since the bump function is infinitely differentiable, $\htil$ is Lipschitz continuous with some constant $L'=\Theta(\frac{\epsilon_1}{w})$, and our assumption of $\nu\ge 1$ further implies $L'=O(B)$.

Recall that for a fixed $u>0$ and the Mat\'ern kernel $k$ on a one-dimension domain, we define
\begin{align}
    \gtil_k(\cdot)& = \frac{\Ltil}{2} \big(k(\cdot,u)-k(\cdot,-u)\big)
\end{align}
with $\Ltil=\frac{\sqrt{2}B}{\sqrt{k(0)-k(2u)}}$. Applying \eqref{eq:norm} with $\gtil_k$ replacing both $f$ and $f'$, we have
\begin{align}
    \|\gtil_k\|_k=\frac{\Ltil}{2}\sqrt{k(u,u)+2k(u,-u)+k(-u,-u)}=\frac{\Ltil}{2}\sqrt{2k(0)-2k(2u)}=B,
\end{align}
where $k(d_{\xbf,\xbf'})=k(\xbf,\xbf')$ with $d_{\xbf,\xbf'}=\|\xbf-\xbf'\|_2$.

We choose the two constants $u>\util>0$ to satisfy the following: 
\begin{enumerate}
    \item $k(d_{\xbf,\xbf'})$ is non-increasing when $d_{\xbf,\xbf'}\in[u-\util,u+\util]$. 
    \item $k(u-\util)-k(u+\util)\ge\frac{2}{\Ltil}=\frac{\sqrt{2}\sqrt{k(0)-k(2u)}}{B}$.
    \item Defining $r_{\max}=\sup_{z\in(0,\util]}\frac{k(u-z)-k(u+z)}{2z}$ and $r_{\min}=\inf_{z\in(0,\util]}\frac{k(u-z)-k(u+z)}{2z}$, it holds that $r_{\max}=\Theta(1)$ and $r_{\min}=\Theta(1)$.
\end{enumerate}
Whenever $k(d_{\xbf,\xbf'})$ is continuous and non-increasing in $d_{\xbf,\xbf'}$ on $\RR^+$ (e.g., the Mat\'ern kernel), condition 1 is satisfied. This condition guarantees that $k(\cdot, u)$ is non-decreasing on $[0,\util]$, and $k(\cdot, -u)$ is non-increasing on $[0,\util]$, which together implies that $\gtil_k$ is non-decreasing on $[0,\util]$. Under condition 1, condition 2 is satisfied as long as $B$ is sufficiently large, and it implies that $\gtil_k(\util)\ge 1$. Examples of $\gtil_k$ for $k=k_\text{Mat\'ern}$ are given in Figure \ref{fig:gtil}.

Let $L=\max\{L',\sup_{z\in(0,\util]}\frac{\gtil_k(z)}{z}\}=\max\{L',r_{\max}\Ltil\}$ and $\alpha=\inf_{z\in(0,\util]}\frac{\gtil_k(z)}{zL}=\frac{r_{\min}\Ltil}{L}$. Then, it holds for all $z\in(0,\util]$ that
\begin{align}
    1<\alpha L\le \frac{\gtil_k(z)}{z} \le L. \label{eq:slope}
\end{align}
Clearly $\alpha$ cannot exceed $1$, and moreover, condition 3 above along with $L'=O(B)$ implies that $\alpha=\min\big\{r_{\min}\sqrt{\frac{2}{k(0)-k(2u)}}\frac{B}{L'},\frac{r_{\min}}{r_{\max}}\big\}=\Theta(1)$.

\begin{figure*}[t!]
%	\minipage[t]{0.33\textwidth}
%	\includegraphics[width=\linewidth]{figs/g_025.pdf}
%	\caption*{$\nu=1/4$}
%	\endminipage
%	\minipage[t]{0.33\textwidth}
%	\includegraphics[width=\linewidth]{figs/g_033.pdf}
%	\caption*{$\nu=1/3$}
%	\endminipage
%	\minipage[t]{0.33\textwidth}
%	\includegraphics[width=\linewidth]{figs/g_050.pdf}
%	\caption*{$\nu=1/2$}
%	\endminipage
%	\newline
	\minipage[t]{0.33\textwidth}
	\includegraphics[width=\linewidth]{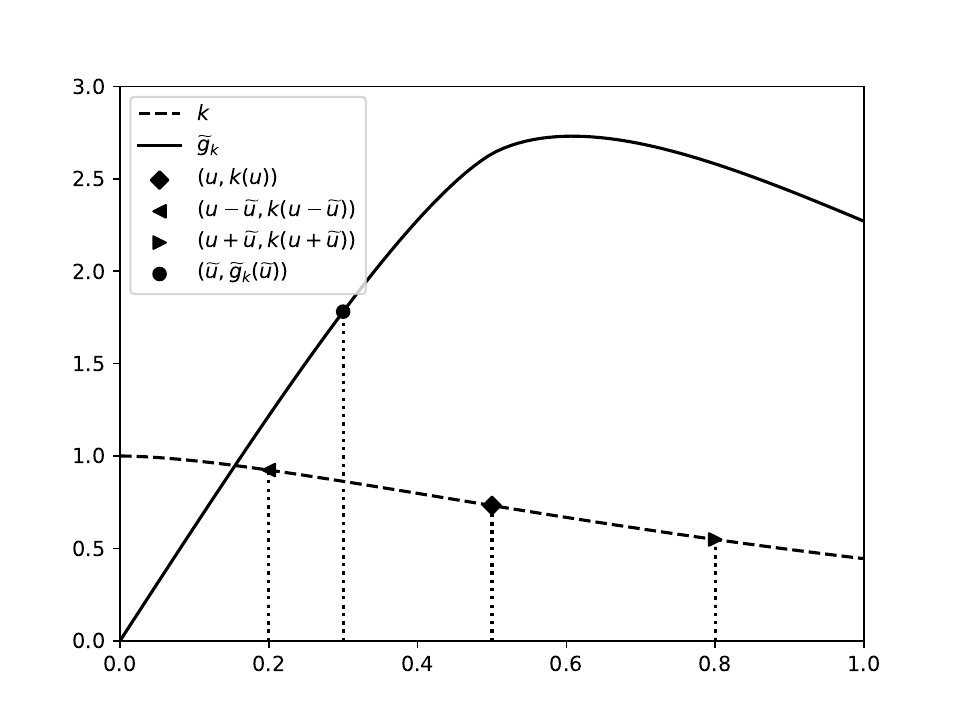}
	\caption*{$\nu=1$}
	\endminipage
	\minipage[t]{0.33\textwidth}
	\includegraphics[width=\linewidth]{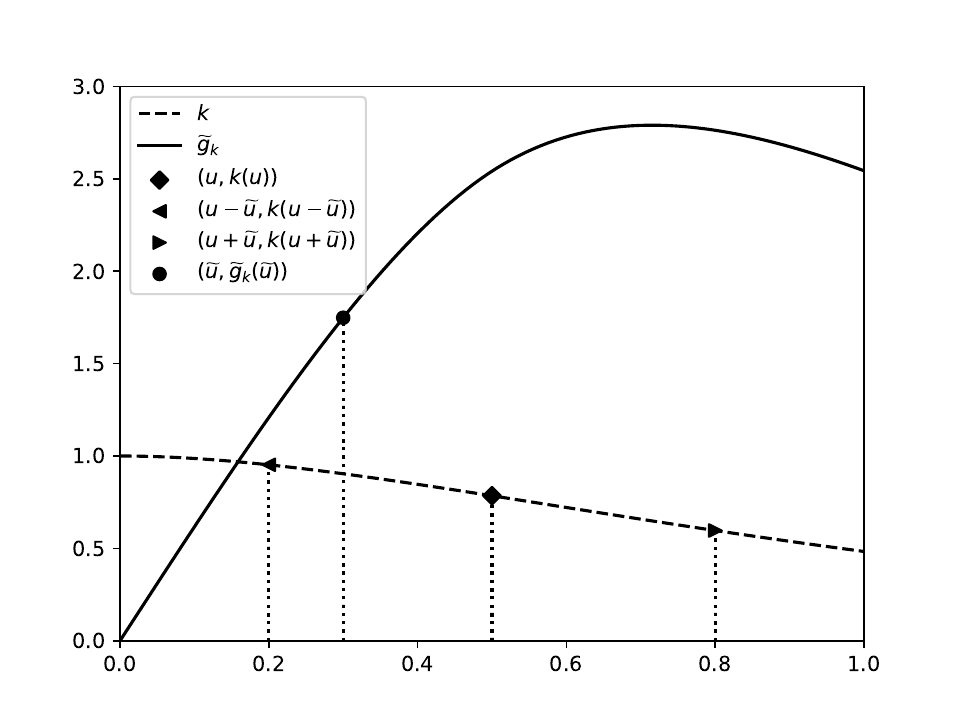}
	\caption*{$\nu=3/2$}
	\endminipage
	\minipage[t]{0.33\textwidth}
	\includegraphics[width=\linewidth]{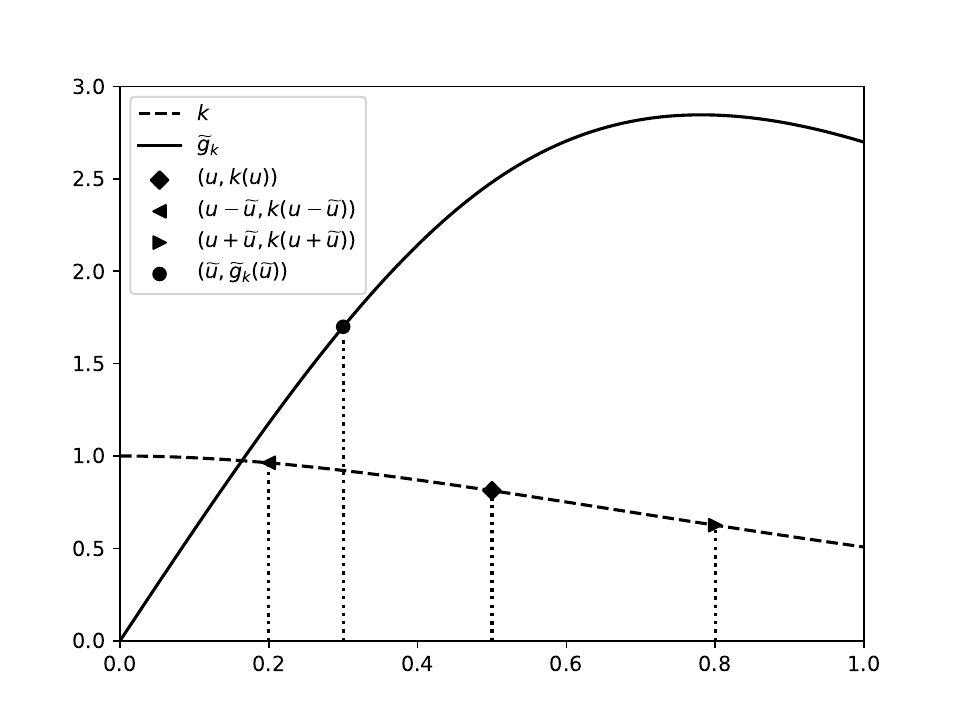}
	\caption*{$\nu=2$}
	\endminipage
	\caption{$k=k_\text{Mat\'ern}$ with lengthscale $l=1$ and smoothness $\nu$, and $\gtil_k$ with $B=5$, $u=0.5$, and $\util=0.3$.}
	\label{fig:gtil}
\end{figure*}

% Now we discuss the height and support of $\gbar$. Since $\eta>1$, there exists $L>\max\{L',\eta\}$, $\alpha\in(\frac{1}{L},\frac{\eta}{L})$, and $\util>0$ such that the following holds for all $z\in[0,\util]$,
% \begin{align}
%     1<\alpha L\le \frac{\partial}{\partial z}\gtil_k(z)\le L.\label{eq:slope}
% \end{align}
% When $\nu>1$, $k_\text{Mat\'ern}$ is continuously differentiable and Lipschitz continuous on $\RR$ \cite{Shekhar20}. With $L$ being the maximum of the Lipschitz constant and $L'$, there must exist $\alpha\in(\frac{1}{L},\frac{\eta}{L})$ and $\util\in(0,u)$ such that \eqref{eq:slope} holds for all $z\in[0,\util]$. When $\nu\le 1$, as shown in \cref{fig:gtil}, there exists $\util\in(0,u)$ such that $\gtil_k$ is non-decreasing on $[0,\util]$, and therefore \eqref{eq:slope} holds on $[0,\util]$ for some constant $L$ and $\alpha\in(\frac{1}{L},\frac{\eta}{L})$.

Defining $\gbar_i=\fbar^{(i)}\circ\fbar^{(i-1)}\circ\dots\circ\fbar^{(1)}$, we denote the height of $\gbar_i$ by $h_i=\max_\xbf \gbar_i(\xbf)$. If $h_{i-1}\le\util$, \eqref{eq:slope} implies
\begin{align}
    h_{i-1}<\alpha L h_{i-1}\le h_i \le L h_{i-1}.\label{eq:height}
\end{align}
Then, given $\epsilon\in(0, \gtil_k(\util)/2]$ (recalling that we assume $\epsilon \in (0,1/2]$ and noting that condition 2 implies $\gtil_k(\util)\ge 1$), we choose $\epsilon_1$ to satisfy 
\begin{align}
    (\fbar^{(m)} \circ \cdots \circ \fbar^{(2)}) (2\epsilon_1)=(\gtil_k \circ \cdots \circ \gtil_k) (2\epsilon_1)=2\epsilon,
\end{align}
By $h_1=2\epsilon_1$ and \eqref{eq:height}, this choice of $\epsilon_1$ must also satisfy 
\begin{align}
    2 (\alpha L)^{m-1} \epsilon_1 = (\alpha L)^{m-1} h_1 \le h_m = 2\epsilon \le L^{m-1} h_1 = 2 L^{m-1} \epsilon_1,
\end{align}
implying (via \eqref{eq:slope}) that
\begin{align}
    \frac{\epsilon}{L^{m-1}} \le \epsilon_1 \le \frac{\epsilon}{(\alpha L)^{m-1}}<\epsilon. \label{eq:alpha_fractions}
\end{align}
Since $\htil$ has a compact support with radius $w$ and $\gtil(0)=0$, $\gbar$ also has a compact support with radius
\begin{align}
    w = \Theta\bigg(\Big(\frac{\epsilon_1}{B}\Big)^{1/\nu}\bigg) = \Theta \bigg(\Big(\frac{\epsilon}{B(cL)^{m-1}}\Big)^{1/\nu}\bigg),
\end{align}
for some constant $c=\Theta(1)$.

Lastly, \eqref{eq:slope} implies $\gtil_k$ on $[0,\util]$ is a member of $\Fc(L)$, and $L\ge L'$ guarantees $\htil\in\Fc(L)$.
\end{proof}

\subsection{Hard Function for Multi-Output Chains}

In this section, we consider the operator-valued Mat\'ern kernel $\Gamma^{(i)}(\xbf,\xbf')=k^{(i)}(\xbf,\xbf')\Ibf^{(i+1)}$, where $k^{(i)}$ is the scalar-valued Mat\'ern kernel on $\RR^{d_i}$ and $\Ibf^{(i+1)}$ is the identity matrix of size $d_{i+1}$. Then, for a fixed $\ubf^{(i)}\in\RR^{d_i}$ and $\Ltil^{(i)}=\sqrt{2}B/\sqrt{k^{(i)}(0)-k^{(i)}(2\|\ubf^{(i)}\|_2)}$ where $k^{(i)}(\|\xbf-\xbf'\|_2)=k^{(i)}(\xbf,\xbf')$, we define
\begin{align}
    \gtil^{(i)}(\cdot)& = \frac{\Ltil^{(i)}}{2} \big(\Gamma^{(i)}(\cdot,\ubf^{(i)})-\Gamma^{(i)}(\cdot,-\ubf^{(i)})\big).
\end{align}

\begin{restatable}[Hard function for multi-output chains]{thm}{hardmul}
\label{thm:hard_mul}
Let $\Ibf^{(i)}$ denote the identity matrix of size $d_i$, let $\ebf^{(i)}_1$ denote the first column of $\Ibf^{(i)}$, and let $k^{(i)}$ denote the scalar-valued Mat\'ern kernel on $\RR^{d_i}$ with smoothness $\nu\ge1$. For $\Gamma^{(i)}(\xbf,\xbf')=k^{(i)}(\xbf,\xbf')\Ibf^{(i+1)}$, sufficiently small $\epsilon>0$, and sufficiently large $B$, there exists $\epsilon_1>0,L=\Theta(B),c=\Theta(1),w=\Theta\big((\frac{\epsilon}{ B(cL)^{m-1}})^{1/\nu}\big),$ and $u\in\RR$ such that $\gbar:=\fbar^{(m)}\circ\fbar^{(m-1)}\circ\cdots\circ\fbar^{(1)}$ with $\fbar^{(1)}(\cdot)=\htil(\cdot,\epsilon_1,w)\ebf_1^{(2)}$ and $\fbar^{(s)}(\cdot)=\gtil^{(s)}(\cdot)\ebf_1^{(s+1)}$ with $\ubf^{(s)}=u\ebf_1^{(s)}$ for $s\in [2,m]$ has the following properties:
\begin{itemize}
    \item $\fbar^{(i)}\in\Hc_{\Gamma^{(i)}}(B)\cap\Fc(L)$ for each $i\in[m]$, 
    \item $\max_{\xbf} \gbar(\xbf)= 2\epsilon$,
    \item $\gbar(\xbf)>0$ when $\|\xbf\|_2<w$ , and $\gbar(\xbf)=0$ otherwise.
\end{itemize}
\end{restatable}

\begin{proof}
The rough idea is to reduce to the case of regular chains by only making use of a single coordinate throughout the network.  We leave open the question as to whether the lower bound can be improved by utilizing all coordinates.
    
With $\Ibf$ denoting the identity matrix of size $d_2$ and $\ebf_1$ denoting the first column of $\Ibf$, we aim to show that $\fbar^{(1)}\in\Hc_{\Gamma^{(1)}}(B)$ by showing that if $\Gamma(\cdot,\cdot)=k(\cdot,\cdot)\Ibf\in\RR^{d_2\times d_2}$ and $\htil\in\Hc_k(B)$ for some scalar-valued kernel $k$ and some constant $B$, then the function $\htil'(\cdot)=\htil(\cdot)\ebf_1$ satisfies $\htil' \in\Hc_\Gamma(B)$. Since $\htil\in\Hc_k(B)$, there exists a sequence $\{(a_i,\xbf_i)\}_{i=1}^\infty$ such that $\htil(\cdot)=\sum_{i=1}^\infty a_i k(\cdot,\xbf_i)$. Then, using the definition of RKHS norm for vector-valued functions in \eqref{eq:mul_norm}, we have
\begin{align}
    \|\htil'\|_\Gamma^2&=\lim_{n\to\infty}\Big\|\sum_{i=1}^n a_i k(\cdot,\xbf_i)\ebf_1\Big\|_\Gamma^2\\
    &=\lim_{n\to\infty}\Big\|\sum_{i=1}^n a_i \Gamma(\cdot,\xbf_i)\ebf_1\Big\|_\Gamma^2\\
    &=\sum_{i,j=1}^\infty \langle \Gamma(\xbf_i, \xbf_j)(a_i\ebf_1),a_j\ebf_1\rangle \\
    &=\sum_{i,j=1}^\infty a_i a_j k(\xbf_i,\xbf_j)\\
    &=\|\htil\|_k^2 \\
    &\le B^2,
\end{align}
and therefore $\htil'\in\Hc_\Gamma(B)$. Next, for $i\in[2,m]$, with $\ubf^{(i)}=u\ebf_1^{(i)}$ we have
\begin{align}
    \|\fbar^{(i)}\|_\Gamma &= \frac{\Ltil^{(i)}}{2}\sqrt{ \langle \Gamma^{(i)}(\ubf^{(i)},\ubf^{(i)})\ebf_1,\ebf_1\rangle - 2\langle \Gamma^{(i)}(\ubf^{(i)},-\ubf^{(i)})\ebf_1,\ebf_1\rangle + \langle \Gamma^{(i)}(-\ubf^{(i)},-\ubf^{(i)})\ebf_1,\ebf_1\rangle } \\
    &= \frac{\Ltil^{(i)}}{2}\sqrt{k^{(i)}(\ubf^{(i)},\ubf^{(i)})-2k^{(i)}(\ubf^{(i)},-\ubf^{(i)})+k^{(i)}(-\ubf^{(i)},-\ubf^{(i)})}\\
    &= \frac{\Ltil^{(i)}}{2}\sqrt{2k^{(i)}(0)-2k^{(i)}(2u)}\\
    &=B.
\end{align}
Since $\htil\in\Fc(L')$ for some $L'>1$, we also have $\htil'\in\Fc(L')$.

We reuse the choice of $u$ and $\util$ in the previous case. Then, for any $z\in\RR$, with $\zbf^{(i)}=z\ebf_1^{(i)}$ and $\ubf^{(i)}=u\ebf_1^{(i)}$, we have
\begin{align}
    \fbar^{(i)}(\zbf)&=\gtil^{(i)}(\zbf^{(i)},\ubf^{(i)})\ebf_1^{(i+1)}\\
    &=\frac{\Ltil^{(i)}}{2}\big(\Gamma^{(i)}(\zbf^{(i)},\ubf^{(i)})-\Gamma^{(i)}(\zbf^{(i)},-\ubf^{(i)})\big)\ebf_1^{(i+1)}\\
    &=\frac{\Ltil^{(i)}}{2}\big(k^{(i)}(\zbf^{(i)},\ubf^{(i)})-k^{(i)}(\zbf^{(i)},-\ubf^{(i)})\big)\ebf_1^{(i+1)}\\
    &=\frac{\Ltil^{(i)}}{2}\big(k(\|\zbf^{(i)}-\ubf^{(i)}\|)-k(\|\zbf^{(i)}+\ubf^{(i)}\|)\big)\ebf_1^{(i+1)}\\
    &=\frac{\Ltil}{2}\big(k(|z-u|)-k(|z+u|)\big)\ebf_1^{(i+1)}\\
    &=\gtil_k(z)\ebf_1^{(i+1)},
\end{align}
where $\gtil_k(\cdot)$ depending on $u$ is defined in \eqref{eq:gtilk}. Hence, as illustrated in \cref{fig:mul_lower}, for any input $\xbf\in[0,1]^d$ of $\gbar$, we have
\begin{align}
    \xx{2,1}&=\htil(\xbf),\\
    \xx{2,j}&=0 &&\hspace{-25ex}\text{ for }j\ge 2, \\
    \xx{i+1,1}&=\gtil_k(\xx{i,1})&&\hspace{-25ex}\text{ for }i\ge 2,\\
    \xx{i+1,j}&=0&&\hspace{-25ex}\text{ for }i\ge 2\text{ and }j\ge 2.
\end{align}
\begin{figure}[t!]
	\centering
	\includegraphics[width=0.6\textwidth]{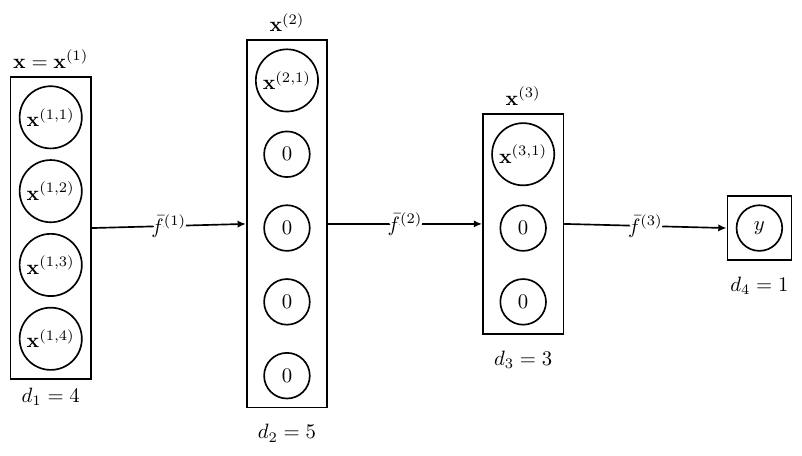}
	\caption{Illustration of $\gbar$ for multi-output chain.}
	\label{fig:mul_lower}
\end{figure}

By a similar argument to the case of single-output chains, there exists $L\ge L'$ and $\alpha=\Theta(1)$ such that for all $z\in[0,\util]$,
\begin{align}
    1<\alpha L\le \frac{\gtil_k(z)}{z}\le L.\label{eq:mul_slope}
\end{align}

For $2\epsilon\le\gtil_k(\util)$, we choose $\epsilon_1$ to satisfy
\begin{align}
    (\fbar^{(m)}\circ\fbar^{(m-1)}\circ\cdots\circ\fbar^{(2)})(2\epsilon_1\ebf_1^{(2)})=(\gtil_k\circ\cdots\circ\gtil_k)(2\epsilon_1)=2\epsilon,
\end{align}
and $\gbar$ has a compact support with radius
\begin{align}
    w = \Theta\Bigg(\Big(\frac{\epsilon_1}{B}\Big)^{1/\nu}\Bigg) = \Theta \Bigg(\Big(\frac{\epsilon}{B(cL)^{m-1}}\Big)^{1/\nu}\Bigg),
\end{align}
for some constant $c=\Theta(1)$.

Lastly, since $L\ge L'$, we immediately deduce that $\fbar^{(i)}\in\Fc(L)$ on its domain for each $i\in[m]$.
\end{proof}

\subsection{Hard Function for Feed-Forward Networks}

For a fixed $\ubf^{(i)}\in\RR^{d_i}$ and $\Ltil^{(i)}=\sqrt{2}B/\sqrt{k^{(i)}(0)-k^{(i)}(2\|\ubf^{(i)}\|_2)}$, where $k^{(i)}(\|\xbf-\xbf'\|_2)=k^{(i)}(\xbf,\xbf')$, we define
\begin{align}
    \gtil_k^{(i)}(\cdot)& = \frac{\Ltil^{(i)}}{2} \big(k^{(i)}(\cdot,\ubf^{(i)})-k^{(i)}(\cdot,-\ubf^{(i)})\big).
\end{align}

\begin{restatable}[Hard function for feed-forward networks]{thm}{hardnet}
\label{thm:hard_net}
Let $\ebf_1^{(i)}$ denote the first column of identity matrix of size $d_i$. For the Mat\'ern kernel $k$ with smoothness $\nu\ge 1$, sufficiently small $\epsilon>0$, and sufficiently large $B$, there exists $\epsilon_1>0,L=\Theta(B),c=\Theta(1), w=\Theta\big((\frac{\epsilon}{ B(cL)^{m-1}})^{1/\nu}\big),$ and $u>0$ such that $\gbar:=\fbar^{(m)}\circ\fbar^{(m-1)}\circ\cdots\circ\fbar^{(1)}$ with $\fbar^{(i)}(\cdot)=[\fbar^{(i,j)}(\cdot)]_{j=1}^{d_{i+1}}$ for $i\in[m]$, $\fbar^{(1,1)}(\cdot)=\htil(\cdot,\epsilon_1,w)$, $\fbar^{(s,1)}(\cdot)=\gtil_k^{(s)}(\cdot)$ with $\ubf^{(s)}=u\ebf_1^{(s)}$, and $\fbar^{(s,r)}(\cdot)=0$ for $s\in [2,m]$ and $r\neq 1$ has the following properties:
\begin{itemize}
    \item $\fbar^{(i,j)}\in\Hc_{k}(B)\cap\Fc(L)$ for each $i\in[m],j\in[d_{i+1}]$;
    \item $\max_{\xbf} \gbar(\xbf)= 2\epsilon$;
    \item $\gbar(\xbf)>0$ when $\|\xbf\|_2<w$ , and $\gbar(\xbf)=0$ otherwise.
\end{itemize}
\end{restatable}

\begin{proof}
We adopt a similar general approach to the case of chains and multi-output chains, but with some different details.  

As noted in our analysis of chains, we have $\htil\in\Hc_{k}(B)$ and $\htil\in\Fc(L')$ for some constant $L'>1$, and we also have
\begin{align}
    \|\gtil^{(i)}_k\|_k&=\frac{\Ltil^{(i)}}{2}\sqrt{k^{(i)}(\ubf^{(i)},\ubf^{(i)})-2k^{(i)}(\ubf^{(i)},-\ubf^{(i)})+k^{(i)}(-\ubf^{(i)},-\ubf^{(i)})}\\
    &=\frac{\Ltil^{(i)}}{2}\sqrt{2k^{(i)}(0)-2k^{(i)}(2\|\ubf^{(i)}\|_2)}\\
    &= B.
\end{align}
Reusing the choice of $u$ and $\util$ in the case of chains, with $\zbf^{(i)}=z\ebf_1^{(i)}$ and $\ubf^{(i)}=u\ebf_1^{(i)}$, we have
\begin{align}
    \gtil^{(i)}_k(\zbf^{(i)})&= \frac{\Ltil^{(i)}}{2} \big(k^{(i)}(\zbf^{(i)},\ubf^{(i)})-k^{(i)}(\zbf^{(i)},-\ubf^{(i)})\big)\\
    &= \frac{\Ltil^{(i)}}{2} \big(k^{(i)}(\|\zbf^{(i)}-\ubf^{(i)}\|_2)-k^{(i)}(\|\zbf^{(i)}+\ubf^{(i)}\|_2)\big)\\
    &= \frac{\Ltil^{(i)}}{2} \big(k^{(i)}(|z-u|)-k^{(i)}(|z+u|)\big)\\
    &= \gtil_k(z)
\end{align}
\begin{figure}[t!]
	\centering
	\includegraphics[width=0.6\textwidth]{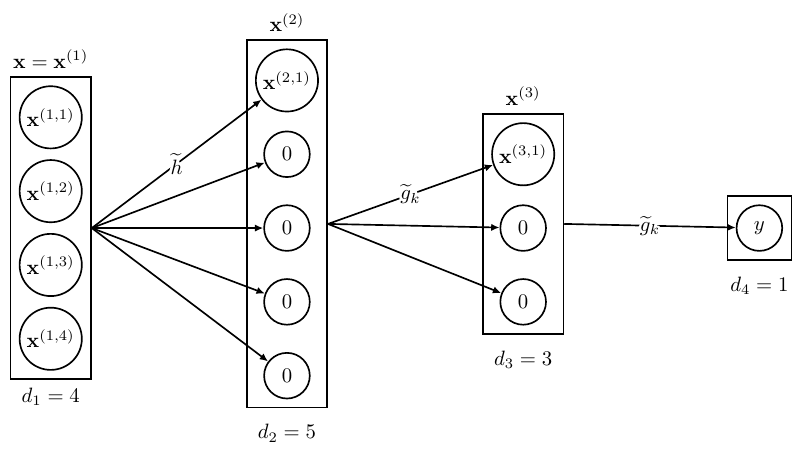}
	\caption{Illustration of $\gbar$ for feed-forward network.}
	\label{fig:net_lower}
\end{figure}

Hence, as illustrated in \cref{fig:net_lower}, for any input $\xbf\in[0,1]^d$ of $\gbar$, we have
\begin{align}
    \xx{2,1}&=\htil(\xbf),\\
    \xx{2,j}&=0 &&\hspace{-25ex}\text{ for }j\ge 2, \\
    \xx{i+1,1}&=\gtil_k(\xx{i,1})&&\hspace{-25ex}\text{ for }i\ge 2,\\
    \xx{i+1,j}&=0&&\hspace{-25ex}\text{ for }i\ge 2\text{ and }j\ge 2.
\end{align}

By a similar argument to the previous cases, there exists $L\ge L'$ and $\alpha=\Theta(1)$ such that for all $z\in[0,\util]$,
\begin{align}
    1<\alpha L\le \frac{\gtil_k(z)}{ z}\le L.\label{eq:net_slope}
\end{align}

For $2\epsilon\le\gtil_k(\util)$, we choose $\epsilon_1$ satisfying
\begin{align}
    (\fbar^{(m)}\circ\fbar^{(m-1)}\circ\cdots\circ\fbar^{(2)})(2\epsilon_1\ebf_1^{(2)})=(\gtil_k\circ\cdots\circ\gtil_k)(2\epsilon_1)=2\epsilon,
\end{align}
and $\gbar$ has a compact support with radius
\begin{align}
    w = \Theta\Bigg(\Big(\frac{\epsilon_1}{B}\Big)^{1/\nu}\Bigg) = \Theta \Bigg(\Big(\frac{\epsilon}{B(cL)^{m-1}}\Big)^{1/\nu}\Bigg),
\end{align}
for some constant $c=\Theta(1)$.

Lastly, due to $L\ge L'$, $\fbar^{(i,j)}\in\Fc(L)$ on its domain for each $i\in[m],j\in[d_{i+1}]$.
\end{proof}

\subsection{Lower Bound on Simple Regret}

With the preceding ``hard functions'' established, the final step is to essentially follow that of \cite{bull2011convergence}.  We provide the details for completeness.

By splitting the domain $\Xc=[0,1]^d$ into a grid of dimension $d$ with spacing $2w$, we construct $M=(\lfloor \frac{1}{2w} \rfloor)^d$ functions with disjoint supports by shifting the origin of $\gbar$ to the center of each cell and cropping the shifted function into $[0,1]^d$, which is denoted by $\Gc=\{g_1,\dots,g_M\}$. For $g$ sampled uniformly from $\Gc$, we first show that the expected simple regret $\EE[r_T^\ast]$ of an arbitrary algorithm is lower bounded, then there must exist a function in $\Gc$ that has the same lower bound.

Now, we prove \cref{thm:lower_sim}, which is restated as follows.
\lowersim*
\begin{proof}
    Since the theorem concerns worst-case RKHS functions, it suffices to establish the same result when the function $g$ is drawn uniformly from the hard subset $\Gc$ introduced above.  Given that $g$ is random, Yao's minimax principle implies that it suffices to consider deterministic algorithms.
    
    For any given deterministic algorithm, let $\Xc'_T=\{\xbf'_1,\xbf'_2,\cdots,\xbf'_T\}$ be the set of points that it would sample if the function were zero everywhere.  We observe that if $g = g_i$ for some $g_i$ satisfying $g(\xbf'_1) = g(\xbf'_2) = \dotsc = g(\xbf'_T) = 0$ (i.e., the bump in $g_i$ does not cover any of the points in $\Xc'_T$), then the algorithm will precisely sample $\xbf'_1,\xbf'_2,\cdots,\xbf'_T$.  Moreover, if there are multiple such functions $g_i$, then the final returned point $\xbf_T^*$ can only be below $2\epsilon$ (i.e., the bump height) for at most one of those functions, since their supports are disjoint by construction.
    
    Now suppose that $T \le \frac{M}{2} - 1$, where $M$ is the number of functions in $\Gc$.  This means that regardless of the values $\xbf'_1,\xbf'_2,\cdots,\xbf'_T$, there are at least $\frac{M}{2}+1$ functions in $\Gc$ such that the sampled values are all zero.  Hence, there are at least $\frac{M}{2}$ functions where a simple regret of at least $2\epsilon$ is incurred, meaning that the average simple regret is at least $\epsilon$.
    
    With this result in place, \cref{thm:lower_sim} immediately follows by substituting $M=\Theta\big((1/w)^d\big)$ and $w=\Theta \big((\frac{\epsilon}{B(cL)^{m-1}})^{1/\nu}\big)$.
\end{proof}

\subsection{Behavior of $c^{m-1}$}
\label{sec:dis_c}

We readily find from \eqref{eq:alpha_fractions} (and the counterparts in the other analyses) that the constant $c$ lies in the range $[\alpha,1]$.  The closer $c$ is to 1, the higher our lower bound is.  To understand how close $c$ can be to 1, we recall from Section \ref{sec:hard_chains} that $\alpha=\frac{r_{\min}\Ltil}{L}$ with $L=\max\{L',r_{\max}\Ltil\}$, $L' = \Theta\big( \frac{\epsilon_1}{w} \big)$, and $\Ltil=\frac{\sqrt{2}B}{\sqrt{k(0)-k(2u)}}$ (as well as $r_{\max}=\sup_{z\in(0,\util]}\frac{k(u-z)-k(u+z)}{2z}$ and $r_{\min}=\inf_{z\in(0,\util]}\frac{k(u-z)-k(u+z)}{2z}$).  

Observe that when $\frac{\epsilon_1}{B} \ll 1$ and $\nu>1$, we have $L'=\Theta\big(\frac{\epsilon_1}{w}\big)=\Theta\big(\frac{\epsilon_1}{(\epsilon_1/B)^{1/\nu}}\big)\ll B$. Since $\Ltil = \Theta(B)$, we conclude that $L = r_{\max}\Ltil$, and hence
\begin{align}
	\alpha=\frac{r_{\min}}{r_{\max}}=\frac{\inf_{z\in(0,\util]}\frac{k(u-z)-k(u+z)}{2z}}{\sup_{z\in(0,\util]}\frac{k(u-z)-k(u+z)}{2z}}.
\end{align}
In our analysis, we need to choose $u$ and $\util$ such that $\gtil_k(\util)\ge 2\epsilon$. Thus, as $\epsilon$ decreases towards zero, $u$ and $\util$ can also be chosen arbitrarily small, which in turn implies that $r_{\min}$ and $r_{\max}$ are arbitrarily close to each other as long as the kernel function has a finite slope near zero.

In fact, in \cref{fig:gtil} we observe that $u$ and $\util$ do not even need to be particularly small to have $\frac{r_{\min}}{r_{\max}}$ close to 1; with $(u,\util) = (0.5,0.3)$ we get this ratio being 0.958, 0.939, and 0.934 for $\nu = 1, 1.5,$ and 2 respectively.

\section{Lower Bound on Cumulative Regret (Proof of \cref{thm:lower_cml})}
\label{sec:lower_cml}
In this section, we prove \cref{thm:lower_cml}, which is restated as follows.
\lowercml*

\begin{proof}
By rearranging \cref{thm:lower_sim}, we have
$\epsilon=\Omega(B(cL)^{m-1}T^{-\nu/d})$, which implies that the lower bound on cumulative regret is
\begin{align}
	\EE[R_T]=\Omega(\epsilon T)=\Omega(B(cL)^{m-1}T^{1-\nu/d}).
\end{align}
However, this lower bound is loose when $d<\nu$. \cref{thm:lower_sim} implies that to have simple regret at most $\epsilon=\Theta(1)$ requires $T=\Omega\big((B(cL)^{m-1})^{d/\nu}\big)$. When $T= \Theta\big((B(cL)^{m-1})^{d/\nu}\big)$, we have $\EE[r^\ast_T]=\Omega(1)$ and $\EE[R_T]=\Omega\big((B(cL)^{m-1})^{d/\nu}\big)$. Since cumulative regret is always non-decreasing in $T$, when $T=\Omega\big((B(cL)^{m-1})^{d/\nu}\big)$, we also have $\EE[R_T]=\Omega\big((B(cL)^{m-1})^{d/\nu}\big)$. 

For $T=o\big((B(cL)^{m-1})^{d/\nu}\big)$ (which is equivalent to $T=o(B(cL)^{m-1}T^{1-\nu/d})$), we show by contradiction that $\EE[R_T]=\Omega(T)$. Suppose on the contrary that there exists an algorithm guaranteeing $\EE[R_T]=o(T)$ when $T=T_0$ for some $T_0 = o\big((B(cL)^{m-1})^{d/\nu}\big)$. Then, by repeatedly selecting the best point among the first $T_0$ time steps, the algorithm attains $\EE[R_T]=o\big((B(cL)^{m-1})^{d/\nu}\big)$ when $T=\Theta\big((B(cL)^{m-1})^{d/\nu}\big)$, which contradicts the lower bound for $T=\Omega\big((B(cL)^{m-1})^{d/\nu}\big)$.

Hence, \cref{thm:lower_cml} follows by combining the two cases.

\end{proof}

\section{Summary of Regret Bounds}
\label{sec:summary}
A detailed summary of our regret bounds for the Mat\'ern kernel is given in \cref{tab:summary}.
\begin{table}[t!]
	\centering
	~\clap{
	
	\begin{tabular}{ll}
		\toprule
		\multicolumn{2}{l}{\bf Algorithm-Independent Cumulative Regret Lower Bound} \\
		Chains/Multi-Output Chains/Feed-Forward Networks & $\begin{cases}
			\Omega\big(\min\{T, B(cL)^{m-1}T^{1-\nu/d}\}\big)&\hspace{-2ex}\text{ when }d>\nu\ge1,\\
			\Omega\big(\min\{T, (B(cL)^{m-1})^{d/\nu}\}\big)&\hspace{-2ex}\text{ when } d\le\nu.
		\end{cases}$ \\
		\textbf{Algorithmic Cumulative Regret Upper Bound} \\ \textbf{(If Conjecture of \cite{vakili2022open} Holds)}\tablefootnote{For bounds not requiring this conjecture, see Section \ref{sec:ucb_disc}.} & \\
		Chains & $\begin{cases}
			O(2^mBL^{m-1}T^{1-\nu/d})&\hspace{7ex}\text{ when }d>\nu,\\
			\Otil(2^mBL^{m-1})&\hspace{7ex}\text{ when }d\le\nu.
		\end{cases}$
		\\
		Multi-Output Chains & $\begin{cases}
			O(5^mBL^{m-1}T^{1-\nu/d_{\max}})&\hspace{4ex}\text{ when }d_{\max}>\nu,\\
			\Otil(5^mBL^{m-1})&\hspace{4ex}\text{ when }d_{\max}\le\nu.
		\end{cases}$
		\\
		Feed-Forward Networks & $\begin{cases}
			O(2^m\sqrt{D_{2,m}}BL^{m-1}T^{1-\nu/d_{\max}})&\hspace{-3ex}\text{ when }d_{\max}>\nu,\\
			\Otil(2^m\sqrt{D_{2,m}}BL^{m-1})&\hspace{-3ex}\text{ when }d_{\max}\le\nu.
		\end{cases}$
		\\
		\midrule
		\multicolumn{2}{l}{\bf Algorithm-Independent Simple Regret Lower Bound} \\
%		\textbf{Algorithm-Independent Simple Regret Lower Bound} & \\
		Chains/Multi-Output Chains/Feed-Forward Networks & $\Omega(B(cL)^{m-1}T^{-\nu/d})$\hspace{14ex}when $\nu\ge1.$ \\ [3mm]
		\textbf{Algorithmic Simple Regret Upper Bound} & \\
		Chains/Multi-Output Chains & $\begin{cases}
			\Otil(BL^{m-1}T^{-\nu^m/d})&\hspace{-1ex}\text{when $\nu\le 1$,}\\
			\Otil(B^{1+\nu+\nu^2+\nu^{m-2}}L^{\nu^{m-1}}T^{-\nu/d})&\hspace{-1ex}\text{when $\nu> 1$.}
		\end{cases}$ \\
	   Chains/Multi-Output Chains (Restrictive Cases) & $\begin{cases}
	   	O(BL^{m-1}T^{-\nu/d})&\hspace{9ex}\text{ when $\nu\le 1$,}\\
	   	O(BL^{(m-1)\nu}T^{-\nu/d})&\hspace{9ex}\text{ when $\nu> 1$.}
	   \end{cases}$ \\
   		Feed-Forward Networks &   $\begin{cases}
   			\Otil(\sqrt{D_{2,m}}BL^{m-1}T^{-\nu^m/d})&\text{ when $\nu\le 1$,}\\
   			\Otil( \Dtil_{2,m}^\nu B^{1+\nu+\nu^2+\nu^{m-2}}L^{\nu^{m-1}}T^{-\nu/d} )&\text{ when $\nu> 1$.}
   		\end{cases}$ \\
   		Feed-Forward Networks (Restrictive Cases) & $\begin{cases}
   		O(\sqrt{D_{2,m}}BL^{m-1}T^{-\nu/d})&\hspace{-1ex}\text{ when $\nu\le 1$,}\\
   		O\big((D_{2,m})^{\nu/2}BL^{(m-1)\nu}T^{-\nu/d}\big)&\hspace{-1ex}\text{ when $\nu> 1$.}
   		\end{cases}$ \\
		\bottomrule                 
	\end{tabular}
	}
	\caption{Summary of regret bounds for the Mat\'ern kernel. (The algorithmic simple regret upper bounds are valid when the domain of each layer is a hyperrectangle.) $T$ denotes the time horizon; $m$ denotes the number of layers; $B$ denotes the RKHS norm upper bound of each layer; $L$ denotes the Lipschitz constant upper bound of each layer; $\nu$ denotes the smoothness of the Mat\'ern kernel; $d$ denotes the domain dimension of $g$; $d_{\max}=\max_{i\in[m]}d_i$ denotes the maximum dimension among all the $m$ layers; $D_{2,m}=\prod_{i=2}^m d_i$ denotes the product of the dimensions from the second to the last layer; $\Dtil_{2,m}^\nu=\prod_{i=2}^m (d_i)^{\nu^{m+1-i}/2}$. The lower bounds hold for some $c=\Theta(1)$.}
	\label{tab:summary}
\end{table}

\section{Comparison to Related Works}
\label{sec:comparison}
\subsection{Comparison to \cite{kusakawa2021bayesian}}
In this section, we compare our theoretical result to two confidence bound based algorithms cascade UCB (cUCB) and optimistic improvement (OI) proposed by \cite{kusakawa2021bayesian}. Both algorithms utilize a novel posterior standard deviation defined using the Lipschitz constant. 
%\cite{kusakawa2021bayesian} also provided the cumulative regret upper bound for cUCB and the simple regret upper bound of OI. 

\textbf{Problem Setup.} 
\begin{itemize}
	\item In each layer, \cite{kusakawa2021bayesian} assumes the entries in the same layer are mutually independent, which is equivalent to our feed-forward network structure.
	\item Different from our assumption of $\ff{i,j}\in\Hc_k(B)\cap\Fc(L, \|\cdot\|_2)$ for each $i\in[m]$ and $j\in[d_{i+1}]$, where the Lipschitz continuity is with respect to $\|\cdot\|_2$, \cite{kusakawa2021bayesian} assumes $\ff{i,j}\in\Hc_k(B)\cap\Fc(L_f, \|\cdot\|_1)$, where the Lipschitz continuity is with respect to $\|\cdot\|_1$. Therefore, for a fixed scalar-valued function $\ff{i,j}\in\Fc(L, \|\cdot\|_2)\cap\Fc(L_f, \|\cdot\|_1)$ with the smallest possible $L$ and $L_f$, it is satisfied that $L_f\le L\le\sqrt{d_i}L_f$.
	\item \cite{kusakawa2021bayesian} has an additional assumption that $\sig{i,j}_t\in\Fc(L_\sigma,\|\cdot\|_1)$ for all $t\ge 1$ for some constant $L_\sigma$. The constant $L_\sigma$ is not used in the algorithms, and only appears in the regret bounds.
\end{itemize}

\textbf{Regret Bounds. }
\begin{itemize}
	\item With $D_{2,m}=\prod_{i=2}^m d_i$ and $D^+_{2,m}=\sum_{i=2}^m d_i$, noise-free cUCB achieves cumulative regret
	\begin{align*}
		R_T=O\big(B(BL_\sigma+L_f)^{m-1}D_{2,m}D^+_{2,m} \sqrt{T\gamma_T}\big),
	\end{align*}
%	 $R_T=O\big(B(BL_\sigma+L_f)^{m-1}D_{2,m}D^+_{2,m} \sqrt{T\gamma_T}\big)$,
	  with a $T$-independent factor no smaller than that in $R_T=O(2^mBL^{m-1}\sqrt{D_{2,m}T\gamma_T})$ of our GPN-UCB (\cref{algo:ucb}) since $BL^{m-1}\sqrt{D_{2,m}}\le BL_f^{m-1}D_{2,m}\le B(BL_\sigma+L_f)^{m-1}D_{2,m}$. 
	
%	\item Similarly, noisy cUCB has cumulative regret $R_T=O\big((\beta_TL_\sigma+L_f)^mD_{2,m}D^+_{2,m} \sqrt{T\gamma_T}\big)$, which is simplified and improved by our bound $R_T=O(2^m\beta_TL^{m-1}\sqrt{D_{2,m}T\gamma_T})$, since $\beta_TL^{m-1}\sqrt{D_{2,m}}\le\beta_TL_f^{m-1}D_{2,m}\le(\beta_TL_\sigma+L_f)^mD_{2,m}$.  Under the choice $\beta_T=O(B+\sqrt{\gamma_T})$ from \cite{chowdhury17kernelized}, due to the scaling $\gamma_T = O\big( T^{\frac{d}{2\nu + d}} \log T \big)$ \citep{vakili2021information}, GPN-UCB achieves sublinear $T$-dependence for the Mat\'ern kernel when $\nu>d/2$, while cUCB requires $\nu>md/2$, which becomes highly restrictive as $m$ increases.  Moreover, as discussed at the end of Section \ref{sec:noisy}, even the restriction $\nu > d/2$ can be removed for GPN-UCB by shifting to the improved confidence bounds from \cite{whitehouse2023sublinear}.

	\item Noise-free OI achieves simple regret
	\begin{align*}
		r^\ast_T=O\big(m^{m^2+m+\frac{1}{2}}B^{m+1}L_f^{m^3}(BL_\sigma+L_f)^{3m^3}(D_{2,m})^{2m^2}(D^+_{2,m})^{m^2+1}T^{-\frac{\nu}{2\nu+d}}\big)
	\end{align*}
%	 $r^\ast_T=O\big(m^{m^2+m+\frac{1}{2}}B^{m+1}L_f^{m^3}(BL_\sigma+L_f)^{3m^3}(D_{2,m})^{2m^2}(D^+_{2,m})^{m^2+1}T^{-\frac{\nu}{2\nu+d}}\big)$ 
	 for the Mat\'ern kernel with smoothness $\nu$. When $\nu\le 1$, our non-adaptive sampling (\cref{algo:fd}) achieves $r^\ast_T=\Otil(\sqrt{D_{2,m}}BL^{m-1}T^{-\nu^m/d})$, which has a significantly smaller $T$-independent factor. When $\nu>1$, our \cref{algo:fd} achieves $r^\ast_T=\Otil( \Dtil_{2,m}^\nu B^{1+\nu+\nu^2+\nu^{m-2}}L^{\nu^{m-1}}T^{-\nu/d} )$, which has a smaller $T$-dependent factor.  The $T$-independent factors are more difficult to compare, though ours certainly become preferable under the more restrictive scenarios discussed leading up to \eqref{eq:more_restrictive}.
\end{itemize}

\textbf{Generality.} \cite{kusakawa2021bayesian} allows additional (multi-dimensional) input independent of previous layers for each layer. GPN-UCB and non-adaptive sampling can also be adapted to accept additional input. Let $\widehat{\Xc}^{(i)}\subset\RR^{\dhat_i}$ denote the domain of the additional input of $\ff{i}$, and let $\widehat{\xbf}_{[i:j]}=[\widehat{\xbf}^{(i)},\dots,\widehat{\xbf}^{(j)}]$ denote the concatenation of the addition inputs from $\ff{i}$ to $\ff{j}$.
\begin{itemize}
	\item For GPN-UCB (\cref{algo:ucb}), we simply modify the upper confidence bound to be
	\begin{align}
		\UCB_t(\xbf,\widehat{\xbf})&=\max_{\zbf\in\Delta^{(m)}_t(\xbf,\widehat{\xbf}_{[2:m-1]})}\UCBBB{m,1}_t([\zbf,\widehat{\xbf}^{(m)}]),
	\end{align}
	where
	\begin{align}
		\Delta^{(1)}_t(\xbf)=&\{\xbf\},\\
		\Delta^{(i+1,j)}_t(\xbf,\widehat{\xbf}_{[2:i]})=&\Bigg[\min_{\zbf\in\Delta^{(i)}_t(\xbf,\widehat{\xbf}_{[2:i-1]})}\LCBBB{i,j}_t([\zbf,\widehat{\xbf}^{(i)}]) , \max_{\zbf\in\Delta^{(i)}_t(\xbf,\widehat{\xbf}_{[2:i-1]})}\UCBBB{i,j}_t([\zbf,\widehat{\xbf}^{(i)}]) \Bigg]\nonumber\\
		&\hspace{42ex}\text{ for }i\in[m-1],j\in[d_{i+1}],\\
		\Delta^{(i)}_t(\xbf,\widehat{\xbf}_{[2:i-1]})=& \Delta^{(i,1)}_t(\xbf,\widehat{\xbf}_{[2:i-1]})\times\cdots\times\Delta^{(i,d_i)}_t(\xbf,\widehat{\xbf}_{[2:i-1]}) \nonumber\\
		&\hspace{57ex}\text{ for }i\in[m].
	\end{align}
	The cumulative regret upper bound will remain the same, since $\dia{\del{i}_t(\xbf,\widehat{\xbf}_{[2:i-1]})}$ remains the same.
	\item For non-adaptive sampling (\cref{algo:fd}), with $\dhat=d+\sum_{i=2}^m \dhat_i$, we choose $\{\xbf_s, \widehat{\xbf}^{(2)}_s,\dots,\widehat{\xbf}^{(m)}_s\}_{s=1}^T$ such that
	\begin{align}
		\max_{[\xbf,\widehat{\xbf}^{(2)},\dots,\widehat{\xbf}^{(m)}]\in \Xc\times\widehat{\Xc}^{(2)}\times\cdots\times\widehat{\Xc}^{(m)}} \min_{s\in[T]} \| [\xbf,\widehat{\xbf}^{(2)},\dots,\widehat{\xbf}^{(m)}] - [\xbf_s,\widehat{\xbf}^{(2)}_s,\dots,\widehat{\xbf}^{(m)}_s] \|_2=O(T^{-1/\dhat}).
	\end{align}
	Then, modify the composite mean to be
	\begin{align}
		\mu_T^g(\xbf,\widehat{\xbf})=\muu{m,1}_T([\zz{m},\widehat{\xbf}^{(m)}])
	\end{align}
	with
	\begin{align}
		\zz{1}&=\xbf,\\
		\zz{i+1,j}&=\muu{i,j}_T([\zz{i},\widehat{\xbf}^{(i)}])&&\text{ for }i\in[m-1],j\in[d_{i+1}],\\
	\zz{i+1}&=\muu{i}_T([\zz{i},\widehat{\xbf}^{(i)}])=[\zz{i+1,1},\dots,\zz{i+1,d_i}]&&\text{ for }i\in[m-1].
	\end{align}
	Then, the only change in simple regret upper bound is that $d$ is replaced by $\dhat$. 
\end{itemize} 

\subsection{Comparison to \cite{sussex2023modelbased} \label{sec:cmp_noisy}}
Another related work \cite{sussex2023modelbased} seeks the best intervention action for a given causal graph structure. Focusing on DAGs, they have a similar setup to \cite{kusakawa2021bayesian}:
\begin{itemize}
	\item $\ff{i}\in\Hc_k(B)\cap\Fc(L, \|\cdot\|_2)$ for each node $i$, where the Lipschitz continuity is with respect to $\|\cdot\|_2$.
	\item $\sig{i}_t\in\Fc(L_\sigma,\|\cdot\|_2)$ for each node $i$ and all $t\ge 1$ for some constant $L_\sigma$.
\end{itemize}
When the causal graph structure is a feed-forward network, their expected improvement based method achieves $R_T=O(B^m L_\sigma^m L^m  d_{\max}^m D^+_{2,m}\sqrt{T\gamma_T})$ when specialized to the noise-free setting, thus containing significantly larger $T$-independent terms compared to GPN-UCB similarly to the above discussion.  For fairness, we note that \cite{sussex2023modelbased} is concerned mainly with the noisy setting, which we do not handle in this paper, instead leaving analogous improvements as potential future work.
%When the causal graph structure is a feed-forward network, their expected improvement based method achieves $R_T=O(\beta_T^m L_\sigma^m L^m  d_{\max}^m D^+_{2,m}\sqrt{T\gamma_T})$ under the noisy setting. By the same reasoning as that above for cUCB with noise, they attain sub-linear $T$-dependence when $\nu>md/2$, whereas using the same confidence from \cite{chowdhury17kernelized}, we only require the much milder condition $\nu > d/2$.   (Again, we also recall from  Section \ref{sec:noisy} that we can even avoid this condition altogether for GPN-UCB by using tighter confidence bounds.)

%In summary, the preceding discussion illustrates that a term of the form $\beta_T^m$ can be very large unless the kernel is very smooth (e.g., SE kernel, or Mat\'ern kernel with very large $\nu$).  The key benefit of our noisy results is avoiding such dependence, and only incurring a single $\beta_T$ term.

\section{Experiments}
\label{sec:exp}

In this section, we experimentally evaluate the performance of our proposed algorithms on chains composed of synthetic functions, and compare to three grey-box algorithms (cascaded UCB (cUCB), optimistic improvement (OI), and cascaded expected improvement (cEI) \cite{kusakawa2021bayesian}) and to two classic black-box algorithms. We note that with our main contributions all being theoretical, these experiments are only intended to suggest the potential plausibility of our algorithms for practical use, rather than being comprehensive or definitive.

\subsection{Synthetic Chains}
We generate two chains $g_1=f_3\circ f_2\circ f_1$ and $g_2=h_3\circ h_2\circ h_1$ (see \cref{fig:gs}) with $d=2$ and $m=3$ by sampling $\{f_1,f_2,f_3,h_1,h_2,h_3\}$ from GP prior with zero mean and squared exponential kernel with lengthscale $l=1$. $g_1$ and $g_2$ share the same domain $\Xc$, which contains $2500$ points obtained by evenly splitting $[-5,5]^2$ into a $50\times 50$ grid. We set $B=2$ and $L=2$ for both chains.

\begin{figure}
	\centering
	\begin{subfigure}[b]{0.32\textwidth}
		\includegraphics[width=\textwidth]{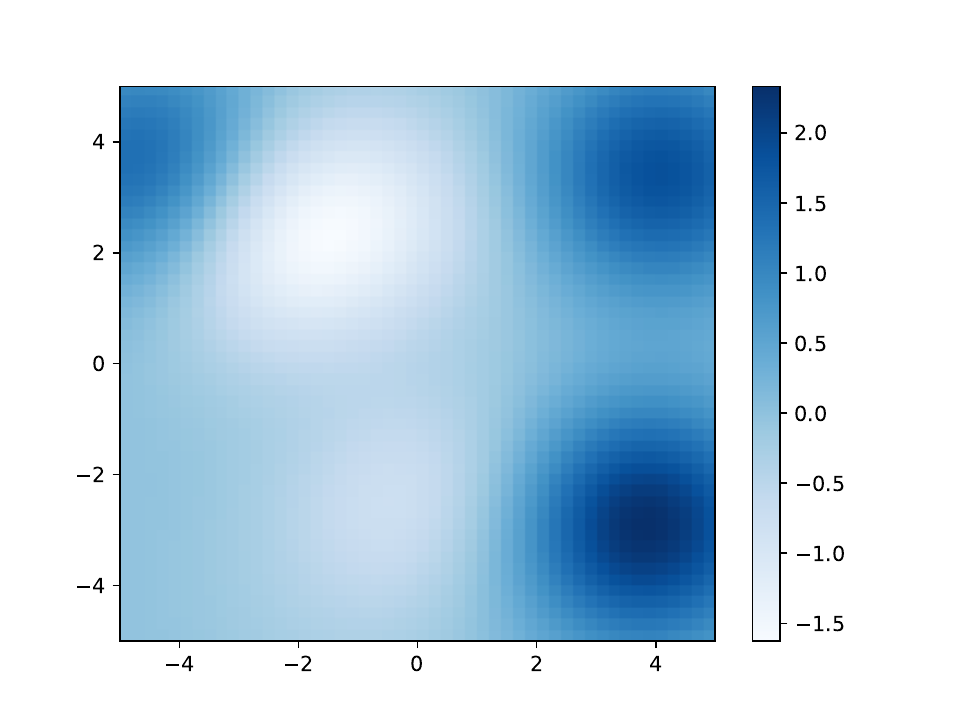}
		\caption{$f_1$}
	\end{subfigure}
	\begin{subfigure}[b]{0.32\textwidth}
		\includegraphics[width=\textwidth]{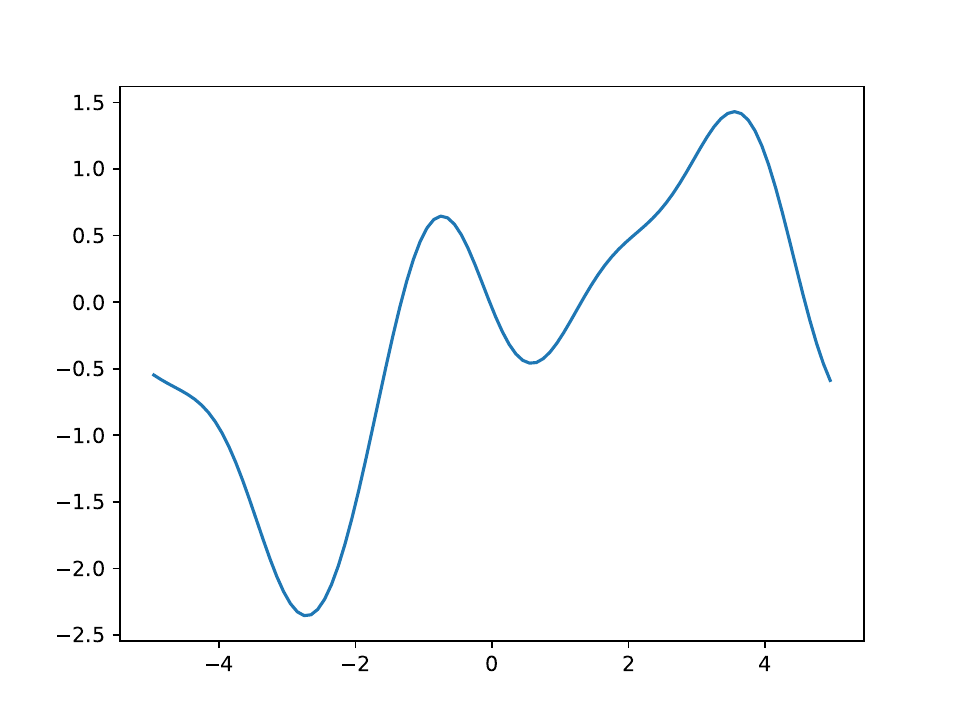}
		\caption{$f_2$}
	\end{subfigure}
	\begin{subfigure}[b]{0.32\textwidth}
		\includegraphics[width=\textwidth]{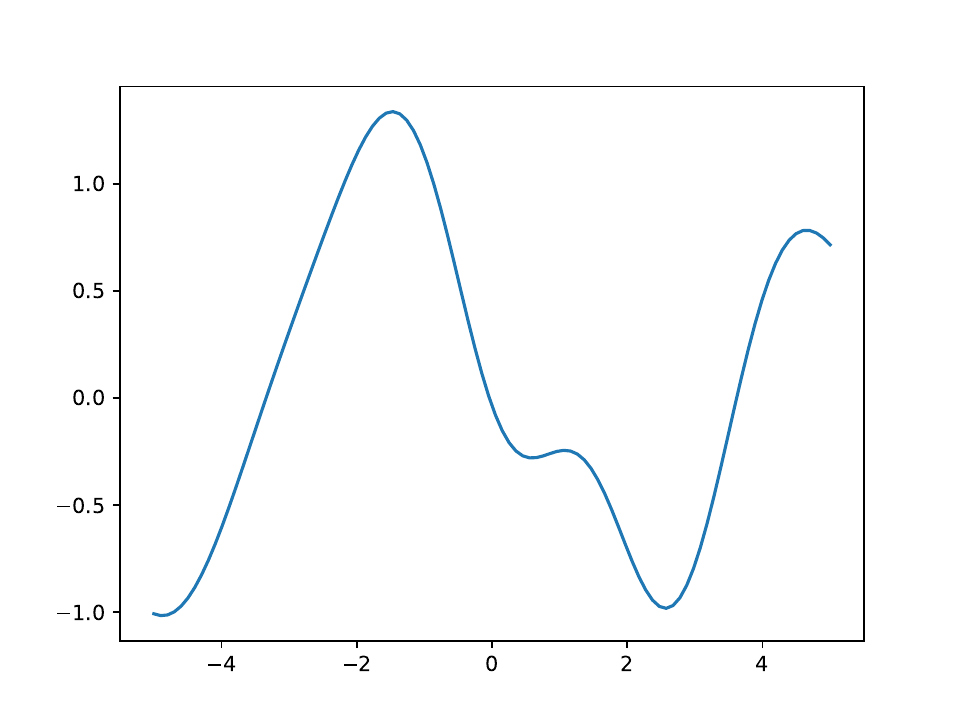}
		\caption{$f_3$}
	\end{subfigure}
	\newline
	\begin{subfigure}[b]{0.32\textwidth}
		\includegraphics[width=\textwidth]{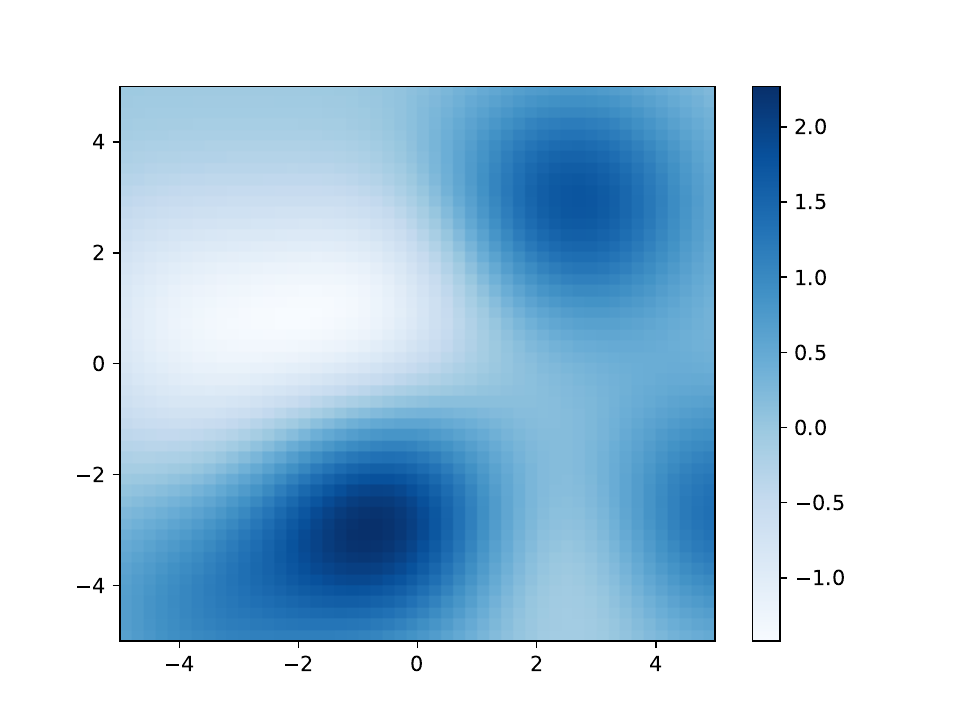}
		\caption{$h_1$}
	\end{subfigure}
	\begin{subfigure}[b]{0.32\textwidth}
		\includegraphics[width=\textwidth]{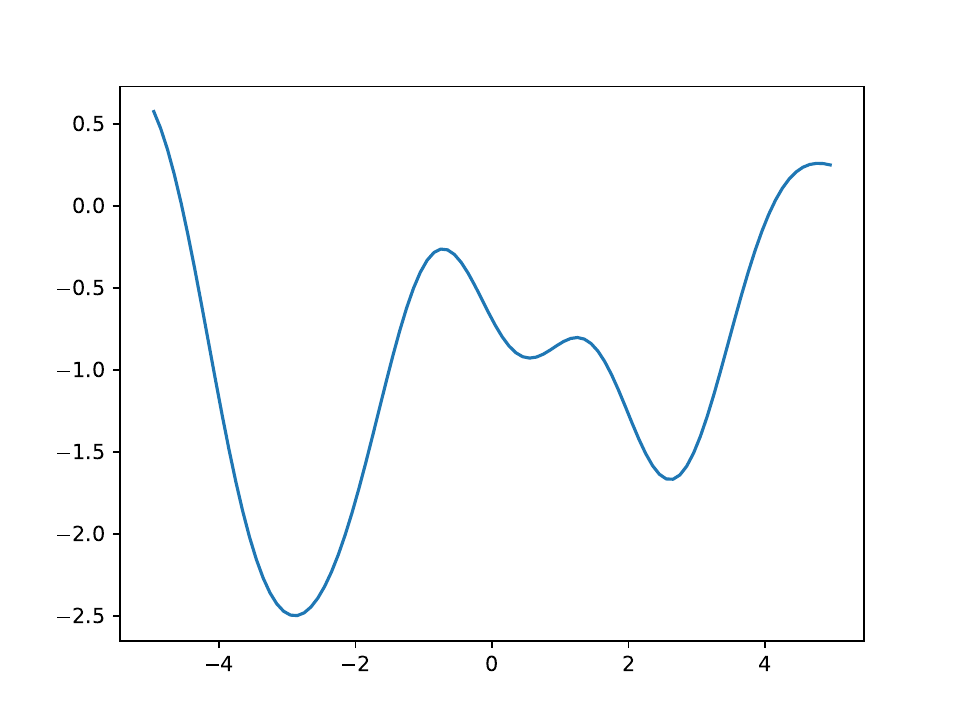}
		\caption{$h_2$}
	\end{subfigure}
	\begin{subfigure}[b]{0.32\textwidth}
		\includegraphics[width=\textwidth]{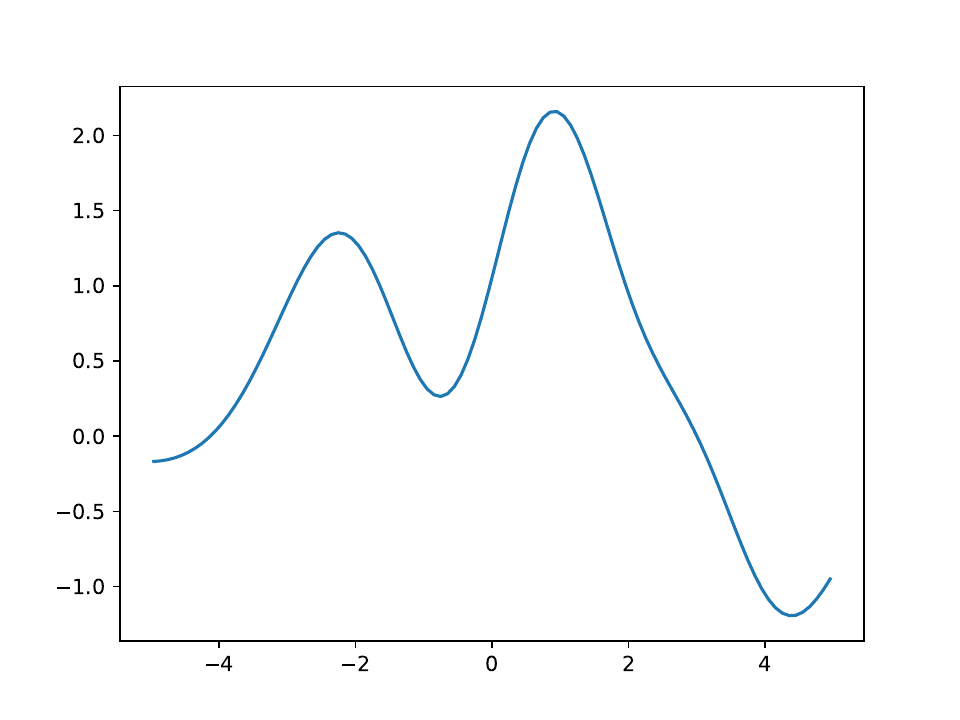}
		\caption{$h_3$}
	\end{subfigure}
	\newline
	\begin{subfigure}[b]{0.45\textwidth}
		\includegraphics[width=\textwidth]{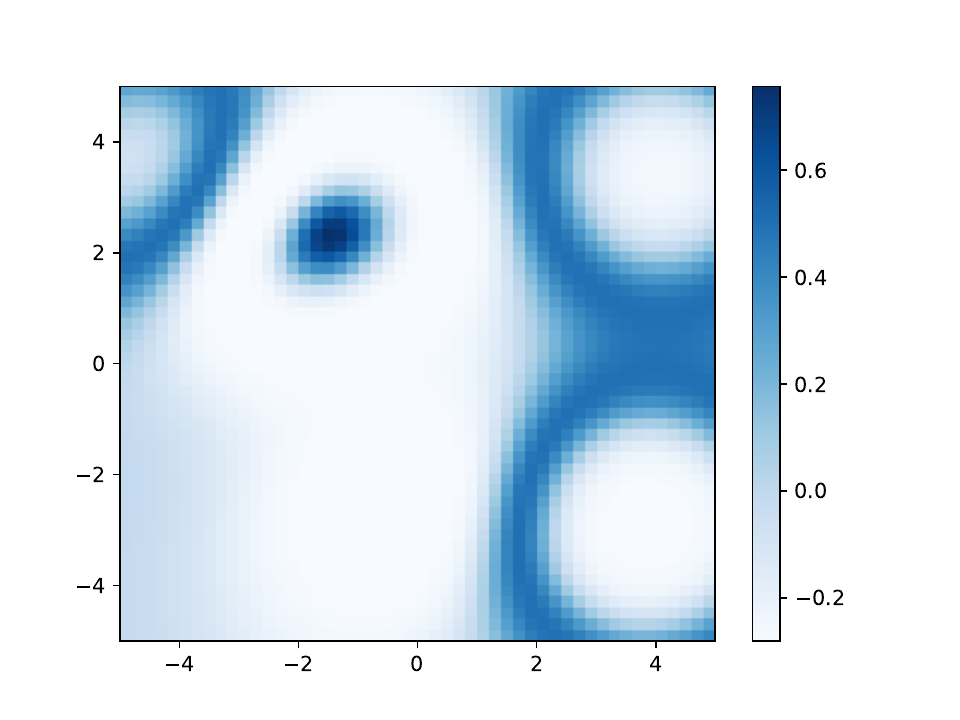}
		\caption{$g_1=f_3\circ f_2\circ f_1$}
	\end{subfigure}
	\begin{subfigure}[b]{0.45\textwidth}
		\includegraphics[width=\textwidth]{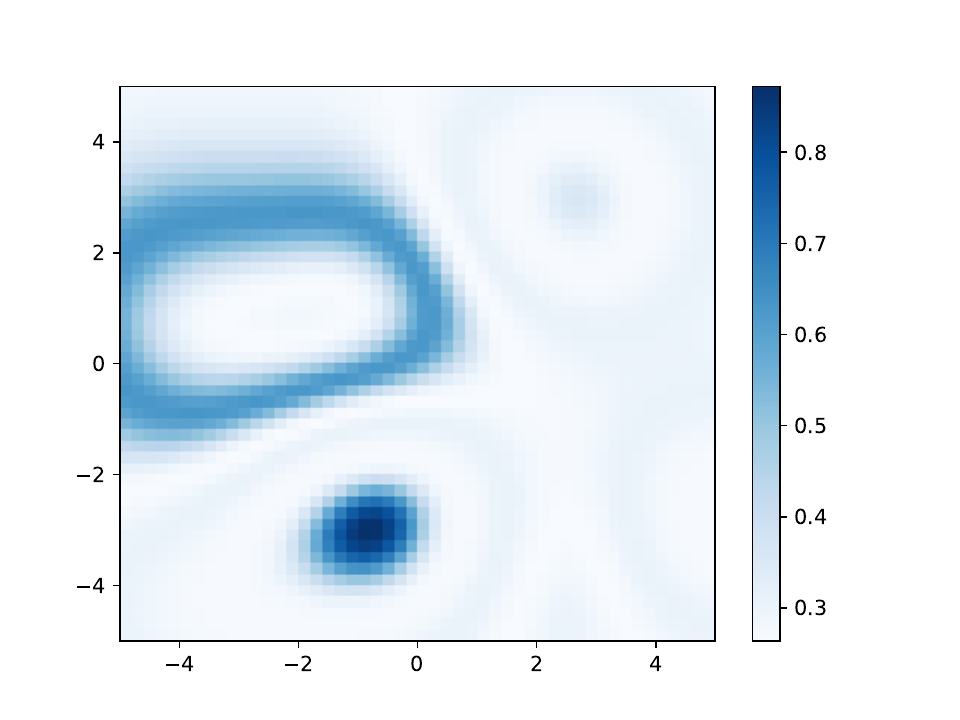}
		\caption{$g_2=h_3\circ h_2\circ h_1$}
	\end{subfigure}
	\caption{Synthetic chains $g_1=f_3\circ f_2\circ f_1$ and $g_2=h_3\circ h_2\circ h_1$}
	\label{fig:gs}
\end{figure}

\subsection{Algorithms} 
We consider our proposed algorithms GPN-UCB and NonAda, three grey-box algorithms cUCB, OI, and cEI from \cite{kusakawa2021bayesian}, and two black-box algorithms GP-UCB and EI.

\textbf{GPN-UCB.} 
To reduce the computational cost of GPN-UCB, with $\Dc$ denoting $100$ points evenly selected from $[-5,5]$, we make the following adjustments for each $i\in\{2, \dots, m\}$ in the implementation:
\begin{itemize}
	\item When computing $\overline{\mathrm{UCB}}^{(i)}_t(\zbf)$ (resp. $\overline{\mathrm{LCB}}^{(i)}_t(\zbf)$), we only take minimum (resp. maximum) over $\{\zbf'\in \Dc: \|\zbf-\zbf'\|\le 1\}$
	\item When computing $\Delta^{(i+1)}_t(\xbf)$, we always replace $\Delta^{(i)}_t(\xbf)$ with $\Delta^{(i)}_t(\xbf)\cap \Dc$. If $\Delta^{(i)}_t(\xbf)$ is too small making the intersection empty, we use the two endpoints of $\Delta^{(i)}_t(\xbf)$ instead.
\end{itemize}

\textbf{NonAda.}
Since the non-adaptive sampling strategy depends on the value of $T$, we run NonAda with $T\in\{2^2, 3^2,\dots, 14^2\}$ in parallel, and compute the simple regret of the returning point for each $T$.

\textbf{cUCB and OI.} \cite{kusakawa2021bayesian} introduces confidence bounds as follows:
\begin{align}
	\mathrm{UCB}_t(\xbf) &= \mutil_t^{(m)}(\xbf)+ B \sigtil_t^{(m)}(\xbf),\\
	\mathrm{LCB}_t(\xbf) &= \mutil_t^{(m)}(\xbf)- B \sigtil_t^{(m)}(\xbf),
\end{align}
where
\begin{align}
	\mutil_t^{(1)}(\xbf) &= \mu_t^{(1)}(\xbf), \\
	\mutil_t^{(i)}(\xbf) &= \mu_t^{(i)}\circ\cdots\circ\mu_t^{(1)}(\xbf) &\text{ for }i=2,\dots,m, \\
	\sigtil_t^{(1)}(\xbf) &= \sigma_t^{(1)}(\xbf), \\
	\sigtil_t^{(i)}(\xbf) &= \sigma_t^{(i)}\big(\mutil_t^{(i-1)}(\xbf)\big)+L\sigtil_t^{(i-1)}(\xbf)&\text{ for }i=2,\dots,m.\label{eq:kusa_sig}
\end{align}
Then, cUCB iteratively selects
\begin{align}
	\xbf^{\mathrm{cUCB}}_t &= \argmax \mathrm{UCB}_{t-1}(\xbf).
\end{align}
OI iteratively selects
\begin{align}
	\xbf^{\mathrm{OI}}_t &= \argmax \{ \mathrm{UCB}_{t-1}(\xbf)-\max\mathrm{LCB}_{t-1},\eta_{t-1}\sigtil_{t-1}(\xbf) \},
\end{align}
where $\eta_t=(1+\log t)^{-1}$ is an additional parameter. 

In our experiments, we introduce a parameter $b$ and set $\eta_t=b\cdot(1+\log t)^{-1}$. To find the best choice of $b$ for OI, we conduct experiments for OI with $b\in\{10^{-4},10^{-3},\dots,10^4\}$, as well as cUCB. We conduct $10$ independent trials and plot (see \cref{fig:oi}) the average simple regret of the best observed points $r_T^\ast= g(\xbf^\ast)-\max_{t\le T} y_t$. The experimental results show that OI with $b\le 1$ has similar performance to cUCB, and we will use OI with $b=1$ as one of the baselines in the following section.

\begin{figure}
	\centering
	\begin{subfigure}[b]{0.45\textwidth}
		\includegraphics[width=\textwidth]{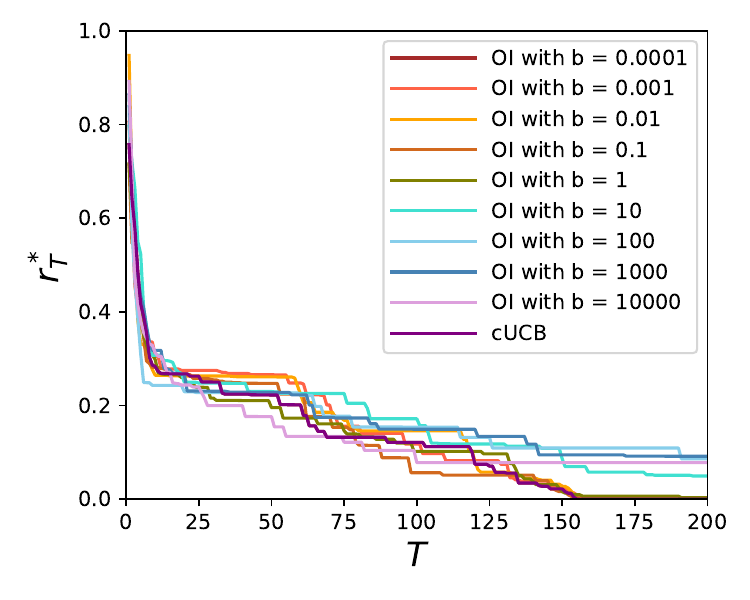}
		\caption{$g_1$}
	\end{subfigure}
	\begin{subfigure}[b]{0.45\textwidth}
		\includegraphics[width=\textwidth]{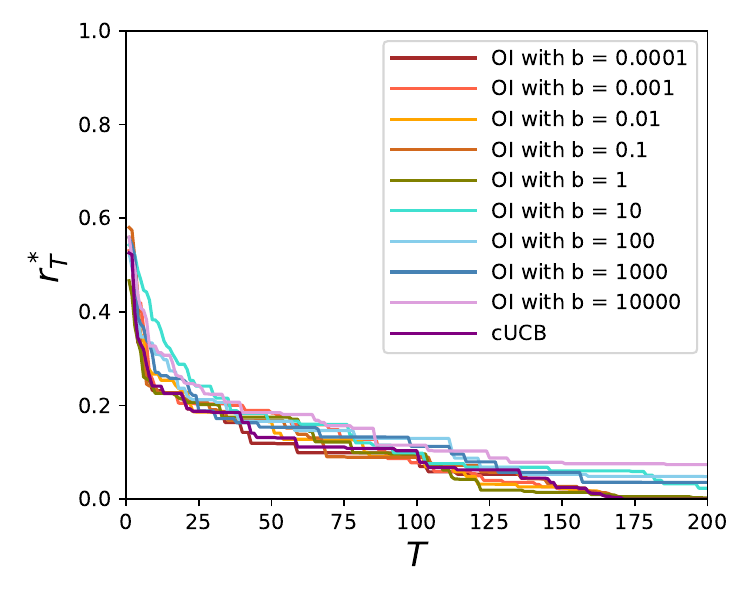}
		\caption{$g_2$}
	\end{subfigure}
	\caption{Simple regret of OI with different values of $b$ and cUCB. (The curves of OI with $b=0.0001$ and cUCB for $g_1$ are indistinguishable.)}
	\label{fig:oi}
\end{figure}

\textbf{cEI.} The exact cascaded expected improvement is
\begin{align}
	\mathrm{cEI}_t(\xbf)=\EE_{f^{(1)},\dots,f^{(m)}}[\max\{(f^{(m)}\circ\cdots\circ f^{(1)})(\xbf)-y_{\max},0\}],
\end{align}
where $\ff{i}$ follows the posterior distribution based on $t$ observations and $y_{\max}$ is the highest observed value. 

Similar to \cite{kusakawa2021bayesian}, we approximate this acquisition function by sampling. Specifically, given $\xbf$, to generate a sample of $g(\xbf)$, we do the following:
\begin{itemize}
	\item draw a sample $y^1_t(\xbf)$ from $N\big(\mu_t^{(1)}(\xbf), \sigma_t^{(1)}(\xbf)^2\big)$;
	\item recursively draw a sample $y^{i+1}_t(\xbf)$ from $N\big(\mu_t^{(i+1)}(\zbf), \sigma_t^{(i+1)}(\zbf)^2\big)$, where $\zbf=y^i_t(\xbf)$.
\end{itemize}
Repeating this process $S$ times, we obtain $Y_t(\xbf)$, containing $S$ samples of $g(\xbf)$ from the cascaded posterior distribution, and $\mathrm{cEI}_t(\xbf)$ can be approximated by
\begin{align}
	\widehat{\mathrm{cEI}}_t(\xbf)= \frac{1}{S} \sum_{s\in Y_t(\xbf)} \max\{s-y_{\max},0\}.
\end{align}
The algorithm selects
\begin{align}
	\xbf^{\mathrm{cEI}}_t &= \argmax \widehat{\mathrm{cEI}}_{t-1}(\xbf).
\end{align}
In the experiments, we set the sample size as $S=1000$.

\textbf{GP-UCB and EI.} Both algorithms consider $g$ as a black-box function and ignore the intermediate observations. With $\mu_t(\xbf)$ and $\sigma_t(\xbf)$ denoting the posterior mean and standard deviation of the overall function $g(\xbf)$, GP-UCB selects
\begin{align}
	\xbf^{\mathrm{GP-UCB}}_t &= \argmax \mu_{t-1}(\xbf)+ B\sigma_{t-1}(\xbf),
\end{align}
and EI selects
\begin{align}
	\xbf^{\mathrm{EI}}_t &= \argmax \EE_{g}[\max\{g(\xbf)-y_{\max},0\}]\\
	&=\big(\mu_{t-1}(\xbf)-y_{\max}\big)\Phi\Big(\frac{\mu_{t-1}(\xbf)-y_{\max}}{\sigma_{t-1}(\xbf)}\Big) +\sigma_{t-1}(\xbf)\phi\Big(\frac{\mu_{t-1}(\xbf)-y_{\max}}{\sigma_{t-1}(\xbf)}\Big),
\end{align}
where $\phi$ and $\Phi$ are the pdf and cdf of the standard Gaussian distribution, and $y_{\max}$ is the highest observed value.

\subsection{Experimental Results}
We let $T=200$, and compute posterior mean and variance with $\lambda=10^{-7}$ to avoid numerical error. The experimental results are displayed in \cref{fig:cml} and \cref{fig:sim}, where the average regret is computed based on $10$ independent trials with error bars indicating one standard deviation. Note that the randomness of cEI mainly comes from sampling Gaussian random variables, and the randomness of other algorithms comes from random tie-breaking in the $\argmax$ operation. In \cref{fig:sim}, except for NonAda, the reported point for computing simple regret is the best observed point, i.e., $r_T^\ast= g(\xbf^\ast)-\max_{t\le T} y_t$. The experimental results show that both algorithms are able to locate a near-optimal point in a short time (by $T=50$). 

In terms of cumulative regret, GPN-UCB outperforms all of the baselines. In terms of simple regret, for $g_2$ with an ``obvious'' maximizer, NonAda outperforms all the baselines, and GPN-UCB has very similar performance to cEI.  For $g_1$ with several potential maximizers, our algorithms are slightly outperformed by cEI, but still surpass other baselines. However, it should be noted that the slightly better performance of cEI requires significantly higher computation resources due to the large sample size. Perhaps unexpectedly, we found that the grey-box OI and cUCB algorithms (but not GPN-UCB and cEI) are sometimes outperformed by the black-box GP-UCB and EI algorithms, at least in this simple setting with relatively small $m$.  This may be caused by the large uncertainty in \eqref{eq:kusa_sig} when $L>1$, which makes the confidence bounds loose. This observation, on the other hand, provides some justification for our envelope technique in constructing confidence bounds. Comparing the cumulative and simple regret, we observe that although some trials of cEI and cUCB query a good point at an early time (due to randomness), this does not always help them find the near-optimal point earlier.

\begin{figure}
	\centering
	\begin{subfigure}[b]{0.45\textwidth}
		\includegraphics[width=\textwidth]{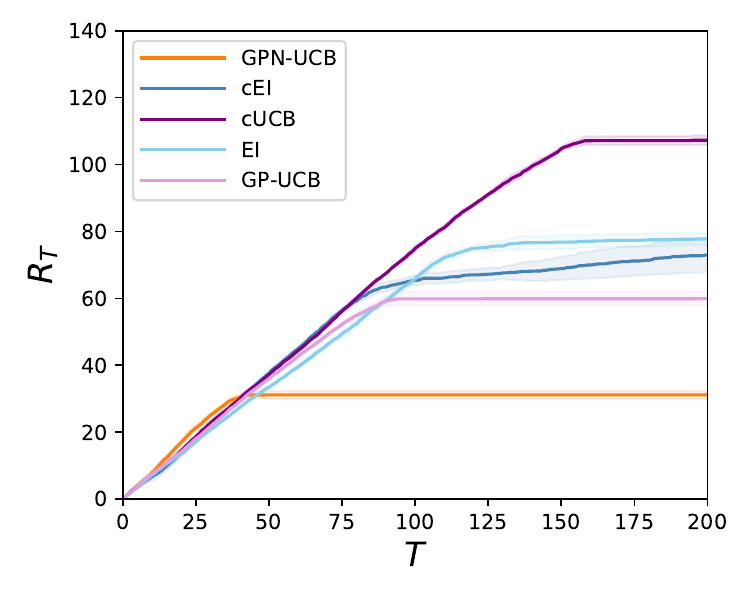}
		\caption{$g_1$}
	\end{subfigure}
	\begin{subfigure}[b]{0.45\textwidth}
		\includegraphics[width=\textwidth]{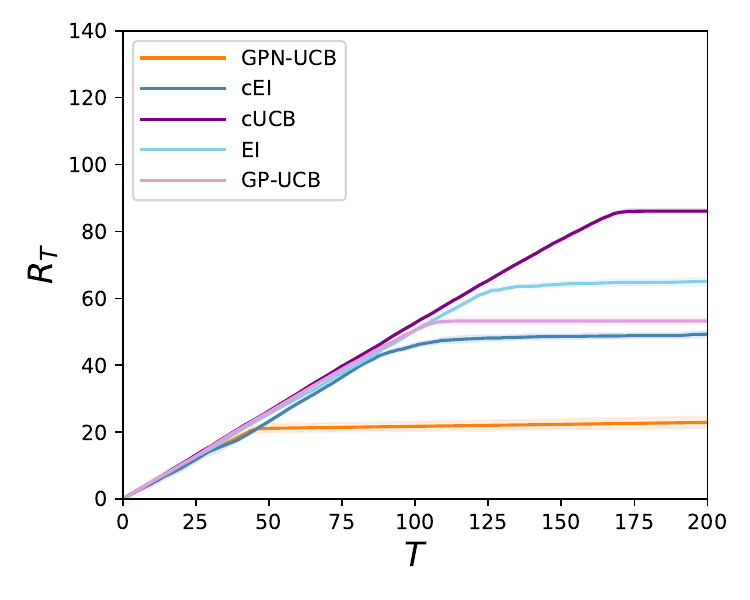}
		\caption{$g_2$}
	\end{subfigure}
	\caption{Cumulative regret of GPN-UCB, cEI, cUCB, EI, and GP-UCB .}
	\label{fig:cml}
\end{figure}

\begin{figure}
	\centering
	\begin{subfigure}[b]{0.9\textwidth}
		\includegraphics[width=\textwidth]{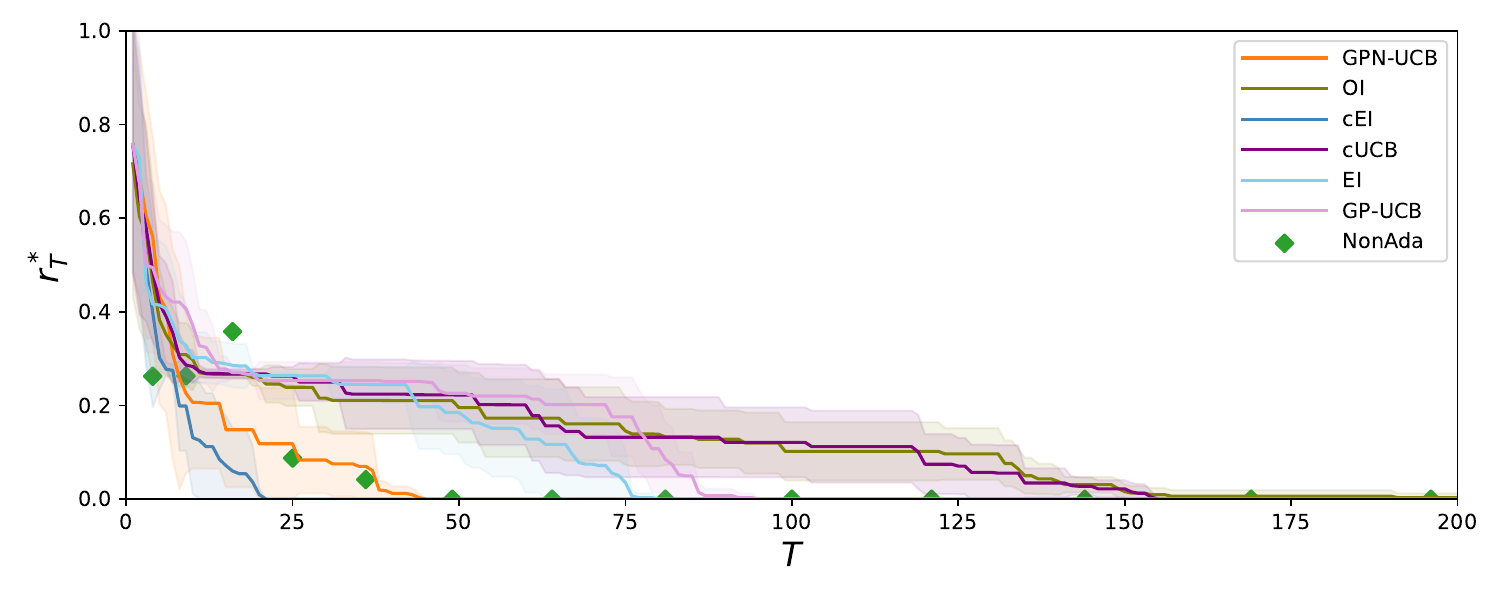}
		\caption{$g_1$}
	\end{subfigure}
	\begin{subfigure}[b]{0.9\textwidth}
		\includegraphics[width=\textwidth]{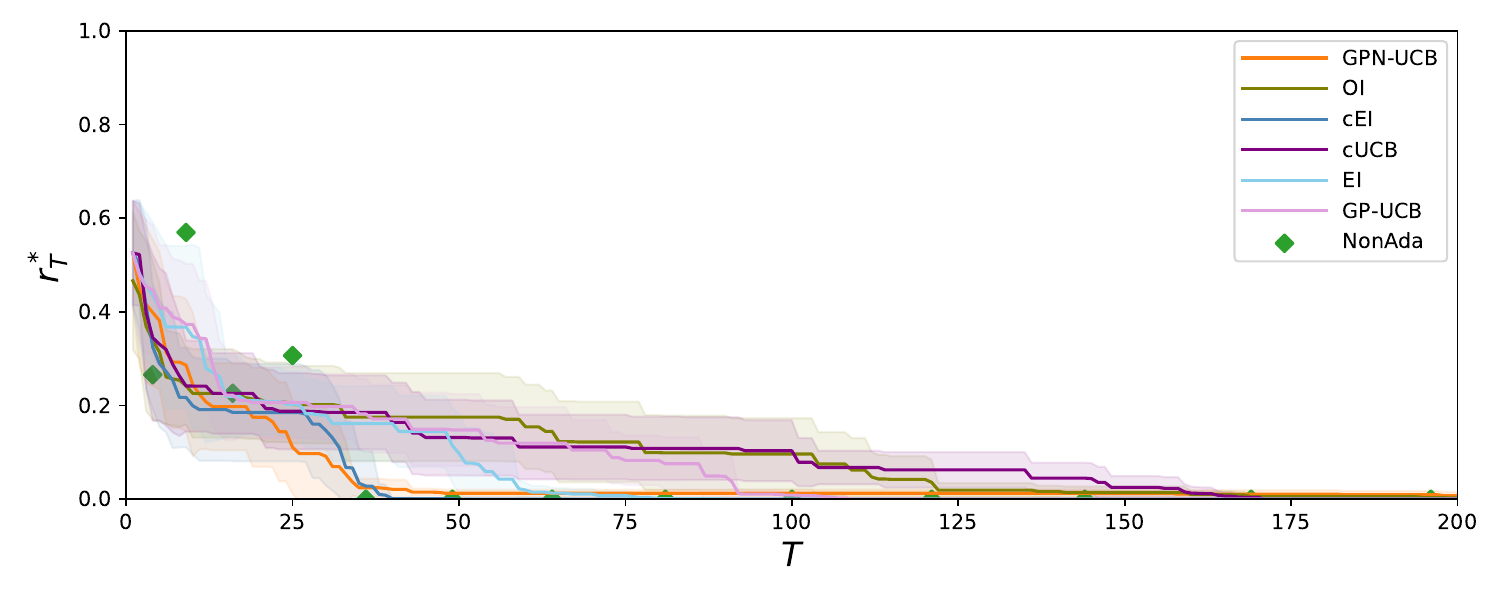}
		\caption{$g_2$}
	\end{subfigure}
	\caption{Simple regret of GPN-UCB, NonAda, OI, cEI, cUCB, EI, and GP-UCB.}
	\label{fig:sim}
\end{figure}

\end{document}